\documentclass[11pt]{article}

\usepackage{times}

\RequirePackage[OT1]{fontenc}
\RequirePackage{amsthm,amsmath}
\RequirePackage[colorlinks,citecolor=blue,urlcolor=blue]{hyperref}
\theoremstyle{plain}

\newtheorem*{rep@theorem}{\rep@title}
\newcommand{\newreptheorem}[2]{%
  \newenvironment{rep#1}[1]{%
    \def\rep@title{#2 \ref{##1}}%
    \begin{rep@theorem}}%
    {\end{rep@theorem}}}
\makeatother
\newreptheorem{theorem}{Theorem}
\newreptheorem{lemma}{Lemma}
\newreptheorem{corollary}{Corollary}
\newreptheorem{proposition}{Proposition}
\newtheorem{thm}{Theorem}[section]
\newtheorem{cor}[thm]{Corollary}
\newtheorem{lem}[thm]{Lemma}

\newtheorem{rem}[thm]{Remark}

\newtheorem{theorem}{Theorem}[section]
\newtheorem{example}[theorem]{Example}
\newtheorem{lemma}[theorem]{Lemma}
\newtheorem{proposition}[thm]{Proposition}
\newtheorem{aspt}[theorem]{Assumption}

\usepackage{appendix}
\usepackage{amsmath,amssymb,amsfonts,amsthm, amsbsy, mathtools, epsfig}
\usepackage[usenames]{color}
\usepackage{rotating}
\usepackage{boxedminipage}
\usepackage[colorlinks,citecolor=blue,urlcolor=blue]{hyperref}
\usepackage{epstopdf}
\usepackage{enumerate}

\usepackage{verbatim}
\usepackage{graphicx}
\usepackage{caption}
\usepackage{subcaption}
\usepackage{bbm}
\usepackage{xcolor}
\usepackage{bm}
\usepackage{mathrsfs,color,dsfont}
\usepackage{algorithm}

\usepackage{thmtools}
\usepackage{thm-restate}
\usepackage{hyperref}
\usepackage{cleveref}

\usepackage{algpseudocode}
\algnewcommand{\Inputs}[1]{%
  \State \textbf{Inputs:}
  \Statex \hspace*{\algorithmicindent}\parbox[t]{.8\linewidth}{\raggedright #1}
}
\algnewcommand{\Initialize}[1]{%
  \State \textbf{Initialize:}
  \Statex \hspace*{\algorithmicindent}\parbox[t]{.8\linewidth}{\raggedright #1}
}
\algnewcommand{\Outputs}[1]{%
  \State \textbf{Outputs:}
  \Statex \hspace*{\algorithmicindent}\parbox[t]{.8\linewidth}{\raggedright #1}
}
\usepackage{todonotes}

\newcommand{\R}{{\mathbb R}}

\newcommand{\T}{{{T}}}

\newcommand{\dcon}{C_d}

\newcommand{\RNum}[1]{\uppercase\expandafter{\romannumeral #1\relax}}
\makeatother
\newcommand{\argmin}{\mathop{\mathrm{arg\,min}{}}}

\newcommand{\tr}{\text{tr}}

\DeclareFontFamily{U}{mathx}{\hyphenchar\font45}
\DeclareFontShape{U}{mathx}{m}{n}{
      <5> <6> <7> <8> <9> <10>
      <10.95> <12> <14.4> <17.28> <20.74> <24.88>
      mathx10
      }{}
\DeclareSymbolFont{mathx}{U}{mathx}{m}{n}
\DeclareFontSubstitution{U}{mathx}{m}{n}
\DeclareMathAccent{\widecheck}{0}{mathx}{"71}
\DeclareMathAccent{\wideparen}{0}{mathx}{"75}

\def\argmin{\mathop{\rm argmin}}

\newcommand\tabcaption{\def\@captype{table}\caption}

\newcommand{\unit}{\mathds{1}}
\definecolor{DSgray}{cmyk}{0,1,0,0}

\newcommand{\E}{\mathbb{E}}
\newcommand{\Ltwo}[1]{\| #1 \|_2}
\newcommand{\opnorm}[1]{\left\| #1 \right\|}
\newcommand{\trnorm}[1]{\left|\tr( #1) \right|}

\newcommand{\norm}[1]{\left\| #1 \right\|}
 
\newcommand{\cN}{\mathcal{N}}

\begin{document}

\title{Statistical Inference for Model Parameters in Stochastic Gradient Descent}
\author{Xi Chen\footnote{Stern School of Business, New York University, Email: xichen@nyu.edu} ~ Jason D. Lee\footnote{Marshall School of Business, University of Southern California, Email: jasondlee88@gmail.com} ~  Xin T. Tong\footnote{Department of Mathematics, National University of Singapore, Email: mattxin@nus.edu.sg}  ~and~ Yichen Zhang\footnote{Stern School of Business, New York University, Email: yzhang@stern.nyu.edu}}
\date{}

\maketitle

\begin{abstract}\label{abstract}
The stochastic gradient descent (SGD) algorithm has been widely used in statistical estimation for large-scale data due to its computational and memory efficiency.
While most existing works focus on the convergence of the objective function or the error of the obtained solution, we investigate the problem of statistical inference of true model parameters based on SGD when the population loss function is strongly convex and satisfies certain smoothness conditions.
Our main contributions are two-fold. First, in the fixed dimension setup, we propose two consistent estimators of the asymptotic covariance of the average iterate from SGD: (1) a plug-in estimator, and (2) a batch-means estimator, which is computationally more efficient and only uses the iterates from SGD. Both proposed estimators allow us to construct asymptotically exact confidence intervals and hypothesis tests.
Second, for high-dimensional linear regression, using a variant of the SGD algorithm, we construct a debiased estimator of each regression coefficient that is asymptotically normal. This gives a one-pass algorithm for computing both the sparse regression coefficients and confidence intervals, which is computationally attractive and applicable to online data.
\end{abstract}

\section{Introduction}
Estimation of model parameters by minimizing an objective function is a fundamental idea in statistics. Let $x^* \in \mathbb{R}^d$ be the true $d$-dimensional model parameters.  
In common models, $x^*$ is the minimizer of a convex objective function $F(x): \R^d \rightarrow \R$, i.e.,
\begin{equation}
\label{eq:objective}
x^*=\argmin  \left(F(x): = \E_{\zeta \sim \Pi} f(x, \zeta) = \int f(x, \zeta) \mathrm{d} \Pi(\zeta)\right),
\end{equation}
where $\zeta$ denotes the random sample from a probability distribution $\Pi$ and $f(x, \zeta)$ is the loss function.

A widely used optimization method for minimizing $F(x)$ is the \emph{stochastic gradient descent (SGD)}, which has a long history in optimization (see, e.g., \cite{robbins1951,PJ92,nemirovski2009robust}). In particular, let $x_0$ denote any given starting point. SGD is an iterative algorithm, where the $i$-th iterate $x_i$ takes the following form,
\begin{align}
\label{eq:sgd_iter}
  x_i =x_{i-1} -\eta_i \nabla f(x_{i-1}, \zeta_i).
\end{align}
The step size $\eta_i$ is a decreasing sequence in $i$, $\zeta_i$ is the $i$-th sample randomly drawn from the distribution $\Pi$, and  $\nabla f(x_{i-1}, \zeta_i)$ denotes the gradient of $f(x, \zeta_i)$ with respect to $x$ at $x=x_{i-1}$. The algorithm outputs either the last iterate $x_n$, or the average iterate
$\bar{x}_n = \frac{1}{n} \sum_{i=1}^n x_i$
as the solution to the optimization problem in \eqref{eq:objective}. When $\bar{x}_n$ is adopted as the solution, the algorithm is referred to as the averaged SGD (ASGD), and such an averaging step is known as the Polyak-Ruppert averaging \cite{Ruppert88,PJ92}. SGD has many computational and storage advantages over traditional deterministic optimization methods. First, SGD only uses \emph{one pass} over the data and the per-iteration time complexity of SGD is $O(d)$, which is independent of the sample size. Second, there is no need for SGD to store the dataset, and thus SGD naturally fits in the \emph{online setting}, where each sample arrives  sequentially (e.g., search queries or transactional data).
Moreover, ASGD is known to achieve the optimal convergence rate in terms of $\E(F(\bar{x}_n)-F(x^*))$ with the rate of $O(1/n)$ \cite{Rakhlin:12} when $F(x)$ is smooth and strongly convex.
It has become the prevailing optimization method for many machine learning tasks \cite{Srebro:10}, e.g., training deep neural networks.

Based on the simple SGD template in \eqref{eq:sgd_iter}, there are a large number of variants developed in the optimization and statistical learning literature. Most existing works only focus on the convergence in terms of the objective function or the distance between the obtained solution and the true minimizer $x^*$ of \eqref{eq:objective}. However, the statistical inference (e.g., constructing confidence intervals) for each coordinate of $x^*$ based on SGD has largely remained unexplored. \label{page:infer} Inference is a core topic in statistics \label{inference_core}  and a confidence interval has been widely used to quantify the uncertainty in the estimation of model parameters. In this paper, we propose computationally efficient methods to conduct the inference for each coordinate of $x^*_j$ for $j=1,2,\dots,d$ based on SGD. With the developed techniques, one can test if $x_j^*=c$ for any number $c$, and tell a range of values that $x^*_j$ lies within it with a certain probability. These objectives cannot be achieved by deriving deviation inequalities or  generalization error bounds (see Section \ref{sec:relatedworks} for details).

The proposed methods are built  on a classical result of ASGD, which characterizes the limiting distribution of $\bar{x}_n$.  In particular, let $A=\nabla^2 F(x^*)$ be the Hessian matrix of $F(x)$ at $x=x^*$ and $S$ be the covariance matrix of $\nabla f(x^*, \zeta)$, i.e.,
\begin{equation}\label{eq:S_general}
S = \mathbb{E}\left([\nabla f(x^*,\zeta)][\nabla f(x^*,\zeta)]^\top  \right).
\end{equation}
Note that $ \mathbb{E}\nabla f(x^*,\zeta)=\nabla F(x^*)=0$, provided the interchangeability of derivative and expectation.  \cite{Ruppert88} and \cite{PJ92} showed that when $d$ is fixed and $F$ is strongly convex with a Lipschitz gradient, by choosing appropriately diminishing step sizes, $\sqrt{n}(\bar{x}_n- x^*)$ converges in distribution to a multivariate normal random vector, i.e.,
\begin{equation}\label{eq:asy_normal}
\sqrt{n}(\bar{x}_n- x^*) \Rightarrow  \mathcal{N}(0, A^{-1} S A^{-1}).
\end{equation}
However, this asymptotic normality result itself cannot be used to provide confidence intervals. To construct an asymptotically valid confidence interval (or equivalently, an asymptotically valid test that controls the type I error), we need to further construct a consistent estimator of the asymptotic covariance of $\sqrt{n}\bar{x}_n$, i.e.,  $A^{-1} S A^{-1} $. \label{page:asy}
The standard covariance estimator simply estimates $A$ and $S$ by their sample versions, and replaces the $x^*$ in $A$ and $S$ by $\bar{x}_{n}$. However, this standard estimator cannot be constructed in an online fashion. In other words, all the data is required to be stored to compute this estimator since $\bar{x}_n$ can only be known when the SGD procedure terminates. This requirement loses the advantage of SGD in terms of data storage.
\label{page:naive}

To address  this challenge, we propose two approaches to estimate $A^{-1}SA^{-1}$ without the need of storing the data. The first approach is the \emph{plug-in} estimator. In particular, we  propose a thresholding estimator $\widetilde{A}_n$ of $A$ based on the sample estimate $A_n=\frac{1}{n} \sum_{i=1}^n \nabla^2 f(x_{i-1}, \zeta_i)$. Note that this is not the standard sample estimate since each term $\nabla^2 f(x_{i-1},\zeta_i)$ is regarding different SGD iterates $x_{i-1}$ (in contrast to a single $\bar{x}_n$) and thus can be constructed online. This construction facilitates the online computation of $A_n$, which does not need to store each $x_i$ and $\zeta_i$. Together with the sample estimate $S_n$ of $S$, the asymptotic covariance $A^{-1} S A^{-1}$ is estimated by $\widetilde{A}_n^{-1} S_n \widetilde{A}_n^{-1}$, which is proven to be a consistent estimator (see Theorem \ref{thm:plugin}).

However, the plug-in estimator requires the computation of the Hessian matrix of the loss function $f$ and its inverse, which is usually not available for legacy codes where only the SGD iterates are available. Now a natural question arises: can we estimate the asymptotic covariance \emph{only using the iterates from SGD without requiring additional information}? We provide an affirmative answer to this question by proposing a computationally efficient \emph{batch-means} estimator. Basically, we split the sequence of SGD iterates $\{x_1, x_2, \ldots, x_n\}$ into $M+1$ batches with batch size $n_0,n_1,\dots,n_M$. The $0$-th batch is discarded since the iterates in that are far from the optimum. The batch-means estimator is a ``weighted" sample covariance matrix that treats each batch-means as a sample.

\label{para:mcmc} The idea of batch-means estimator can be traced to Markov Chain Monte Carlo (MCMC), where the batch-means method with equal batch size (see, e.g., \cite{Glynn90BM, GLYNN1991431, Damerdji:91, geyer1992, Fishman96MC, Jones:06, FlegalJo10}) is widely used for variance estimation in a time-homogeneous Markov chain. The SGD iterates in \eqref{eq:sgd_iter} indeed form a Markov chain, as $x_{i}$ only depends on $x_{i-1}$. However, since the step size sequence $\eta_i$ is a diminishing sequence,  it is a \emph{time-inhomogenous} Markov chain.
Moreover, the asymptotic behavior of SGD and MCMC are fundamentally different: while the former converges to the optimum, the latter travels ergodically inside the state space. As a consequence of these important differences,  previous literature on batch-means methods is not applicable to our analysis. To address this challenge, our new batch-means method constructs batches of \emph{increasing sizes}. The sizes of batches are chosen to ensure that the correlation decays appropriately among far-apart batches, so that far-apart batch-means can be roughly treated as independent. In Theorem \ref{thm:general}, we prove that the proposed batch-means method is a consistent estimator of the asymptotic covariance.  Further, we believe this new batch-means algorithm with increasing batch sizes is of independent interest since it can be used to estimate the covariance structure of other time-inhomogeneous Markov chains.

As both the plug-in and the batch-means estimator provide asymptotically exact confidence intervals, each of them has its own advantages:
\begin{enumerate}
  \item The plug-in estimator has a faster convergence rate than the batch-means estimator (see Theorem \ref{thm:plugin} and Corollary \ref{cor:batch}).
  \item The plug-in estimator requires the computation of the Hessian matrix of the loss function and its inverse, which can be expensive to obtain for many applications. The batch-means estimator does not require computing any of them. To establish the consistency result, the plug-in estimator requires an additional Lipschitz condition over the Hessian matrix of the loss function (see Assumption \ref{aspt:third}).
  \item The plug-in estimator directly computes the entire estimator $\widetilde{A}_n^{-1} S_n \widetilde{A}_n ^{-1}$ for the purpose of estimating diagonal elements of $A^{-1} S A^{-1}$. Furthermore, when $d$ is large, storing $\widetilde{A}_n$ and $S_n$ requires $O(d^2)$ bits, which is wasteful since only  estimates of the diagonal elements of $A^{-1}SA^{-1}$ are useful for the inference of each $x_j^*$ for $j=1,2,\dots,d$. Meanwhile, the batch-means estimator is able to merely compute and store diagonals.
\end{enumerate}
\label{para:comp}
Practitioners may decide to choose between the plug-in and batch-means estimators based on their tasks and computing resources. The plug-in estimator has a faster convergence rate, which leads to better performance in practice. However, in some cases when the computation and storage are limited, the batch-means estimator is able to provide an asymptotically exact confidence interval with  comparably good performance. Furthermore, the computation of the Hessian matrix in the plug-in estimator is an ``intrusive'' requirement for SGD \cite{Sul15}, i.e., it is not available for legacy codes where only the SGD iterates are computed. For example, if one has already obtained SGD iterates and wants to compute confidence intervals afterward, a non-instructive method like batch-means can be directly applied. Such a non-intrusive method that can operate with black-box SGD iterates is more desirable and welcomed by practitioners, as it only uses the existing SGD iterates without the need to change the original SGD code.

For the second part of our contribution, we further study the problem of confidence interval construction for $x^*$ in high-dimensional linear regression based on SGD, where the dimensionality $d$ can be much larger than the sample size $n$. 
In a high-dimensional setup, it is natural to solve an $\ell_1$-regularized problem,
$
  \min_{x} \; F(x)+ \lambda \|x\|_1,
$
where $F(x)$ is defined in \eqref{eq:objective}. \label{page:prox} A popular approach to solve it is the \emph{proximal stochastic gradient approach} (see, e.g., \cite{Lan12OSGD} and references therein). However, due to the proximal operator (i.e., the soft-thresholding operator for $\ell_1$-regularized problem), the distribution of the average iterate $\bar{x}_n$ no longer converges to a multivariate normal distribution. To address this challenge, we use the recently proposed RADAR algorithm \cite{agarwal2012stochastic}, which is a variant of SGD, together with the debiasing approach  \cite{Zhang14CI,vandegeer14asy,javanmard2014confidence}. The standard debiasing method relies on solving $d$ convex optimization problems (e.g., node-wise Lasso in \cite{vandegeer14asy}) to construct an approximation of the inverse of the design covariance matrix. \label{para:complexity} Each deterministic optimization problem requires a per-iteration complexity $O(nd)$, which is prohibitive when $n$ is large. In contrast, we adopt the stochastic RADAR algorithm to solve these optimization problems, where each problem only requires one pass of the data with the per-iteration complexity $O(d)$. Moreover, since the resulting approximate inverse covariance matrix from the stochastic RADAR is not an exact solution of the corresponding optimization problem, the analysis of \cite{vandegeer14asy}, which heavily relies on the KKT condition, is no longer applicable. We provide a new analysis to establish the asymptotic normality of the obtained estimator of $x^*$ from the stochastic optimization algorithm.

\subsection{Some related works on SGD}

\label{sec:relatedworks}
There is a large body of literature on stochastic gradient approaches and their applications to statistical learning problems (see, e.g., \cite{zhang2004solving, nesterov2008confidence, Xiao:10, Lan12OSGD, Bach:12, agarwal2012stochastic,Lin14SVRG} and references therein).  Most works on SGD focus on the convergence rate of the objective function instead of the asymptotic distribution of the obtained solution. Thus, we only review a few closely related works with results on distributions.

Back in 1960s,  \cite{fabian1968asymptotic} studied the distribution of SGD iterates. However, without averaging, the asymptotic variance is inflated and thus  the resulting statistical inference would have a reduced power even if the asymptotic is known. \cite{Ruppert88, PJ92, BM11} studied the averaged SGD  (ASGD) and established the asymptotic normality and efficiency of the estimators. However, these works do not discuss the estimation of the asymptotic covariance.

A few works in the SGD literature (e.g.,  \cite{nesterov2008confidence,nemirovski2009robust}) show large deviation results of $\Pr( \norm{\bar{x}_n - x^*}_2 >t) \le C(t)$ by combining the Markov inequality with the expected deviation of $\bar{x}_n$ to $x^*$. However, we note that large deviation results cannot be used to obtain asymptotically exact confidence intervals, which refer to the exact $1-q$ coverage as $n\rightarrow \infty$. That is, $\Pr(x^*\in\mathrm{CI}_q)\rightarrow 1-q$, where $\mathrm{CI}_q$ denotes the confidence interval. Deviation inequalities, which are unable to quantify the exact probability, fail to provide the exact $1-q$ coverage and will lead to wider confidence intervals. Moreover, note that the $\ell_2$ bounds in the SGD literature are  generally $O(\sigma \sqrt{\frac{d }{n}})$ (where $\sigma^2$ is the variance of the norm of the stochastic gradient) and do not imply a $\ell_\infty$ bound of size $O(\frac{\sigma}{\sqrt{n}})$, whereas a confidence interval for any single coordinate should  be $O( \frac{\sigma}{\sqrt{n}})$ (the $O(\cdot)$ notation here does not depend on $d$). Therefore, although $d$ is fixed, $\ell_2$-norm error bound results still lead to conservative confidence intervals. Instead, we will use the central limit theorem that shows $ \lim \Pr( |\bar{x}_{n,j} - x^* | < z_{q /2 } \sigma /\sqrt{n} ) \to 1-q$, where $z_{q/2}$ is the $(1-q/2)$-quantile of the standard normal distribution. This allows us to construct an asymptotically exact confidence interval.

We also note that \cite{toulis2014implicit} established the asymptotic normality for the \emph{averaged implicit SGD} procedure, which is an algorithm different from ASGD. Moreover, this paper does not discuss the estimation of the asymptotic covariance and thus their results cannot be directly used to obtain the confidence intervals.

\subsection{Notations and organization of the paper}
\label{sec:notation}
As a summary of notations, throughout the paper, we use $\|x\|_p$ to denote the vector $\ell_p$-norm of $x$, $\|x\|_0$ the number of non-zero entries in $x$, $\|X\|$ the matrix operator norm of $X$ and $\|X\|_{\infty}$ the element-wise $\ell_\infty$-norm of $X$ (i.e., $\|X\|_{\infty}=\max_{i,j} |X_{ij}|$). For a square matrix $X$, we denote its trace by $\mathrm{tr}(X)$.
For a positive semi-definite (PSD) matrix $A$, let $\lambda_{\max}(A)$ and $\lambda_{\min}(A)$  be its maximum and minimum eigenvalue. For a vector $a$ of length $d$ and any index subset $J \subseteq\{1 ,\ldots, d\}$, we denote by $a_J$ the sub-vector of $a$ with the elements indexed by $J$ and $a_{-J}$ the sub-vector of $a$ with the elements indexed by $\{1, \ldots, d \} \backslash J$. Similarly, for a $d_1 \times d_2$ matrix $X$ and two index subsets $R \subseteq \{1, \ldots, d_1\}$ and $J \subseteq \{1, \ldots, d_2\}$, we denote by $X_{R,J}$ the $|R| \times |J|$ sub-matrix of $X$ with elements in rows in $R$ and columns in $J$. When $R = \{1, \ldots, d_1\}$ or $J =\{1, \ldots, d_2\}$, we denote $X_{R,J}$ by $X_{\cdot,J}$ or $X_{R,\cdot}$, respectively. We use $I$ to denote the identity matrix. The function $\Phi(\cdot)$ denotes the CDF of the standard normal distribution.

For any
sequences $\{a_n\}$ and $\{b_n\}$ of positive numbers, we write $a_n \gtrsim b_n$ if $a_n \geq c b_n$ holds for all $n$ large enough and some constant $c>0$, $a_n \lesssim b_n$ if $b_n  \gtrsim a_n$ holds, and $a_n \asymp b_n$ if $a_n \gtrsim b_n$ and $a_n \lesssim b_n$.

The rest of the paper is organized as follows. In Section \ref{sec:setup_back}, we provide more background of SGD and detailed results from \cite{PJ92}.  In Section \ref{sec:setup_assump}, we provide the assumptions and some error bounds on SGD iterates.
In Section \ref{sec:cov}, we propose the plug-in estimator and batch-means estimator for estimating the asymptotic covariance of $\bar{x}_n$ from ASGD. In Section \ref{sec:high_dim}, we discuss how to conduct inference for high-dimensional linear regression. In Section \ref{sec:exp}, we demonstrate the proposed methods by simulated experiments. Further discussions appear in Section \ref{sec:discuss} and all proofs are given in the Appendix.

\section{Background}
\label{sec:setup_back}

In the classical work of \cite{PJ92}, the SGD method was introduced in a form equivalent with \eqref{eq:sgd_iter} to facilitate the analysis. In particular, the iteration is given by
\begin{equation}\label{eq:sgd_general}
x_n=x_{n-1}-\eta_n \nabla F (x_{n-1})+\eta_n\xi_n,
\end{equation}
where $\xi_n:=\nabla F(x_{n-1})-\nabla f(x_{n-1},\zeta_n)$. The formulation \eqref{eq:sgd_general} decomposes the descent into two parts: $\nabla F(x_{n-1})$ represents the direction of population gradient which is the major driving force behind the convergence of SGD, and $\xi_n$ is a martingale difference sequence under Assumption \ref{aspt:martingale} (see below). That is,
$
\mathbb{E}_{n-1}[\xi_n] = \nabla F(x_{n-1})-\mathbb{E}_{n-1}\nabla f(x_{n-1},\zeta_n)=0.
$
Here and in the sequel,  $\E_n(\cdot)$ denotes the conditional expectation $\E(\cdot | \mathcal{F}_n)$, where $\mathcal{F}_{n}$ is the $\sigma$-algebra generated by $\{\zeta_1, \ldots, \zeta_n\}$ ($\zeta_k$ is the $k$-th sample). Let $\Delta_{n}:=x_n-x^*$ be the error of the $n$-th iterate. It it noteworthy that by subtracting $x^*$ from both sides of \eqref{eq:sgd_general}, the recursion \eqref{eq:sgd_general} is equivalent to
\begin{equation}\label{eq:sgd_general_delta}
\Delta_n=\Delta_{n-1} -\eta_n \nabla F(x_{n-1}) +\eta_n \xi_n,
\end{equation}
which will be extensively used throughout the paper.

Given the SGD recursion in the form of \eqref{eq:sgd_general_delta} and under suitable assumptions (see Section \ref{sec:setup_assump} below), Theorem 2 of \cite{PJ92} shows that when the step size sequence $\eta_i = \eta i^{-\alpha}, i=1,2,\dots, n$ with $\alpha \in (1/2,1)$, we have
\begin{equation}\label{eq:sandwich}
\sqrt{n} \cdot \overline{\Delta}_n\Rightarrow \mathcal{N}(0, A^{-1}S A^{-1}),\quad \text{if}\quad \alpha\in (\tfrac12, 1)
\end{equation}
where $\overline{\Delta}_n=\frac{1}{n} \sum_{i=1}^n \Delta_i= \bar{x}_n - x^*$. Based on this limiting distribution result, we only need  to estimate the asymptotic covariance  matrix $A^{-1}S A^{-1}$. Then we can form the confidence interval $\bar{x}_{n,j} \pm z_{q/2} \hat \sigma_{jj}/\sqrt{n}$, where $ \hat \sigma_{jj}$ is a consistent estimator of $(A^{-1}S A^{-1})_{jj}^{1/2}$ and $z_{q/2}$ is the $(1-q/2)$-quantile of the standard normal distribution (i.e., $z_{q/2}=\Phi^{-1}(1-q/2)$ and $\Phi(\cdot)$ is the CDF of the standard normal distribution). Therefore, the main purpose of the paper is to provide consistent estimators of the asymptotic covariance  matrix.  

\begin{rem}\label{rem:cramer}
In the model well-specified case, $\overline{x}_n$ is an asymptotically efficient estimator of the true model parameter $x^*$ according to \eqref{eq:sandwich}. In particular, suppose $\zeta$ comes from the probability distribution $\Pi$ with density $p_{x^*}(\zeta)$  parameterized by $x^*$. If the loss function $f(x,\zeta)=-\log p_x(\zeta)$ is the negative log-likelihood, under certain regularity conditions, one can show that
\begin{align*}
A=\nabla^2\E\left[-\log p_{x^*}(\zeta)\right]=\E\left(-\nabla\log p_{x^*}(\zeta)\right)\left(-\nabla\log p_{x^*}(\zeta)\right)^\top =S=I(x^*)
\end{align*}
Here $I=I(x^*)$ is the Fisher information matrix. Therefore, the limiting covariance matrix $A^{-1}SA^{-1}=I^{-1}$ achieves the Cram\'er-Rao lower bound, which indicates that $\bar{x}_n$ is asymptotically efficient.
It is worth noting that the asymptotic normality result \eqref{eq:sandwich} does not require that the model is well-specified. In a model mis-specified case, the asymptotic distribution of $\bar{x}_n$ is centered at $x^*$, where $x^*$ is the unique minimizer of $F(x)$ and the asymptotic covariance $A^{-1} S A^{-1}$ is of the so-called ``sandwich covariance'' form (e.g., see \cite{buja2013conspiracy}).
\end{rem}

To illustrate this SGD recursion in \eqref{eq:sgd_general_delta} and the form of $A$ and $S$, we consider the following two motivating examples.

\label{sec:2example}
\begin{example}[Linear Regression]
\label{exp:linear}
Under the classical linear regression setup, let the $n$-th sample be $\zeta_n = (a_n, b_n)$, where the input $a_n \in \R^d$ is a sequence of random vectors independently drawn from the same multivariate distribution and the response $b_n \in \R$ follows a linear model,
$
  b_n = a_n^\top  x^* + \varepsilon_n.
$
Here $x^* \in \R^d$ represents the true parameters of the linear model, and $\{\varepsilon_n\}$ are independently and identically distributed (\emph{i.i.d.}) centered random variables, which are uncorrelated with $a_n$. For simplicity, we assume $a_n$ and $\varepsilon_n$ have all moments being finite.
Given $\zeta_n = (a_n, b_n)$, the loss function at $x$ is a quadratic one:
\[
f(x, \zeta_n)=\frac{1}{2} (a_n^\top  x -b_n)^2.
\]
and the true parameters $x^*=\argmin_x \left(F(x):=\E f(x, \zeta)\right)$.
Given the loss function, the SGD iterates in \eqref{eq:sgd_iter} become,
$
x_{n} = x_{n-1} - \eta_n a_n(a^\top _nx_{n-1} - b_n).
$
This can also be written in the form of \eqref{eq:sgd_general} as
\[
x_n=x_{n-1}-\eta_n A\Delta_{n-1}+\eta_n\xi_n,\quad \xi_n:=(A-a_na_n^\top )\Delta_{n-1}+a_n\varepsilon_n,
\]
where $A=\mathbb{E}a_na_n^\top $  is the population gram matrix of $a_n$. It is easy to find that
\[
F(x)=\frac{1}{2} (x-x^*)^\top A(x-x^*)+\mathbb{E}\varepsilon^2,
\]
which implies that $\nabla F(x)=A(x-x^*)$ and $\nabla^2 F(x)=A$ for all $x$. As for matrix $S$, it is given by
$
S:=\mathbb{E}\left([\nabla f(x^*,\zeta)][\nabla f(x^*,\zeta)]^\top  \right)=\mathbb{E}\varepsilon^2_n a_na_n^\top .
$
\end{example}

\begin{example}[Logistic Regression]
\label{exp:logistic}
One of the most popular applications for general loss in statistics is the logistic regression for binary classification problems. In particular, the logistic model assumes that the binary response $b_n \in \{-1,1\}$ is generated by the following probabilistic model,
\[
 \Pr(b_n | a_n) =\frac{1}{1+\exp\left(-b_n \langle a_n, x^*\rangle\right)},
\]
where $a_n$ is an \emph{i.i.d.} sequence. The population objective function is given by
$
 F(x)=\E f(x,\zeta_n)= \E \log (1+ \exp\left(-b_n \langle a_n, x\rangle\right)).
$
Let $\varphi(x):= \frac{1}{1+\exp(-x)}$ denote the sigmoid function,  we have
$
 \nabla f(x,\zeta_n) =  - \varphi(-b_n \langle a_n, x \rangle) b_n a_n.
$
Moreover, we have the formulation of matrix $A$ and $S$ as
\begin{equation}\label{eq:logA}
A=S=\mathbb{E}\frac{a_n a_n^\top }{[1+\exp(\langle a_n, x^* \rangle)][1+\exp(-\langle a_n, x^* \rangle)]}.
\end{equation}
\end{example}

\section{Assumptions and Error Bounds}
\label{sec:setup_assump}
In this section, we provide the assumptions used in the fixed-dimensional case and then provide some  useful error bounds on $\Delta_n$.
We first make the following standard assumption on the population loss function $F(x)$.
\begin{aspt}[Strong convexity and Lipschitz continuity of the gradient]
\label{aspt:convexity}
Assume that the objective function $F(x)$ is continuously differentiable and  strongly convex with parameter $\mu>0$, that is, for any $x_1$ and $x_2$,
\[
F(x_2) \geq  F(x_1)+\langle\nabla F(x_1), x_2-x_1 \rangle+\frac{\mu}{2}\Ltwo{x_1-x_2}^2.
\]
Further, assume that $\nabla^2 F(x^*)$ exists, and $\nabla F(x)$ is Lipschitz continuous with a constant $L_F$, i.e., for any $x_1$ and $x_2$,
$
\Ltwo{\nabla F(x_1)-\nabla F(x_2)} \leq L_F \Ltwo{x_1-x_2}.
$
\end{aspt}

Note that the strong convexity of $F(x)$ was adopted by \cite{PJ92} (see Assumption 4.1 in \cite{PJ92}) to derive the limiting distribution of averaged SGD, which serves as the basis of our work.
 In fact, the strong convexity of $F(x)$ implies  $\lambda_{\min} (A) = \lambda_{\min}(\nabla^2 F(x^*))\geq \mu$  is an important condition for parameter estimation and inference.   There are recent works in optimization on relaxing the strong convexity assumption (e.g., \cite{bach2013non}), but they were only able to obtain fast convergence rates in terms of  the objective value $F(\bar{x}_n) -F(x^*)$.

We further assume that the martingale difference $\xi_n$ satisfies the following conditions.
\begin{aspt}
\label{aspt:martingale} The following hold for the sequence $\xi_n=\nabla F(x_{n-1})-\nabla f(x_{n-1},\zeta_n)$:
\begin{enumerate}
\item Assume that $f(x,\zeta)$ is continuously differentiable in $x$ for any $\zeta$ and $ \|\nabla f(x,\zeta)\|_2$ is uniformly integrable for any $x$ so that $\E_{n-1} \xi_n=0$.

\item The conditional covariance of $\xi_n$ has an expansion around $x=x^*$:
$
 \mathbb{E}_{n-1} \xi_n\xi_n^\top =S+\Sigma(\Delta_{n-1}),
$
and there exists  constants $\Sigma_1$ and $\Sigma_2>0$ such that for any $\Delta \in \mathbb{R}^d$.
\[
\opnorm{\Sigma(\Delta)} \leq \Sigma_1\Ltwo{\Delta}+\Sigma_2\Ltwo{\Delta}^2,\quad
|\mathrm{tr}(\Sigma(\Delta))| \leq  \Sigma_1\Ltwo{\Delta}+\Sigma_2\Ltwo{\Delta}^2.
\]
Note that $S$ is the covariance matrix of $\nabla f(x^*,\zeta)$ defined in \eqref{eq:S_general}.
\item There exists  constants $\Sigma_3, \Sigma_4$ such that the fourth conditional moment of $\xi_n$ is bounded by
$
\mathbb{E}_{n-1}\|\xi_n\|_2^4\leq\Sigma_3+\Sigma_4\|\Delta_{n-1}\|_2^4.
$
\end{enumerate}
\end{aspt}
For part 1, we note that our assumption on $f(x, \zeta)$ guarantees that Leibniz's integration rule holds, i.e.,
$
\E_{\zeta\sim \Pi} \nabla f(x,\zeta)= \nabla F(x)$ for all $x$.
Therefore, we have $\E_{n-1} \xi_n=0$, which implies that $\xi_n$ is a martingale difference sequence.
Assumption \ref{aspt:martingale} is  a mild condition over the regularity and boundedness of the loss function. In fact, one can easily verify Assumption \ref{aspt:martingale} using the following lemma.
\begin{lem}
\label{lem:simplegeneral}
If there is a function $H(\zeta)$ with bounded fourth moment, such that the Hessian of $f(x,\zeta)$  is bounded by
\begin{equation*}
\opnorm{\nabla^2 f(x,\zeta)}\leq H(\zeta)
\end{equation*}
for all $x$, and  $\nabla f(x^*,\zeta)$ have a bounded fourth moment, then Assumption \ref{aspt:martingale} holds with
$\Sigma_1=2 \sqrt{\mathbb{E}\Ltwo{\nabla f(x^*,\zeta)}^2 \mathbb{E}H(\zeta)^2}$, $\Sigma_2=4  \mathbb{E}H(\zeta)^2$,
$\Sigma_3=8\mathbb{E}\Ltwo{\nabla f(x^*,\zeta)}^4$ and $\Sigma_4= 64  \mathbb{E}H(\zeta)^4$ .
\end{lem}

Although we consider the fixed-dimensional case, it is still of practical interest to investigate the dimension dependence in our results. The dimension dependence is rather complicated since our results involve a number of constants in Assumption \ref{aspt:convexity} and \ref{aspt:martingale} that all depend on the dimension $d$ (e.g., $L_F$, $\Sigma_1$, $\Sigma_2$, $\Sigma_3$, $\Sigma_4$, $\mathrm{tr}(S)$). For example,  $\mathrm{tr}(S)$ grows with $d$. Moreover, the way it grows depends on how $S$ is configured.  Therefore, for the ease of presentation, we define the following quantity
\begin{equation}
\label{eqn:dcon}
\dcon:=\max\left\{L_F, \Sigma_1^{\frac{2}{3}},\sqrt{\Sigma_2},\sqrt{\Sigma_3},\Sigma^{\frac{1}{4}}_4, \mathrm{tr}(S)\right\}.
\end{equation}
In both linear and logistic regression, $C_d$ increases linearly in $d$ (see Appendix \ref{sec:verify}). We will state our results in terms of this single quantity $C_d$. We also assume $\|x_0-x^*\|_2^2=O(\dcon)$, and there is a universal constant $c$ such that the  step size satisfies $\eta_i \dcon\leq c \mu$ for all $i$.
Note that the choice of step sizes does not sacrifice much of generality since when $d$ is a constant, we could always ignore the first a few iterations, which is usually considered as the ``burn in'' stage. Also, for the starting point $x_0$, if all the components of $x_0-x^*$ are bounded by a constant, it naturally satisfies $\|x_0-x^*\|_2^2=O(d)$.

In the sequel, we will impose Assumptions \ref{aspt:convexity} and \ref{aspt:martingale}. In Appendix \ref{sec:verify}, we show that Assumptions \ref{aspt:convexity} and \ref{aspt:martingale} hold on our motivating examples of linear and logistic regression (see Examples \ref{exp:linear} and \ref{exp:logistic}). Under these assumptions, the classical works \cite{Ruppert88,PJ92} establish the asymptotic normality and efficiency of the $\bar{x}_n$ (see \eqref{eq:asy_normal} and Remark \ref{rem:cramer}). Moreover, we could obtain the following error bounds on the SGD iterates.

\begin{lem}
\label{lem:Delta}
Under Assumptions \ref{aspt:convexity} and \ref{aspt:martingale}, if the step size is chosen to be $\eta_n = \eta n^{-\alpha}$ with $\alpha \in (0,1)$, the iterates of error $\Delta_n=x_n-x^*$ satisfy the following.
\[
\mathbb{E} \|\Delta_n\|^k_2  \lesssim  n^{-k\alpha/2}(C^{k/2}_d+\|\Delta_{n_0}\|^k_2),\quad k=1,2,4.
\]
\end{lem}

The proof of Lemma \ref{lem:Delta} is provided in Appendix \ref{supp:Delta}. A result similar to Lemma \ref{lem:Delta} providing the convergence of $\|\Delta_n\|_2$ and  $\|\Delta_n\|^2_2$ has been shown in  \cite{BM11} (see Theorem 1 therein). Here, we provide simpler bounds on conditional moments of $\Delta_n$ and extend the results in \cite{BM11}  to the fourth moment bound, since we need to access the variance of a variance estimator. This result also tells us how the error decorrelates in terms of the number of iterations. Our proof strategy is similar to \cite{BM11} in that we setup up a recursive formula for the $\Delta_n$ term, and then show it decays at a certain rate by leveraging the convexity of $F(x)$. 
\label{para:convergence}
\label{supp:Delta}
\section{Estimators for Asymptotic Covariance}
\label{sec:cov}
Following the inference procedures illustrated above, when $d$ is fixed and $n\rightarrow\infty$, it is essential to estimate the asymptotic covariance matrix $A^{-1}S A^{-1}$. In this section, we will propose two consistent estimators, the plug-in estimator and the batch-means estimator.

\subsection{Plug-in estimator}
\label{sec:plugin}

The idea of the plug-in estimator is to separately estimate $A$ and $S$ by some $\widehat{A}$ and $\widehat{S}$ and use $\widehat{A}^{-1} \widehat{S} \widehat{A}^{-1}$ as an estimator of $A^{-1}SA^{-1}$. Since $x_i$ converges to $x^*$,  according to the  definitions of $A$ and $S$ in \eqref{eq:S_general}, an intuitive way to construct $\widehat{A}$ and $\widehat{S}$ is to use the sample estimate
\begin{eqnarray*}
  A_n: = \frac1n \sum_{i=1}^n \nabla ^2 f( x_{i-1},\zeta_i), \qquad S_n : = \frac1n \sum_{i=1}^n \nabla f(x_{i-1},\zeta_i) \nabla f(x_{i-1},\zeta_i)^\top  ,
\end{eqnarray*}
as long as the information of $\nabla^2 f(x_{i-1},\zeta_i)$ is available. It is worthwhile noting that each summand in $A_n$ and $S_n$ involves different $x_{i-1}$. Therefore, $A_n$ and $S_n$ can be computed in an online fashion without the need of storing all the data.

Since we are interested in estimating $A^{-1}$, it is necessary to avoid the possible singularity of $A_n$ from statistical randomness. Therefore, we propose to use thresholding estimator $\widetilde{A}_n$, which is strictly positive definite.  In particular, fix $\delta>0$, and let $\Psi D_n\Psi^\top $ be the eigenvalue decomposition of $A_n$, where $D_n$ is a non-negative diagonal matrix. We construct the thresholding estimator $\widetilde{A}_n$:
\begin{equation*}
\widetilde{A}_n=\Psi \widetilde{D}_n \Psi^\top , \quad \left(\widetilde{D}_n\right)_{i,i}=\max\left\{\delta, \left({D}_n\right)_{i,i}\right\}.
\end{equation*}
By construction, it is guaranteed that $\widetilde{A}_n$ is invertible.
With the construction of $S_n$ and $\widetilde{A}_n$ in place, we propose the \emph{plug-in estimator} as $\widetilde{A}_n^{-1} S_n \widetilde{A}_n^{-1}$. Our goal is to establish the consistency of the plug-in estimator, i.e.,
\[
\E \; \opnorm{ \widetilde{A}_n^{-1} S_n \widetilde{A}_n^{-1} - A^{-1} S A^{-1}} \longrightarrow  0 \quad \text{as} \;  n \rightarrow \infty.
\]
Since this estimator relies on the Hessian matrix of the loss function,  we need an additional assumption to establish the consistency.
\begin{aspt}
 \label{aspt:third}
There are constants $L_2$ and $L_4$ such that for all $x$, 
\begin{align}\label{eq:hessian}
\mathbb{E}\|\nabla^2 f(x,\zeta)-\nabla^2 f(x^*,\zeta)\| & \leq L_2\|x-x^*\|_2,\\
\|\mathbb{E}[\nabla^2 f(x^*,\zeta)]^2-A^2\| & \leq L_4. \nonumber
\end{align}
Moreover, we assume that for the choice of $\delta$, we have $\lambda_{\min}(A) > \delta$.
\end{aspt}
We note that it is easy to verify that \eqref{eq:hessian} holds for the two motivating examples in Section \ref{sec:2example}. For quadratic loss, the Hessian matrix at any $x$ is $A$ itself, and \eqref{eq:logA} gives the Hessian for the logistic loss, which is Lipschitz in $x$ and also bounded. In addition,  according to Assumption \ref{aspt:convexity}, we have $\lambda_{\min}(\nabla^2 F(x)) \geq \mu $ for any $x$ and thus $\lambda_{\min}(A) \geq \mu$. Therefore, a valid choice of $\delta$ satisfying Assumption \ref{aspt:third} always exists.

To track the dependence of our results on dimension, we assume $L_2$ and $L_4$ are also controlled by $\dcon$ in \eqref{eqn:dcon} as $L_2\lesssim \dcon^{3/2}, L_4\lesssim \dcon^2$. Lemmas \ref{lem:example1} and \ref{lem:example2} in the Appendix verify this requirement is satisfied in linear and logistic regression.

With this additional  assumption, we first establish  the consistency of the sample estimate $A_n$ and $S_n$ in the following lemma.
\begin{lem}\label{lem:A_con}\label{lem:S_con}
Under Assumptions \ref{aspt:convexity}, \ref{aspt:martingale} and \ref{aspt:third}, the followings hold
\[
\mathbb{E}\|A_n-A\|\lesssim \dcon^2 n^{-\frac\alpha2}, \quad \mathbb{E}\|S_n-S\|\lesssim \dcon^2 n^{-\frac\alpha2}+\dcon^3 n^{-\alpha},
\]
where $\alpha\in(0,1)$ is given {in the step size sequence $\eta_i=\eta i^{-\alpha}$, $i=1,2,\dots, n$}.
\end{lem}
The proof of Lemma \ref{lem:S_con} is provided in Appendix \ref{supp:A_con}.
Using Lemma \ref{lem:A_con} and a matrix perturbation inequality for the inverse of a  matrix (see Lemma \ref{lem:matrixinv} in Appendix \ref{subsec:plugin}), we obtain the consistency result of the proposed plug-in estimator $\widetilde{A}_n^{-1} S_n \widetilde{A}_n^{-1}$:
\begin{theorem}[Error rate of the plug-in estimator]\label{thm:plugin}
Under Assumptions \ref{aspt:convexity}, \ref{aspt:martingale} and \ref{aspt:third}, the thresholded plug-in estimator initialized from any bounded $x_0$  converges to the asymptotic covariance matrix,
\begin{equation}\label{eq:plug_in_bound}
\mathbb{E}\opnorm{\widetilde{A}_n^{-1} S_n \widetilde{A}_n ^{-1} -A^{-1} S A^{-1}}\lesssim  \|S\|(\dcon^2 n^{-\frac{\alpha}{2}}+\dcon^3 n^{-\alpha}),
\end{equation}
where $\alpha \in (0,1)$ is given in the step size sequence $\eta_i=\eta i^{-\alpha}$, $i=1,2,\dots, n$. When $C_d$ is a constant, the right hand side of \eqref{eq:plug_in_bound} is dominated by $O(n^{-\frac{\alpha}{2}})$.
\end{theorem}

\begin{rem}\label{rem:directplugin}
In practice, we usually do not need to perform the thresholding step, since $A_n$ is positive definite with high probability as $A_n$ is close to $A$. The thresholding step is mainly for obtaining the expected error bound in Theorem \ref{thm:plugin}. In fact, without the thresholding step, we are still able to the obtain the following error bound. In our numerical experiments, we do not apply the thresholding procedure and the obtained $A_n$'s are always invertible.

\begin{cor}
\label{cor:directplugin}Under Assumptions \ref{aspt:convexity}, \ref{aspt:martingale} and \ref{aspt:third}, as $n\rightarrow \infty$, \[
\opnorm{A_n^{-1} S_n A_n ^{-1} -A^{-1} S A^{-1}}=O_p\big(\|S\|(\dcon^2 n^{-\frac{\alpha}{2}}+\dcon^3 n^{-\alpha})\big).
\]
\end{cor}
\end{rem}

We also note that since the element-wise $\ell_\infty$-norm is bounded from above by the matrix operator norm, we have
$
\mathbb{E}\max_{ij} \bigl|(\widetilde{A}_n^{-1} S_n \widetilde{A}_n ^{-1} -A^{-1} S A^{-1})_{ij}\bigr|
$
converges to zero as $n \rightarrow \infty$ according Theorem \ref{thm:plugin}.
Therefore $(A^{-1} S A^{-1})^{1/2}_{jj}$ can be estimated by $\hat{\sigma}^P_{n,j}=(\widetilde{A}_n^{-1} S_n \widetilde{A}_n ^{-1})^{1/2}_{jj}$  for the construction of confidence intervals. In particular, we have the following corollary, which shows that $\bar{x}_{n,j}\pm  z_{q/2}\hat{\sigma}^P_{n,j}/\sqrt{n}$ is an asymptotic exact confidence interval.
\begin{restatable}{cor}{cor:pluginterval}
\label{cor:pluginterval}
Under the assumptions of Theorem \ref{thm:plugin}, {if the step size is chosen to be $\eta_i = \eta i^{-\alpha}$ with $\alpha\in (\frac12,1)$,} when $d$ is fixed and $n\to \infty$,
\[
\Pr\left(\bar{x}_{n,j}-z_{q/2}\hat{\sigma}^P_{n,j}/\sqrt{n}\leq x^*_j\leq \bar{x}_{n,j}+z_{q/2}\hat{\sigma}^P_{n,j}/\sqrt{n}\right)\to 1-q.
\]
\end{restatable}
Proof of Corollary \ref{cor:pluginterval} is given in Appendix \ref{sec:supp_pluginterval}. Note that while Theorem \ref{thm:plugin} holds for all $\alpha\in (0,1)$, the asymptotic normality in \eqref{eq:sandwich} holds only when $\alpha\in (\frac12,1)$. Thus, Corollary \ref{cor:pluginterval} requires that $\alpha\in (\frac12, 1)$.

\subsection{Batch-means estimator}
\label{sec:batchmean}
Although the plug-in estimator is intuitive, it requires the computation of the Hessian matrix and its inverse, as well as an additional Assumption \ref{aspt:third} on the Lipschitz condition of the Hessian matrix. In this section, we develop the batch-means estimator, which only uses the iterates from SGD without requiring computation of any additional quantities. Intuitively, if all iterates are independent and share the same distribution, the asymptotic covariance can be directly estimated by the sample covariance,
$
\frac{1}{n}\sum_{i=1}^n(x_i-\bar{x})(x_i-\bar{x})^\top .$
Unfortunately, the SGD iterates are far from independent. To  understand the correlation between two consecutive iterates, we note that for sufficiently large $n$ such that $x_{n-1}$ is close to $x^*$, by the Taylor expansion of $\nabla F(x_{n-1})$ at $x^*$, we have $\nabla F(x_{n-1}) \approx \nabla F(x^*)+\nabla^2 F(x^*) (x_{n-1} -x^*)=A \Delta_{n-1}$, where $\nabla F(x^*)=0$ by the first order condition and $A=\nabla^2F(x^*)$. Combining this with the recursion  in \eqref{eq:sgd_general_delta}, we have for  sufficiently large $n$,
\begin{equation}\label{eq:Delta_recursion}
  \Delta_n \approx (I_d- \eta_n A) \Delta_{n-1}  + \eta_n \xi_n.
\end{equation}
Based on \eqref{eq:Delta_recursion}, the strength of correlation between $\Delta_n$ and $\Delta_{n-1}$ can be approximated by $\opnorm{I_d-\eta_n A}$, which is very close to $1$ as $\eta_n\asymp n^{-\alpha}$. To address the challenge of strong correlation among neighboring iterates, we split the entire sequence of iterates into batches with carefully chosen batch sizes. 
In particular, we split $n$ iterates of SGD $\{x_1, \ldots, x_n\}$ into $M+1$ batches with  sizes $n_0, n_1, \ldots,  n_M$:
\[
\underbrace{\{x_{s_0},\ldots, x_{e_0}\}}_{\text{0-th batch}}, \; \underbrace{\{x_{s_1}, \ldots x_{e_1}\}}_{\text{1-st batch}}, \; \ldots,\; \underbrace{\{x_{s_M},\ldots, x_{e_M}\}}_{\text{$M$-th batch}}.
\]
Here $s_k$ and $e_k$ are the starting and ending index of $k$-th batch with $s_0=1$, $s_k=e_{k-1}+1$, $n_k=e_k-s_k+1$, and $e_M=n$. We treat the $0$-th batch as the ``burn-in stage". More precisely, the iterates $\{x_{s_0},\ldots, x_{e_0}\}$ will not be used for constructing the batch-means estimator since the step sizes are not small enough and the corresponding iterates in the $0$-th batch are far away from the optimum. The batch-means estimator is given by the following:
\begin{equation}\label{eq:batch_mean}
  \frac{1}{M}\sum_{k=1}^Mn_k(\overline{X}_{n_k}-\overline{X}_M)(\overline{X}_{n_k}-\overline{X}_M)^\top .
\end{equation}
where $\overline{X}_{n_k}:=\frac{1}{n_k} \sum_{i=s_k}^{e_k} x_i$ is the mean of the iterates for the $k$-th batch and $\overline{X}_M:=\frac{1}{e_M-e_0}\sum_{i=s_1}^{e_M} x_i$ is the the mean of all the iterates except for the $0$-th batch.

Note that when batch sizes $n_k$ are predetermined, we may rewrite \eqref{eq:batch_mean} in the following form,
\begin{equation}\label{eq:recursive}
\frac1M\sum_{k=1}^Mn_k \overline{X}_{n_k}\overline{X}_{n_k}^\top +\frac nM \overline{X}_M\overline{X}_M^\top -2\left(\frac1M\sum_{k=1}^Mn_k\overline{X}_{n_k}\right)\overline{X}_{M}^\top .
\end{equation}
Here, $\overline{X}_{M}$, $\frac1M\sum_{k=1}^M n_k \overline{X}_{n_k}\overline{X}_{n_k}$ and $\frac1M\sum_{k=1}^Mn_k\overline{X}_{n_k}$ can be updated recursively so that there is no need to store all the batch-means $\{\overline{X}_{n_k}\}$. In other words, the memory requirement for the batch-means estimator is only $O(d^2)$ instead of $O(Md^2)$. 

Intuitively, the reason why our batch-means estimator with increasing batch size can overcome the strong dependence between iterates is as follows. Although the correlation between neighboring iterates is strong, it decays exponentially for far-apart iterates. Roughly speaking, by \eqref{eq:Delta_recursion}, for large $j$ and $k$, the strength of correlation between $\Delta_j$ and $\Delta_k$ is approximately
\begin{equation}
\label{tmp:decor}
\prod_{i=j}^{k-1} \opnorm{I_d-\eta_{i+1}A}\approx \exp\Bigl(-\lambda_{\min}(A)\sum_{i=j}^{k-1}\eta_{i+1} \Bigr).
\end{equation}
Therefore, the correlations between the batch-means $\overline{X}_{n_k}$ are close to zero if the batch sizes are large enough, in which case different batch-means can be roughly treated as independent. As a consequence, the sample covariance gathered from the batch-means will serve as a good estimator of the true asymptotic covariance.

The remaining difficulty is how to determine the batch sizes. The approximation of correlation given by \eqref{tmp:decor} provides us a clear clue. If we want the correlation between two neighboring batches to be  on the order of $\exp(-cN)$, where $N$ (with $N \rightarrow \infty$) is a parameter controlling the amount of decorrelation and $c$ is a constant, we need
$
\sum_{i=s_k}^{e_k} \eta_i\asymp N
$
for every batch $k$. When $\eta_i=\eta i^{-\alpha}$, $\sum_{i=s_k}^{e_k} \eta_i \asymp (e_k^{1-\alpha}-e_{k-1}^{1-\alpha})$, which leads to the following batch size setting:
\begin{equation}\label{eq:ending}
e_k=((k+1)N)^{\frac{1}{1-\alpha}}\quad k=0,\ldots, M,
\end{equation}
where $e_k$ is the ending point for the $k$-th batch.
There are two practical scenarios to apply the proposed batch-means estimator,
\begin{itemize}\label{para:batching_strategy}
\item Total number of iterates $n$ is given: Noting that $e_M=n$,  the decorrelation strength  factor $N$ takes the following form,
\begin{equation}\label{eq:N}
N=\frac{n^{1-\alpha}}{M+1},
\end{equation}
where $M$ is the number of batches. Based on the result of Theorem \ref{thm:general} below, it is preferable to take $N=n^{\frac{1-\alpha}{2}}$. 
\item When $n$ is unknown (but sufficiently large): Given a target error bound $\epsilon$, we pick an $N\asymp \epsilon^{-2}$. Then, we receive the online data and batch the SGD iterates according to \eqref{eq:ending}. When the number of batches $M$ is sufficiently large (e.g., the upper bound in \eqref{eq:batch_mean_con} below is smaller than $\epsilon$), we stop our SGD procedure and output the batch-means estimator.

\end{itemize}
\par
Under this setting, the batch-means covariance estimator \eqref{eq:batch_mean} is consistent as shown in the following theorem.
\begin{theorem}[{Error rate of the batch-means estimator}]
\label{thm:general}
Under Assumptions \ref{aspt:convexity} and \ref{aspt:martingale}, when $d$ is fixed {and the step size is chosen to be $\eta_i = \eta i^{-\alpha}$ with $\alpha\in (\frac12,1)$,} the batch-means estimator initialized by any bounded $x_0$ is a consistent estimator. In particular, for sufficiently large $N$ and $M$, we have 
\begin{eqnarray}
 &\mathbb{E}  \opnorm{M^{-1}\sum_{k=1}^Mn_k(\overline{X}_{n_k}-\overline{X}_{M})(\overline{X}_{n_k}-\overline{X}_{M})^\top -A^{-1}SA^{-1}} \cr
 &\lesssim \dcon M^{-\frac{1}{2}}+\dcon N^{-\frac{1}{2}}+\dcon^{\frac32} (MN)^{-\frac{\alpha}{4-4\alpha}}+\dcon^2M^{-1}+\dcon^3M^{-1}N^{\frac{1-2\alpha}{1-\alpha}}.
\label{eq:batch_mean_con}
\end{eqnarray}
\end{theorem}

As $n \rightarrow \infty$, by \eqref{eq:N}, we can choose $M, N \rightarrow \infty$ and thus the right hand side of \eqref{eq:batch_mean_con} will converge to zero for any $\alpha \in (1/2,1)$, which shows the consistency of the proposed covariance estimator. When $d$ is fixed, $\dcon$ is a constant, and it is straightforward  to see that the right hand side of \eqref{eq:batch_mean_con} is dominated by $\dcon (M^{-\frac{1}{2}}+ N^{-\frac{1}{2}})$. Therefore, according to \eqref{eq:N} (i.e., $N(M+1)=n^{1-\alpha}$), we have the following Corollary \ref{cor:batch} that suggests the optimal order of $M$.

\begin{cor}\label{cor:batch}
Under Assumptions \ref{aspt:convexity} and \ref{aspt:martingale}, when $d$ is fixed and $n$ is sufficiently large, by choosing the step size $\eta_i = \eta i^{-\alpha}$ with $\alpha\in (\frac12,1)$, $M\asymp n^{\frac{1-\alpha}{2}}$, and $N\asymp n^{\frac{1-\alpha}{2}}$, we have
\begin{align}
 &\mathbb{E}  \opnorm{M^{-1}\sum_{k=1}^Mn_k(\overline{X}_{n_k}-\overline{X}_{M})(\overline{X}_{n_k}-\overline{X}_{M})^\top -A^{-1}SA^{-1}} \nonumber \\
 & \lesssim \dcon n^{-\frac{1-\alpha}{4}}+\dcon^{\frac32} n^{-\frac{\alpha}{4}}+\dcon^2n^{-\frac{1-\alpha}{2}}+\dcon^3n^{-\alpha}.
 \label{eq:batch_means_bound}
\end{align}
When $C_d$ is a constant, the right hand side of \eqref{eq:batch_means_bound} is dominated by $O(n^{-\frac{1-\alpha}{4}})$.
\end{cor}

As we will show in  simulations in Section \ref{sec:exp}, wide choices between $M=n^{0.2}$ to $M=n^{0.3}$ lead to reasonably good coverage rates when $\alpha$ is close to $1/2$. Moreover, since $\alpha\in (1/2,1)$, the convergence rate $n^{-\frac{1-\alpha}{4}}$ is slower than the rate of the plug-in estimator $n^{-\frac\alpha 2}$. Although batch-means estimator has a slower convergence rate, the next corollary shows that this method still constructs asymptotic exact confidence intervals.

\begin{cor}
\label{cor:general}
Under the assumptions of Theorem \ref{thm:general}, when $d$ is fixed, $n\to \infty$, and {the step size $\eta_i = \eta i^{-\alpha}$ with $\alpha\in (\frac12,1)$,} we have that
\[
\Pr\left(\bar{x}_{n,j}-z_{q/2}\hat{\sigma}^B_{n,j}/\sqrt{n}\leq x^*_j\leq \bar{x}_{n,j}+z_{q/2}\hat{\sigma}^B_{n,j}/\sqrt{n}\right)\to 1-q,
\]
where
$
\hat{\sigma}^B_{n,j}:=\left[M^{-1}\sum\limits_{k=1}^Mn_k(\overline{X}_{n_k}-\overline{X}_{M})(\overline{X}_{n_k}-\overline{X}_{M})^\top \right]^{1/2}_{j,j}.$
\end{cor}
The proof is identical to the one of Corollary \ref{cor:pluginterval} and therefore omitted.

\subsection{Intuition behind the proof}
Now let us provide the main idea behind the proof of Theorem \ref{thm:general}. Recall that the SGD recursion in \eqref{eq:sgd_general_delta} can be approximated by \eqref{eq:Delta_recursion}:
$
\Delta_n \approx (I_d- \eta_n A) \Delta_{n-1}  + \eta_n \xi_n.
$
We replace ``$\approx$'' by the equal sign and define an auxiliary sequence $U_n$:
\begin{equation}
\label{eqn:linear}
U_n=U_{n-1}-\eta_n AU_{n-1}+\eta_n \xi_n,\quad U_0=\Delta_0.
\end{equation}
Note that $\Delta_n=U_n$ in the linear model setting, but our proof applies to \emph{non-linear models (e.g., generalized linear models)}. For a non-linear model, the high-level idea of the proof consists of two steps:
\begin{enumerate}
  \item Establishing the consistency (and the rate of convergence) of the batch-means estimator based on the sequence $U_n$;
  \item Quantifying the difference between $\Delta_n$ and $U_n$, where $\Delta_n$ in \eqref{eq:sgd_general_delta} is the original sequence of interest generated from SGD for general loss functions, and $U_n$ in \eqref{eqn:linear} is its auxiliary linear approximation sequence.
\end{enumerate}
In fact, the sequence $U_n$ is the so-called ``oracle iterate sequence'', which has also been considered in \cite{PJ92}. It can be written in a more explicit form:
\begin{equation}\label{eq:U_n_exp}
U_n=\prod_{k=1}^n\left(I-\eta_k A\right) U_0+\sum_{m=1}^n\prod_{k=m+1}^n\left(I-\eta_k A\right)\eta_m\xi_m.
\end{equation}
Given the sequence $U_n$, we construct the batch-means estimator based on $U_n$ as
$
\frac{1}{M} \sum_{k=1}^M n_k(\overline{U}_{n_k}-\overline{U}_M)(\overline{U}_{n_k}-\overline{U}_M)^\top ,
$
where $\overline{U}_{n_k}$ and $\overline{U}_M$ are defined as in \eqref{eq:batch_mean} with $x_i$ being replaced by $U_i$. The analysis of the batch-means estimator from $U_n$ is simpler than that from SGD iterates $x_n$ since the expression of $U_n$ in \eqref{eq:U_n_exp} only involves the product of matrices and the martingale differences $\xi_m$. In particular, we establish the consistency of the batch-means estimator based on $U_n$ in the following lemma.
\begin{restatable}{lem}{thmlinear}
\label{thm:linear}
Under Assumptions \ref{aspt:convexity} and \ref{aspt:martingale}, when $d$ is fixed and {the step size is chosen to be $\eta_i = \eta i^{-\alpha}$ with $\alpha\in (\frac12,1)$,} the batch-means estimator based on the sequence $U_n$ with any bounded $U_0$ satisfies the following inequality for sufficiently large $N$ and $M,$
\begin{eqnarray*}
&& \mathbb{E}\;   \opnorm{ M^{-1}\sum_{k=1}^M n_k(\overline{U}_{n_k}-\overline{U}_M)(\overline{U}_{n_k}-\overline{U}_M)^\top -A^{-1}SA^{-1}}\cr
&&\lesssim  \dcon M^{-\frac{1}{2}}+\dcon N^{-\frac{1}{2}}+\dcon^{\frac32} (MN)^{-\frac{\alpha}{4-4\alpha}}.
\end{eqnarray*}
\end{restatable}
The proof of Lemma \ref{thm:linear} is provided in Appendix \ref{suppsec:thmlinear}. With Lemma \ref{thm:linear} in place, to obtain the desired  consistency result in  Theorem \ref{thm:general}, we only need to study the difference between $\Delta_n$ and $U_n$. In particular, let $\delta_n:=\Delta_n-U_n$. We have the following recursion:
\begin{align*}
\delta_n&=\Delta_{n-1}-U_{n-1}-\eta_n\nabla F(x_{n-1})-\eta_n AU_{n-1}\\
&=\delta_{n-1}-\eta_n A \delta_{n-1}+\eta_n(A\Delta_{n-1}-\nabla F(x_{n-1})).
\end{align*}
Notably,  by replacing $\xi_n$ in \eqref{eqn:linear} with $A\Delta_{n-1}-\nabla F(\Delta_{n-1})$, $\delta_{n}$ follows a similar recursion relationship to that of the sequence  $U_n$. Based on this observation, we show that $\delta_n$ is a sequence of small numbers, and hence  $\Delta_n$ and $U_n$ are close to each other. Combining this with Lemma \ref{thm:linear}, we will reach the conclusion in Theorem \ref{thm:general} (see Appendix \ref{sec:supp_batch} for the rigorous proof).

\section{High-dimensional Linear Regression}
\label{sec:high_dim}
In Section \ref{sec:plugin} and \ref{sec:batchmean}, we assumed that the dimension $d$ is fixed while $n \to \infty$. However in high-dimensional settings, it is often the case that $d \asymp n$ or $n=o(d)$. In below we consider a high-dimensional linear model
$
b_i= a_i^\top  x^* +\varepsilon_i,
$
where $x^*$ is $s$-sparse (i.e., $\norm{x}_0 \le s$) and let $S=\{j:x^*_j\neq 0\}$ be the support of true regression coefficients. Each covariate $a_i \in \mathbb{R}^d$ is an \emph{i.i.d.} sub-Gaussian random vector from a common population $a$ with the covariance matrix $A$, and $\varepsilon_i \sim N(0, \sigma^2)$.
For simplicity, we assume that $\sigma$ is known.
For high-dimensional linear regression, one of the most popular estimators is the Lasso estimator, denoted by
 $\widehat{x}_{\mathrm{Lasso}}$. That is,
\begin{equation*}
\widehat{x}_{\mathrm{Lasso}}=\frac{1}{2n}\argmin_{x\in\R^d}\|b-Dx\|_2^2+\lambda\|x\|_1,
\end{equation*}
where $D=[a_1, \ldots, a_n]^\top  \in \mathbb{R}^{n \times d}$ is the design matrix, $b=[b_1, \ldots, b_n]^\top  \in \mathbb{R}^{n \times 1}$ is the response vector.

As suggested by earlier work (see, e.g., \cite{meinshausen2009p,Wainwright:09,buhlmann2011statistics,belloni2013least}), 
the Lasso estimator can be used as a screening method to reduce the set of the variables to $\widehat{S}$, a subset which contains the true support $S$ with probability tending to 1. For example, by choosing the regularization parameter $\lambda$ as (2.12) in \cite{belloni2013least}  and under certain assumptions, \cite{belloni2013least} proved  that $S\subseteq \widehat{S}$ and $|\widehat{S}\backslash S|\lesssim s$ with high probability (see Theorems 2 and 3 therein). When $s$ is treated as a constant, the selected model will be of fixed dimension.  Based on the selected model, we are able to directly apply our plug-in or batch-means estimator in Section \ref{sec:cov} on $\widehat{S}$ to conduct inference for $x_j^*$ for $j \in \widehat{S}$.

However, this approach has several limitations. First, the screening approach requires a strong ``beta-min'' assumption. In particular,  this assumption requires that $\min_{j\in S }|x^*_j|>\max_{j\in S }|\widehat{x}_{\mathrm{Lasso},j}-x^*_j|$, or $\min_{j\in S }|x^*_j|\gtrsim\sqrt{s(\log d)/n}$, e.g., \cite{buhlmann2011statistics,buhlmann2014high,belloni2013least}. Other screening methods (e.g., ``Sure Independence Screening'' (SIS) method \cite{fan2008sure}) also require a similar beta-min condition. However, since we are interested in inference of the model parameters instead of the model selection, the ``beta-min'' condition should be avoidable. Second, the sparsity level $s$ has to be treated as a constant to apply our theoretical results of either the plug-in or batch-means estimator. Furthermore, when using Lasso as a screening approach, it inevitably requires more than one pass of the data which does not fit our online setting.

\subsection{Debiasing approach}

To relax the strong conditions  when using the Lasso as a screening approach, we propose a new approach for conducting inference for high-dimensional linear regression that only uses one pass of the data. Our approach is based on the following
debiased Lasso estimator \cite{vandegeer14asy, javanmard2014confidence, Zhang14CI},
\begin{eqnarray*}
 \widehat{x}^d_{\mathrm{Lasso}}= \widehat{x}_{\mathrm{Lasso}} + \frac{1}{n} \widehat{\Omega} D^\top (b-D\widehat{x}_{\mathrm{Lasso}}),
\end{eqnarray*}
where $\widehat{\Omega}$ is an estimator of the inverse covariance matrix of the design $\Omega=A^{-1}$.
To construct $\widehat{\Omega}$,  \cite{vandegeer14asy} adopts the node-wise Lasso approach (see also \cite{meinshausen06}), i.e.,
\begin{eqnarray}\label{eq:L1_nodewise_det}
\widehat{\gamma}^j=\argmin_{\gamma^j \in \mathbb{R}^{d-1}} \frac{1}{2n}\|D_{\cdot,j} -D_{\cdot,-j} \gamma^j\|_2^2+ \lambda_j \|\gamma^j\|_1,
\end{eqnarray}
where $D_{\cdot, j}$ is the $j$-th column of the design matrix $D$ and $D_{\cdot, -j}$ is the design submatrix without the $j$-th column.
Further, one can estimate $\Omega_{j,j}$ by
\[
  \widehat{\tau}_j=\frac{1}{n}(D_{\cdot, j} - D_{\cdot, -j}\widehat{\gamma}^j)^\top  D_{\cdot, j},
\]
 Given $\widehat{\gamma}^j$ and $ \widehat{\tau}_j$, the matrix ${\Omega}$ is estimated by,
\begin{eqnarray}\label{eq:M}
\widehat{\Omega} = \widehat{T} \widehat{C},
\end{eqnarray}
where $\widehat{T} := \mathrm{diag}(1/\widehat{\tau}_1, \ldots,  1/\widehat{\tau}_d)$ and
$
  \widehat{C}:=\begin{pmatrix}
    1 & -\widehat{\gamma}^1_2 & \ldots & -\widehat{\gamma}^1_d \\
    -\widehat{\gamma}^2_1 & 1  & \ldots & -\widehat{\gamma}^2_d\\
    \vdots        &  \vdots &  \ddots  & \vdots \\
    -\widehat{\gamma}^d_1 & -\widehat{\gamma}^d_2 & \ldots & 1
  \end{pmatrix}.
$

Note that in the existing literature, $\widehat{x}_{\mathrm{Lasso}}$ and $\widehat{\Omega}$ are obtained via deterministic convex optimization. Therefore, debiased Lasso approaches  cannot be directly applied to the stochastic setting in this work. To address this issue, we propose to compute the estimators for both $x^*$ and $\Omega$ using the Regularization Annealed epoch Dual AveRaging (RADAR) algorithm \cite{agarwal2012stochastic}, which is a variant of SGD. Similar to SGD, RADAR computes the stochastic gradient on one data point at each iteration.
\label{para:RADAR} Please refer to \cite{agarwal2012stochastic} for more details of the RADAR algorithm.
The reason why we use RADAR instead of the vanilla SGD is because RADAR provides the optimal convergence rate in terms of the $\ell_1$-norm. In particular, we apply RADAR to the following $\ell_1$-regularized problem,
\begin{eqnarray}\label{eq:L1_reg_obj}
 \min_{x \in \mathbb{R}^{d}}\E(b-a^\top  x)^2 + \lambda \|x\|_1
\end{eqnarray}
and let $\widehat{x}_n$ be the solution output from RADAR with $n$ iterations. Similarly, we again use stochastic optimization instead of deterministic optimization in the node-wise Lasso in \eqref{eq:L1_nodewise_det}, that is, applying RADAR to the following optimization problem for each dimension $1 \leq j \leq d$,
\begin{eqnarray}\label{eq:L1_nodewise}\label{eq:debiasedLasso}
\widehat{\gamma}^j=\argmin_{\gamma^j \in \mathbb{R}^{d-1}} \E\|a_j - a_{-j} \gamma^j\|_2^2+ \lambda_j \|\gamma^j\|_1,
\end{eqnarray}
where $a_j$ is the $j$-th coordinate of the population design vector $a$ and $a_{-j}$ is the subvector of $a$ without the $j$-th coordinate. Given $\widehat{\gamma}^j$ from solving \eqref{eq:L1_nodewise} via the iterative stochastic algorithm,  the inverse covariance estimator $\widehat{\Omega}$ is constructed according to \eqref{eq:M}.

It is noteworthy that  although the proposed  $\widehat{\Omega}$ is of the same form as the estimator for $A^{-1}$ in \cite{vandegeer14asy}, our $\widehat{\gamma}^j$ is different from the one in \cite{vandegeer14asy}. More precisely, our $\widehat{\gamma}^j$  is the output of a  stochastic gradient-based algorithm, while \cite{vandegeer14asy} obtained $\widehat{\gamma}^j$ from deterministic optimization in \eqref{eq:L1_nodewise_det}.  With all these ingredients in place, we present the stochastic gradient based construction of the confidence interval for $x^*_j$ for $j\in \{1, \ldots, d\}$ in Algorithm \ref{algo:high_CI}.  The hypothesis test can also be performed once the estimator of the asymptotic variance of $\widehat{x}_j^d$ is available (see Theorem \ref{thm:high_inf}). We note that the proposed method is computationally more efficient than the methods based on deterministic optimization. It only requires one pass of the data with the total per-iteration complexity $O(d^2)$ (note that the node-wise Lasso needs to solve $d$ optimization problems) and is applicable to online data (in contrast to multiple passes of data with deterministic optimization used in existing methods). The details of the algorithm are provided in Algorithm \ref{algo:high_CI}.

\begin{algorithm}[!t]
\caption{Stochastic Optimization Based Confidence Interval Construction for High-dimensional Sparse Linear Regression}
\label{algo:high_CI}
\begin{algorithmic}
\Inputs{Regularization parameter $\lambda \asymp \sqrt{\log d/n}$, and $\lambda_j \asymp \sqrt{\log d/n}$ for each               dimension $j$, the noise level $\sigma$, confidence level $1-\alpha$.}
\For{$t= 1$ to $n$}
\State
Randomly sample the data $(a_t, b_t)$ and update the design $D \leftarrow [D^\top , a_t]^\top $ and response $b \leftarrow [b^\top , b_t]^\top $.
\State
Update $x_t$ by running one iteration of RADAR on the optimization problem  \eqref{eq:L1_reg_obj} using the stochastic gradient $(a_t^\top  x_{t-1} -b_t) a_t$,

\For{$j=1$ to $d$ }
\State
Update $\gamma^j_t$ by running one iteration of RADAR on the  optimization problem \eqref{eq:L1_nodewise} using the stochastic gradient $(a_{t,-j}^\top  \gamma^j_{t-1} -a_{t,j}) a_{t,-j}$.
\EndFor
\EndFor
\State
Let $\widehat{x}_n=x_n$ and $\widehat{\gamma}^j=\gamma^j_n$ for $j \in \{1,\ldots, d\}$ be the final outputs.
\State
Construct the debiased estimator $\widehat{x}^d$ with $\widehat{\Omega}$ defined in \eqref{eq:M}.
\begin{equation}\label{eq:xd}
\widehat{x}^d= \widehat{x}_{n} + \frac{1}{n} \widehat{\Omega} D^\top (b-D\widehat{x}_{n}).
\end{equation}
\Outputs{The $(1-\alpha)$ confidence interval for each $x_j^*$ : $\widehat{x}^d_j \pm z_{\alpha/2}\sigma \sqrt{(\widehat{\Omega} \widehat{A} \widehat{\Omega})_{jj}/n}$, where $\widehat{A}=\frac{1}{n} D^\top  D $.}
\end{algorithmic}
\end{algorithm}

To provide the theoretical justification for Algorithm \ref{algo:high_CI} in terms of constructing valid confidence intervals, we make the following assumptions (which are similar to the assumptions made in \cite{vandegeer14asy}).

\begin{aspt}
The covariate $a$ is a sub-Gaussian random vector with variance proxy $K^2$. The population covariance $A$ has bounded eigenvalues,
\begin{align*}
0< \mu < \lambda_{\min} (A) <  \lambda_{\max} (A) < L_F.
\end{align*}
    Denote the set of parameters by $\mathcal{B}(s)=\{x\in\R^d;\;\norm{x}_0\leq s$ and $\|x\|_1$ is bounded by a constant\,$\}$. The true regression parameter $x^*\in\mathcal{B}(s)$ where $s = o( \sqrt{n}/ \log d)$. Moreover, the inverse covariance $\Omega$ has sparse rows. In particular, define
\[
s_j = \left|\{1 \leq k\leq d: k \neq j , \Omega_{j,k} \neq 0 \}\right|.
\]
We assume that $\max _j s_j \le C s$ for some constant $C$.
\label{aspt:cov}
\end{aspt}
Under Assumption \ref{aspt:cov}, we first present an $\ell_1$-bound result as a corollary of Proposition 1 in \cite{agarwal2012stochastic},
\begin{proposition}\label{prop:L1_conv}
 Under Assumption \ref{aspt:cov} and using the same algorithm parameters as Proposition 1 in \cite{agarwal2012stochastic}, there exists a constant $c_0$, such that $\widehat{x}_n$ in Algorithm \ref{algo:high_CI} satisfies
$
\|\widehat{x}_n-x^*\|_1 \leq c_0 s  \sqrt{\frac{\log d}{n}}
$
uniformly in $x^*\in\mathcal{B}(s)$ with high probability. Further, for each $j= \{ 1,\ldots, d\}$, we have
  \begin{eqnarray*}
    \|\widehat{\gamma}^j + \Omega_{j,j}^{-1} (\Omega_{j,-j})^\T \|_1 \leq c_0} s_j \sqrt{\frac{\log d}{n}
  \end{eqnarray*}
 holds with high probability.
\end{proposition}

The proof of Proposition \ref{prop:L1_conv} is provided in Appendix \ref{subsec:L1_conv}. {Let $\mathbb{P}_{x^*}$ be the distribution under the high-dimensional linear model $b_i=a_i^\top x^*+\epsilon_i$. }Given Proposition \ref{prop:L1_conv}, we state the inference result in the next theorem. We note that although the statement of the following theorem is similar to Theorem 2.2 and Corollary 2.1 in \cite{vandegeer14asy}, the proof is more technically involved. The main challenge is that the existing analysis in \cite{vandegeer14asy} starts from the  KKT condition of the deterministic optimization for estimating $\Omega$. However, we estimate $\Omega$ using the stochastic optimization and thus the corresponding KKT condition no longer holds. Please refer to the proof in Appendix \ref{supp:high_inf} for more detail.

\begin{theorem}\label{thm:high_inf}
Under Assumption \ref{aspt:cov}, for suitable choices of $\lambda \asymp \sqrt{\log d/n}$ and $\lambda_j \asymp \sqrt{\log d/n}$, we have for all $j \in \{1,\ldots, d\}$ and all $z \in \mathbb{R}$,
\label{cor:uniform}
\[
\sup\limits_{x^*\in\mathcal{B}(s)}\left|\mathbb{P}_{x^*}\Bigg(\frac{\sqrt{n}(\widehat{x}^d_j- x_j^*)}{\sigma\sqrt{(\widehat{\Omega} \widehat{A} \widehat{\Omega}^\T)_{jj}}}\leq z\Bigg)-\Phi(z)\right|=o_{p}(1),
\]
where $\widehat{x}^d$ is the debiased estimator defined in \eqref{eq:debiasedLasso}, $\widehat{\Omega}$ is defined in \eqref{eq:M} and the sample covariance matrix $\widehat{A}= \frac{1}{n} D^\top  D$.
\end{theorem}

Theorem \ref{thm:high_inf} shows that $\frac{1}{\sigma\sqrt{(\widehat{\Omega} \widehat{A} \widehat{\Omega}^\top )_{jj}}}\sqrt{n} (\widehat{x}^d_j- x_j^*) $ converges in distribution to $N(0,1)$ uniformly for any $x^*\in\mathcal{B}(s)$ and $j\in\{1,2,\dots,d\}$, which verifies the correctness of the asymptotic pointwise confidence interval in the output of Algorithm \ref{algo:high_CI} for $x_j^*$.

Given the uniform convergence result in Theorem \ref{thm:high_inf}, we can construct $p$-values for each single component, and further conduct multiple testing based on component-wise $p$-values. We also note that a similar uniform convergence result has been established in \cite{ning2017general} for a score test approach (see Remark 4.6). It is also interesting to investigate the stochastic optimization based score test as a future work.

\section{Numerical Simulations}
\label{sec:exp}
In this section, we investigate the empirical performance of the plug-in estimator and batch-means estimator of the asymptotic covariance matrix. We consider both linear and logistic regression models, where $\{a_i,b_i\}$ are \emph{i.i.d.} samples with $a_i \sim \cN(0, \Sigma)$ and $x^*$ is the true parameter vector of the model. For both models, we consider three different structures of the $d \times d$ covariance matric $\Sigma$:
\begin{itemize}
\setlength\itemsep{2pt}
\item Identity:\qquad $\Sigma=I_d$;
\item Toeplitz:\qquad $\Sigma_{i,j}=r^{|i-j|}$;
\item Equi Corr: \quad $\Sigma_{i,j}=r$ for all $i\neq j$,\quad $\Sigma_{i,i}=1$ for all $i$.
\end{itemize}
We report $r=0.5$ for Toeplitz and $r=0.2$ for equicorrelation (Equi Corr) covariance matrices in the main paper. The experimental results on other settings of $r$ are relegated to Appendix \ref{sec:sim_supp} due to space limitations.  The noise $\varepsilon_n$ in linear regression is set to \emph{i.i.d.} $N(0,\sigma^2)$ with $\sigma=1$. The parameter $\alpha$ in the step size is chosen to be $0.501$ (slightly larger than 0.5). All the reported results are obtained by taking the average of $500$ independent runs.
We consider the finite sample behavior of the plug-in estimator and the batch-means estimator for the inference of each individual regression coefficient $x_j$, $j\in\{1,2,\dots,d\}$.

\subsection{Low-dimensional cases}
In each case, we consider the sample size $n=10^5$ and the dimension $d=5$, $20$, $100$, $200$. For each model, the corresponding parameter $x^*$ is a $d$-dimensional vector linearly spaced between $0$ and $1$. The thresholding scheme is not used for the plug-in estimator. In fact, we observe that the obtained $A_n$ is always invertible and the results are stable without the thresholding. For the batch-means estimator (BM in short), we consider  three different choices of the number of batches: $M=n^{0.2}$, $M=n^{0.25}$, and $M=n^{0.3}$. Note that $\alpha=0.501$. As we suggested in Corollary \ref{cor:batch}, to achieve a better convergence rate, the number of batches $M$ is chosen around the optimal value $n^{\frac{1-\alpha}{2}} \approx n^{0.25}$.

We set the nominal coverage probability $1-q$ to 95\%. The performance of an estimator is measured by the average coverage rate (Cov Rate) of the confidence intervals and the average length  (Avg Len) of the intervals for each individual coefficient.

For each setting, we also report the oracle length of the confidence interval with respect to the true covariance matrix $A^{-1} S A^{-1}$ and the corresponding coverage rate when using the same center as the BM.

For linear regression, the asymptotic covariance is $A^{-1} S A^{-1} = \sigma^2 \Sigma^{-1}= \Sigma^{-1}$ and the oracle interval length for each coordinate $j$ will be $\frac{2 z_{q/2} (\Sigma^{-1/2})_{jj}}{\sqrt{n}}$.
Table \ref{table:linear_CI} shows the empirical performance of the plug-in and BM under linear models with three different design covariance matrices.

\begin{table}[!t]
\caption{Linear Regression: The average coverage rate and length of confidence intervals, for the nominal coverage
probability  $95\%$.  The columns (BM: $n^c$ for $c$=0.2, 0.25, and 0.3) correspond to the batch-means estimator with $M=n^c$ number of batches. Cov Rate under ``Oracle'' refers to coverage rates when using the same center as BM but with oracle interval lengths. Standard errors are reported in the brackets.}
\begin{tabular}{llrrrrr}
\hline
 & $d$ & Plug-in & \multicolumn{3}{c}{BM}& Oracle \\
&&&$M=n^{0.2}$ &  $M=n^{0.25}$ & $M=n^{0.3}$ &\\
\hline
Identity $\Sigma$\\
Cov Rate (\%)&5&95.68(0.87)&90.28(0.46)&93.68(0.79)&91.64(0.79)&87.44\\
Avg Len ($\times 10^{-2}$)&&1.49(0.01)&1.39(0.01)&1.47(0.01)&1.43(0.01)&1.24 \\
Cov Rate (\%)&20&94.99(0.94)&91.30(1.08)&93.92(1.25)&92.95(1.19)&88.24\\
Avg Len ($\times 10^{-2}$)&&1.44(0.01)&1.35(0.01)&1.41(0.01)&1.38(0.01)&1.24 \\
Cov Rate (\%)&100&95.04(1.01)&90.75(1.36)&93.15(1.12)&92.37(1.10)&87.89\\
Avg Len ($\times 10^{-2}$)&&1.41(0.01)&1.32(0.01)&1.35(0.01)&1.35(0.01)&1.24 \\
Cov Rate (\%)&200&94.75(1.13)&90.49(1.21)&92.97(1.17)&91.97(1.18)&88.12\\
Avg Len ($\times 10^{-2}$)&&1.39(0.01)&1.30(0.01)&1.31(0.01)&1.32(0.01)&1.24 \\
\hline
Toeplitz $\Sigma$\\
Cov Rate (\%)&5&95.24(0.92)&91.16(0.50)&94.28(0.86)&93.04(0.90)&88.31\\
Avg Len ($\times 10^{-2}$)&&1.83(0.10)&1.74(0.10)&1.82(0.11)&1.78(0.12)&1.53 \\
Cov Rate (\%)&20&94.84(0.97)&90.97(1.08)&93.75(0.93)&92.77(0.81)&87.26\\
Avg Len ($\times 10^{-2}$)&&1.81(0.05)&1.71(0.06)&1.78(0.06)&1.76(0.06)&1.58 \\
Cov Rate (\%)&100&95.01(1.12)&90.36(1.33)&91.83(1.09)&91.52(1.17)&89.11\\
Avg Len ($\times 10^{-2}$)&&1.77(0.02)&1.67(0.03)&1.67(0.03)&1.69(0.02)&1.60 \\
Cov Rate (\%)&200&94.69(1.33)&90.01(1.41)&91.65(1.36)&91.24(1.41)&89.43\\
Avg Len ($\times 10^{-2}$)&&1.74(0.02)&1.62(0.02)&1.62(0.02)&1.62(0.02)&1.60 \\
\hline
Equi Corr $\Sigma$\\
Cov Rate (\%)&5&94.80(0.88)&90.92(1.09)&93.60(0.92)&92.32(0.68)&86.79\\
Avg Len ($\times 10^{-2}$)&&1.60(0.01)&1.46(0.01)&1.55(0.01)&1.52(0.01)&1.31 \\
Cov Rate (\%)&20&95.10(0.99)&91.15(1.14)&93.66(0.99)&92.78(0.92)&88.04\\
Avg Len ($\times 10^{-2}$)&&1.59(0.01)&1.47(0.01)&1.54(0.01)&1.51(0.01)&1.36 \\
Cov Rate (\%)&100&94.93(1.06)&90.86(1.26)&93.19(1.15)&92.29(1.10)&87.15\\
Avg Len ($\times 10^{-2}$)&&1.56(0.01)&1.47(0.01)&1.52(0.01)&1.50(0.01)&1.38 \\
Cov Rate (\%)&200&94.49(1.09)&90.57(1.45)&92.45(1.27)&91.91(1.13)&87.22\\
Avg Len ($\times 10^{-2}$)&&1.51(0.01)&1.45(0.01)&1.49(0.01)&1.49(0.01)&1.38 \\
\hline
\end{tabular}
\label{table:linear_CI}
\end{table}

From Table \ref{table:linear_CI}, both the plug-in and BM achieve good performance. The plug-in gives better average coverage rate than BM: the average coverage rates in all different settings are nearly 95\%. However, the average length of  plug-in is usually larger than that of BM and the corresponding oracle interval length. On the other hand, BM achieves about 92\% coverage rate when $M=n^{0.25}$ or $M=n^{0.3}$.
We further consider the logistic regression. To provide an oracle interval length based on the true  asymptotic covariance $A^{-1}SA^{-1}=A^{-1}$, we estimate $A$ in \eqref{eq:logA} by its empirical version $\widehat{A}$ using one million fresh samples and the oracle interval length of each coordinate $j$ is computed as $\frac{2z_{q/2}(\widehat{A}^{-1/2})_{jj}}{\sqrt{n}}$. We provide the result in Table \ref{table:log_CI} for different design covariance matrices. From Table \ref{table:log_CI}, the plug-in still achieves nearly 95\% average coverage rate. The BM achieves about 90\% coverage rate and the average length is usually smaller than the oracle length. Moreover, as $d$ becomes larger, the interval lengths for both estimators increase. Finally, the performance of BM is insensitive to the choice of the number of batches $M$: different $M$'s lead to comparable coverage rates. There are two reasons for the undercoverage of BM. First, the obtained center could deviate from $x^*$ that introduces the bias. Second, the BM has a slower convergence rate as compared to the plug-in (especially for the case of logistic regression). However, since the BM only uses the iterates from SGD,  it is computationally more efficient than the plug-in estimator which requires the computation of the Hessian matrix $\tilde{A}_n$ and its inverse.

\begin{table}[!t]
\centering
\caption{{Logistic Regression: The average coverage rate and length of confidence intervals, for the nominal coverage probability  $95\%$.  The columns (BM: $n^c$ for $c$=0.2, 0.25, and 0.3) correspond to the batch-means estimator with $M=n^c$ number of batches. Cov Rate under ``Oracle'' refers to coverage rates when using the same center as BM but with oracle interval lengths. Standard errors are reported in the brackets.}}
\begin{tabular}{llrrrrr}
\hline
& $d$ & Plug-in & \multicolumn{3}{c}{BM}& Oracle \\
&&&$M=n^{0.2}$ &  $M=n^{0.25}$ & $M=n^{0.3}$ &\\
\hline
 Identity  $\Sigma$\\
Cov Rate (\%) & 5 & 95.04(1.13)& 89.24(1.55) &90.12(1.70)&89.36(1.97)&91.45\\
Avg Len ($\times 10^{-2}$) &  & 3.24(0.41) & 3.01(0.26)&2.94(0.25)& 2.87(0.23)&3.09\\
Cov Rate (\%) & 20 &95.00(1.34) & 89.35(2.00)& 90.22(1.67) &89.74(2.11)&90.37\\
Avg Len ($\times 10^{-2}$) & &3.79(0.27) &3.53(0.25)&3.46(0.23)&3.42(0.22)&3.68   \\
Cov Rate (\%) & 100 &94.69(1.06) & 89.42(1.66) &90.84(1.68) & 90.41(2.01)&91.24\\
Avg Len ($\times 10^{-2}$) &   &5.21(0.26)&4.97(0.24)&4.87(0.23)&4.80(0.24)&5.06\\
Cov Rate (\%) & 200 &94.47(0.91) & 89.01(1.41) &90.47(1.49) & 90.36(1.74)&92.08\\
Avg Len ($\times 10^{-2}$) &   &6.05(0.29)&5.94(0.26)&5.82(0.27)&5.71(0.25)&5.97\\
\hline
Toeplitz $\Sigma$  \\
Cov Rate (\%) & 5 & 94.96(1.58) & 88.96(2.32) & 90.56(2.06) & 90.12(2.04) &92.41\\
Avg Len ($\times 10^{-2}$) &  &4.06(0.34) & 3.75(0.28) & 3.73(0.27) & 3.61(0.25) &4.04\\
Cov Rate (\%) & 20 & 95.17(1.23) & 89.01(1.93) & 90.39(1.88) & 89.79(1.81)&91.07\\
Avg Len ($\times 10^{-2}$) &  & 5.74(0.29) & 5.57(0.25)  &5.22(0.23)  & 4.95(0.22)&5.59\\
Cov Rate (\%) & 100 &94.91(0.89) & 89.91(1.74)& 90.83(1.81) & 90.54(1.97)&91.47\\
Avg Len ($\times 10^{-2}$) &  &8.47(0.37) & 8.01(0.28)  & 7.71(0.26)  & 7.37(0.25)&8.28\\
Cov Rate (\%) & 200 &94.59(1.04) & 89.72(1.81)& 90.74(1.93) & 90.32(2.02)&92.29\\
Avg Len ($\times 10^{-2}$) &  &9.81(0.41) & 9.24(0.34) & 8.95(0.31)  & 8.78(0.29)&9.84  \\
\hline
 Equi Corr $\Sigma$ \\
Cov Rate (\%) & 5 &94.80(1.66) & 88.08(1.46) & 88.64(1.73) & 89.48(1.51)&93.79 \\
Avg Len ($\times 10^{-2}$) &  & 3.43(0.35)&3.28(0.28) & 3.24(0.25)& 3.20(0.24) &3.38\\  
Cov Rate (\%) & 20 & 94.54(1.73) & 89.27(1.33) & 90.64(1.60) & 90.31(2.10)&92.50\\
Avg Len ($\times 10^{-2}$) &  &5.37(0.31) &  4.84(0.26)  & 4.77(0.24) & 4.51(0.21)&5.19 \\
Cov Rate (\%) & 100 & 94.79(1.08) & 89.01(1.70) & 90.27(1.76) & 89.42(2.01)&94.92\\
Avg Len ($\times 10^{-2}$) &  &10.24(0.51) & 10.17(0.47)  & 9.75(0.42)  & 9.24(0.40)&10.89 \\
Cov Rate (\%) & 200 & 94.24(1.09) & 89.13(1.44) & 90.01(1.92) & 89.23(1.79)&92.40\\
Avg Len ($\times 10^{-2}$) &  &15.70(0.62) & 14.82(0.57)  & 14.01(0.55)  &13.88(0.52) &15.31 \\
\hline
\end{tabular}
\label{table:log_CI}
\end{table}

\subsection{High-dimensional cases}

In a high-dimensional setting, we consider the sample size $n=100$, and the dimension $d=500$. The active set $S_0=\{1,2,\dots,s_0\}$, where the cardinality $s_0=|S_0|=3$ or $15$. The non-zero regression coefficients $\{x_j^*\}_{j \in S_0}$ are from a fixed realization of $s_0$ \emph{i.i.d.} uniform distribution $U[0,c]$ with $c=2$. 

First, we consider the average coverage rate and the average length of the intervals for individual coefficients corresponding to the variables in either $S_0$ or $S_0^c$ where $S_0^c=\{1,\ldots,d\}\backslash S_0$.  Again, we set the nominal coverage probability $1-q$ to 95\%.

Our experimental setup follows directly from \cite{vandegeer14asy}, and we provide the oracle length of the confidence intervals for comparison. For linear regression, the asymptotic covariance is $A^{-1} S A^{-1} = \sigma^2 \Sigma^{-1}= \Sigma^{-1}$ and the oracle interval length for each coordinate $j$ will be $\frac{2 z_{q/2} (\Sigma^{-1/2})_{jj}}{\sqrt{n}}$. We provide the result in Table \ref{table:high_dim} for different design covariance matrices.  

\begin{table}[t!]
\centering
\caption{{High-dimensional linear regression, the average coverage rate and length of confidence intervals, for  the nominal coverage probability  $95\%$. Standard errors are reported in the brackets. 
}}
\begin{tabular}{llll}
\hline
Measure & Identity $\Sigma$ & Toeplitz  $\Sigma$ &Equi Corr $\Sigma$\\
\hline
$S_0=\{1,2,3\}$\\
Cov Rate  $S_0 (\%)$&91.93 (3.13)&91.40 (2.39)&90.20 (1.38)\\
Avg Len $S_0$&0.387 (0.002)&0.401 (0.019)&0.360 (0.014)\\
Cov Rate   $S_0^c (\%) $&90.80 (1.79)&90.21 (1.98)&89.73 (1.96)\\
Avg Len $S_0^c$&0.386 (0.002)&0.417 (0.022)&0.384 (0.023)\\
Oracle  Len &0.392&0.506&0.438\\
\hline
$S_0=\{1,2,\dots, 15\}$\\
Cov Rate  $S_0$ (\%)&90.48 (1.73)&89.84 (2.61)&89.45 (0.87)\\
Avg Len $S_0$&0.379 (0.002)&0.430 (0.024)&0.384 (0.020)\\
Cov Rate   $S_0^c$ (\%)&88.43 (2.30)&86.79 (2.10)&87.12 (1.36)\\
Avg Len $S_0^c$&0.360 (0.003)&0.425 (0.024)&0.383 (0.022)\\
Oracle Len &0.392&0.506&0.438\\
\hline
\end{tabular}
\label{table:high_dim}
\end{table}

\begin{table}[t!]
\centering
\caption{{High-dimensional linear regression using node-wise lasso instead of RADAR for inference, the average coverage rate and length of confidence intervals, for the nominal coverage probability  $95\%$. Standard errors are reported in the brackets.}}
\begin{tabular}{llll}
\hline
Measure & Identity $\Sigma$ & Toeplitz  $\Sigma$ &Equi Corr $\Sigma$\\
\hline
$S_0=\{1,2,3\}$\\
Cov Rate  $S_0 (\%)$&94.73 (1.42) &92.60 (1.95) &91.33 (1.99)\\
Avg Len $S_0$&0.393 (0.001) &0.472 (0.011) &0.431 (0.007)\\
Cov Rate   $S_0^c (\%) $&96.13 (1.44)&95.17 (1.71) &95.92 (2.04)\\
Avg Len $S_0^c$&0.386 (0.001)&0.481 (0.014) &0.429 (0.011)\\
Oracle  Len &0.392&0.506&0.438\\
\hline
$S_0=\{1,2,\dots, 15\}$\\
Cov Rate  $S_0$ (\%)&91.80 (1.04)&91.07 (2.01)&90.33 (1.92)\\
Avg Len $S_0$&0.399 (0.001)&0.512 (0.015)&0.424 (0.008)\\
Cov Rate   $S_0^c$ (\%)&95.75 (1.80)&93.17 (1.90) &94.11 (1.73)\\
Avg Len $S_0^c$&0.379 (0.001) &0.505 (0.016)&0.453 (0.008)\\
Oracle Len &0.392&0.506&0.438\\
\hline
\end{tabular}
\label{table:high_dim_compare}
\end{table}

For high-dimensional linear regression, Algorithm \ref{algo:high_CI} achieves good performance, especially in sparse settings ($s_0=3$).  From Table \ref{table:high_dim}, the average coverage rate is about 90\%. For less sparse problems ($s_0=15$), our method still achieves about 88\% average coverage rate for different design covariance matrices. The coverage rates of the obtained confidence intervals on active sets $S_0$ are slightly better than those on $S_0^c$. The average lengths on both sets are slightly smaller than the oracle lengths. The performance of the cases with identity design matrices are better than those with Toeplitz and equicorrelation design matrices (e.g., having smaller standard deviations). It is reasonable since it is easier to estimate the inverse covariance matrix $\Omega$ when it is an identity matrix.

In Table \ref{table:high_dim_compare}, we also provide the results using the deterministic optimization (instead of the stochastic RADAR) for constructing $\widehat{\Omega}$ \cite{vandegeer14asy}. Both methods achieve comparably reliable coverage rates. From Table \ref{table:high_dim_compare}, the average coverage rates are closer to the nominal levels, better than those in Table \ref{table:high_dim}. The undercovering in Table \ref{table:high_dim} is due to the estimation error of the diagonals of $\hat{\Omega}\hat{A}\hat{\Omega}$ using stochastic optimization method. Based on the computational and storage requirements, a practitioner may decide to use a one-pass algorithm or a more accurate estimator under deterministic optimization.

\section{Conclusions and Future Works}
\label{sec:discuss}
This paper presents two consistent estimators of the asymptotic variance of the average iterate from SGD, especially a computationally more efficient batch-means estimator that only uses iterates from SGD. With the proposed estimators, we are able to construct asymptotically exact confidence intervals and hypothesis tests.

We further discuss statistical inference based on SGD for high-dimensional linear regression. An extension to generalized linear models is an interesting problem for future work.

The seminal work by \cite{toulis2014implicit} develops the averaged implicit SGD procedure and provides the characterization of the limiting distribution. It would be interesting to establish the consistency of  the batch-means estimator based on iterates from  implicit SGD. It is also interesting to relax the current assumptions and consider SGD for more challenging optimization problems (e.g., non-convex problems).

\section*{Acknowledgements}
The authors would like to thank John Duchi, Jessica Hwang, Lester Mackey, Yuekai Sun, and Jonathan Taylor for early discussions on Markov process and the relationship to stochastic gradient. The authors would also like to thank Wei Biao Wu and Wanrong Zhu for the discussion on the martingale difference sequence, and Selina Carter for correcting a typo regarding the construction of the confidence intervals. 

\renewcommand\thesection{\Alph{section}}

\setcounter{equation}{50}
\setcounter{section}{0}
\section{Verifying assumptions for two examples}
\label{sec:verify}
In this section, we verify Assumptions \ref{aspt:convexity}, \ref{aspt:martingale} and \ref{aspt:third} on Examples \ref{exp:linear} and \ref{exp:logistic}.

\begin{lem}
  \label{lem:example1}
  In Example \ref{exp:linear}, suppose $A=\E a_na_n^\top $ is a positive definite matrix. Assumptions \ref{aspt:convexity} and \ref{aspt:martingale} hold for  $0<\mu = \lambda_{\min}(A)\leq \lambda_{\max}(A) = L_F$. To track the dependence on dimension $d$, we consider the case where $a_n\sim \mathcal{N}(0, I_d), \varepsilon\sim \mathcal{N}(0,1)$, then Assumptions \ref{aspt:convexity}, \ref{aspt:martingale} and  \ref{aspt:third} hold with
  \begin{align*}
  L_F=1,\quad \mathrm{tr}(S)=O(d),\quad \Sigma_1=0, \quad \Sigma_2=O(d^2),\\
  \Sigma_3=O(d^2),\quad \Sigma_4=O(d^4),\quad L_2=0,\quad L_4=O(d^2).
  \end{align*}
  Therefore, the dimension constant $\dcon=O(d)$.
\end{lem}
\begin{proof}
  It is easy to see that $\nabla F(x)=A(x-x^*)$ and $\nabla^2 F(x)=A$ for all $x$.
  $
  S:=\mathbb{E}\left([\nabla f(x^*,\zeta)][\nabla f(x^*,\zeta)]^\top  \right)=\mathbb{E}\varepsilon^2_n a_na_n^\top .
  $
  Therefore, Assumption \ref{aspt:convexity} holds for  $0<\mu = \lambda_{\min}(A)\leq \lambda_{\max}(A) = L_F$.
  For Assumption \ref{aspt:martingale}, the sequence $\xi_n=(A-a_na_n^\top )\Delta_{n-1}+a_n\varepsilon_n$. First, we notice it is indeed a martingale sequence, since $\mathbb{E}_{n-1}(A-a_n a_n^\top )=0$ and $\mathbb{E}_{n-1}\varepsilon_n=0$. Second, by $S=\mathbb{E}\varepsilon^2_n a_na_n^\top $,
  \[
  \Sigma(\Delta_{n-1})=\mathbb{E}_{n-1}\xi_n\xi_n^\top -S=\mathbb{E}_{n-1}[(a_n a_n^\top -A) \Delta_{n-1}\Delta_{n-1}^\top  (a_n a_n^\top -A)].
  \]
  Assumption \ref{aspt:martingale} holds because
  \[
  \opnorm{\Sigma(\Delta)}\leq|\mathrm{tr}(\Sigma(\Delta))|\leq \Ltwo{\Delta}^2\mathbb{E}_{n-1}\opnorm{(a_n a_n^\top -A)}^2.
  \]
  and by  H\"{o}lder's inequality:
  \begin{align*}
  & \mathbb{E}_{n-1}\Ltwo{(A-a_n a_n^\top )\Delta_{n-1}+a_n\varepsilon_n}^4 \\
  \leq \;\; & 8\Ltwo{\Delta_{n-1}}^4 \mathbb{E}_{n-1}\opnorm{A-a_n a_n^\top }^4+ 8 \mathbb{E}_{n-1}(\varepsilon^4_n\Ltwo{a_n}^4).
  \end{align*}
  
  To track the dependence on dimension $d$, we consider the case where $a_n\sim \mathcal{N}(0, I_d), \varepsilon\sim \mathcal{N}(0,1)$.
  Lemma \ref{lem:example1} shows that the dimension constant can be chosen as $\dcon=d$.
  
  To track the dependence on dimension $d$, note that $\nabla^2 f(x,\zeta_n)=a_n a_n^\top $, so $\nabla^2 F(x)=I_d$, and $L_F=1$. Next note that $S=\mathbb{E}\varepsilon^2_n a_na_n^\top =\text{var}(\varepsilon_n) I_d$, so $\mathrm{tr}(S)=O(d)$. Moreover $\Sigma_1=0$, while $A=I_d$, so
  \[
  \Sigma_2=\E \opnorm{(a_n a_n^\top -A)}^2=\E (\Ltwo{a_n}^2 -1)^2=O(d^2),\quad \Sigma_3=8 \mathbb{E}(\varepsilon^4_n\Ltwo{a_n}^4)=O(d^2),
  \]
  and
  \[
  \Sigma_3=8 \mathbb{E}(\varepsilon^4_n\Ltwo{a_n}^4)=O(d^2),\quad \Sigma_4=64\E\opnorm{A-a_n a_n^\top }^4\asymp \E(\|a_n\|_2^2-1)^4=O(d^4)
  \]
  And for the Assumption \ref{aspt:third}, because $\nabla^2 f(x,\zeta_n)=a_n a_n^\top $, therefore $L_2=0$, and $L_4= \| \E (a_n a_n^\top )^2-I_d\|\leq \E \|a_n\|_2^4+1=O(d^2)$.
\end{proof}

\begin{lem}
  \label{lem:example2}
  In Example \ref{exp:logistic}, suppose $\Ltwo{a_n}$ has bounded eighth moment. Furthermore, assuming $a_n$ has a strictly positive density with respect to Lebesgue measure and that the iterates are bounded. Assumptions \ref{aspt:convexity} and \ref{aspt:martingale} hold. To track the dependence on dimension $d$, we consider the case where $a_n\sim \mathcal{N}(0, I_d)$, then Assumptions \ref{aspt:convexity}, \ref{aspt:martingale} and  \ref{aspt:third} hold with
  \begin{align*}
  L_F=O(d),\quad \mathrm{tr}(S)=O(d),\quad \Sigma_1=O(d^\frac32),\quad \Sigma_2=O(d^2),\\
  \Sigma_3=O(d^2),\quad\Sigma_4=O(d^4),\quad L_2=O(d^{\frac32}),\quad L_4=O(d^2).
  \end{align*}
  Therefore, the dimension constant $\dcon=O(d)$.
\end{lem}
\begin{proof}
  
  Using the fact that $|\varphi'(x)| \leq \frac{1}{4}$ for any $x$, $\nabla F(x) $ is Lipschitz continuous with the Lipschitz constant $L_F= \frac{1}{4} \E \|a_n\|_2^2$.
  Moreover,
  \begin{equation}\label{eq:logA}
  A=\nabla^2 F(x^*)=\mathbb{E}\frac{a_n a_n^\top }{[1+\exp(\langle a_n, x^* \rangle)][1+\exp(-\langle a_n, x^* \rangle)]},
  \end{equation}
  is positive semi-definite. Under assumption that $a_n$ has a strictly positive density with respect to Lebesgue measure, and that the iterates are bounded, Lemma \ref{lem:nondeg} below guarantees that $\nabla^2 F(x)$ has strictly positive minimum eigenvalue.
  
  For the martingale difference assumption in Assumption \ref{aspt:martingale}, note that
  \begin{equation*}
  \label{eqn:logHess}
  \nabla^2 f(x,\zeta_n)=\frac{a_n a_n^\top }{[1+\exp(\langle a_n, x \rangle)][1+\exp(-\langle a_n,x\rangle)]},
  \end{equation*}
  which implies that $\opnorm{\nabla^2 f(x,\zeta_n)}\leq \Ltwo{a_n}^2$.  As long as $\Ltwo{a_n}$ has bounded eighth moment, Lemma \ref{lem:simplegeneral} applies, which establishes Assumption \ref{aspt:martingale}.
  
  To track the dependence on dimension $d$, recall that
  \[
  \nabla f(x,\zeta_n)=\frac{-b_n a_n}{1+\exp(-b_n\langle x, a_n\rangle)},
  \]
  \[
  \nabla^2 f(x,\zeta_n)=\frac{a_n a_n^\top }{[1+\exp(\langle a_n, x \rangle)][1+\exp(-\langle a_n,x\rangle)]}.
  \]
  So $H(\zeta_n)=\|\nabla^2 f(x,\zeta_n)\|\leq \|a_n\|^2_2$. This leads to  $L_F\leq  \E H(\zeta_n)\leq \E \|a_n\|_2^2=O(d)$, and for any integer $m$, $\E H(\zeta)^m=O(d^m)$. Further more, for any integer $m$,
  \[
  \E \|\nabla f(x^*,\zeta)\|^m\leq \E \|a_n\|^m=O(d^{\frac{m}{2}}).
  \]
  So by Lemma \ref{lem:simplegeneral}
  \[
  \mathrm{tr}(S)=O(d),\quad \Sigma_1=O(d^\frac32),\quad \Sigma_2=O(d^2),\quad \Sigma_3=O(d^2),\quad\Sigma_4=O(d^4).
  \]
  And for  Assumption \ref{aspt:third}, we check the Frechet derivative of $\nabla^2 f$. Given a uninorm vector $v\in \mathbb{R}^d$,
  \[
  \lim_\epsilon \frac1\epsilon\nabla^2 f(x+\epsilon v,\zeta_n)- \nabla^2 f(x^*,\zeta_n)=a_n a_n^\top \frac{\langle a_n, v\rangle( \exp(\langle a_n, x \rangle)-\exp(2\langle a_n, x \rangle)]}{[1+\exp(\langle a_n, x \rangle)]^3}.
  \]
  It is bounded by $\|a_n\|^3$, so $L_2=O(d^\frac32)$. Then $\nabla^2 f(x,\zeta_n)\preceq \|a_n\|^2I_d$, so $L_4\lesssim O(d^2)$.
\end{proof}

\begin{lem}
  \label{lem:nondeg}  
  Let $a$ be a random variable  on $\R^d$ with strictly positive density with respect to the Lebesgue measure, and $f$ be a continuous strictly positive function on $\R^d$. Then
  \[
  A:=\E f(a)a a^\top 
  \]
  is a PSD matrix with the minimum eigenvalue strictly above zero.
\end{lem}
\begin{proof}
  Let $x$ be a unit-norm eigenvector of $A$ associated with the minimum eigenvalue of $A$. It suffices to show that $x^\top A x>\mu_0$. Let $B$ be a neighborhood of $x$ such that for any $y\in B$, $\langle y, x\rangle>\frac{1}{2}$, and $f(y)>m$ for a constant $m>0$. Then if we denote the density of $a$ as $\mu(dy)$, we have
  \begin{align*}
  x^\top A x&=\int \mu(dy) f(y)xyy^\top  x\geq\int_B\mu(dy) f(y)xyy^\top  x\geq\frac{1}{4} \mu(B) m=:\mu_0.
  \end{align*}
\end{proof}

\section{Supplement to Section \ref{sec:setup_assump}}
\label{sec:supp_setup}
\subsection{Proof of Lemma \ref{lem:simplegeneral}}
Before we prove Lemma \ref{lem:simplegeneral}, we introduce the following Lemma \ref{lem:4th},
\label{supp:setup_assump}
\begin{lem}
  \label{lem:4th}
  The following holds for any vectors:
  \[
  n^3(\Ltwo{x_1}^4+\ldots+\Ltwo{x_n}^4)\geq \Ltwo{x_1+\ldots+x_n}^4
  \]
\end{lem}
\begin{proof}
  Apply Cauchy-Schwartz inequality twice:
  \[
  n(\Ltwo{x_1}^4+\ldots+\Ltwo{x_n}^4)\geq (\Ltwo{x_1}^2+\ldots+\Ltwo{x_n}^2)^2
  \]
  and
  \[
  n(\Ltwo{x_1}^2+\ldots+\Ltwo{x_n}^2)\geq (\Ltwo{x_1}+\ldots+\Ltwo{x_n})^2.
  \]
  Note that $\Ltwo{x_1}+\ldots+\Ltwo{x_n}\geq \Ltwo{x_1+\ldots+x_n} $.
\end{proof}

\begin{replemma}{lem:simplegeneral}
  If there is a function $H(\zeta)$ with bounded fourth moment, such that the Hessian of $f(x,\zeta)$  is bounded by
  \begin{equation}\label{eq:simplegeneral}
  \opnorm{\nabla^2 f(x,\zeta)}\leq H(\zeta)
  \end{equation}
  for all $x$, and  $\nabla f(x^*,\zeta)$ have a bounded fourth moment, then Assumption \ref{aspt:martingale} holds, where
  \[
  \Sigma_1=2 \sqrt{\mathbb{E}\Ltwo{\nabla f(x^*,\zeta)}^2 \mathbb{E}H(\zeta)^2},\quad \Sigma_2=4  \mathbb{E}H(\zeta)^2.
  \]
  \[
  \Sigma_4=\max\{ 64  \mathbb{E}H(\zeta)^4 , 8\mathbb{E}\Ltwo{\nabla f(x^*,\zeta)}^4\}.
  \]
\end{replemma}
\begin{proof}
  The first item of Assumption \ref{aspt:martingale} holds by definition.
  For notational simplicity, define
  \[
  L(\Delta_{n-1},\zeta_n)=\left(\nabla F(x_{n-1})-\nabla F(x^*)\right)-\left(f(x_{n-1},\zeta_n)-\nabla f(x^*,\zeta_n)\right),
  \]
  By the definition of $\xi_n$ and notice that $\nabla F(x^*)=0$, we have,
  \begin{equation}\label{eq:xi_decomp}
  \xi_n=L(\Delta_{n-1},\zeta_n)-\nabla f(x^*,\zeta_n).
  \end{equation}
  To prove Lemma \ref{lem:simplegeneral}, we need the fourth moment bounds on both $L(\Delta_{n-1},\zeta_n)$ and $\nabla f(x^*,\zeta_n)$. By our assumption, $\nabla f(x^*,\zeta_n)$ has the bounded fourth  moment. Now we upper bound $\mathbb{E}_{n-1}\Ltwo{L(\Delta_{n-1},\zeta_n)}^m$ for $m=2$ and $m=4$.
  By the mean-value theorem and our assumption Lemma \ref{lem:simplegeneral}, for $m=2$ and  $4$, we have
  \begin{align*}
  \mathbb{E}_{n-1}\|\nabla f(x_{n-1},\zeta_n)-\nabla f(x^*,\zeta_n)\|_2^m\leq &
  \mathbb{E}_{n-1}\left[\sup_{x}\opnorm{\nabla^2 f(x,\zeta_n)}^m\right] \|\Delta_{n-1}\|_2^m
  \\
  \leq &  \|\Delta_{n-1}\|_2^m\mathbb{E}H(\zeta)^m.
  \end{align*}
  By the convexity of $\|x\|_2^m$ and  Jensen's inequality,
  \begin{align*}
  \|\nabla F(x_{n-1})-\nabla F(x^*)\|_2^m&=\|\mathbb{E}_{n-1}\left(\nabla f(x_{n-1},\zeta_n)-\nabla f(x^*,\zeta_n)\right)\|_2^m\\
  &\leq \mathbb{E}_{n-1}\|\nabla f(x_{n-1},\zeta_n)-\nabla f(x^*,\zeta_n)\|_2^m\\
  & \leq  \|\Delta_{n-1}\|_2^m\mathbb{E}H(\zeta)^m.
  \end{align*}
  By Lemma \ref{lem:4th}, for $m=2$ or $4$,
  \[
  \mathbb{E}_{n-1}\Ltwo{L(\Delta_{n-1},\zeta_n)}^m\leq   2^m\|\Delta_{n-1}\|_2^m\mathbb{E}H(\zeta)^m.
  \]
  Therefore, by the decomposition of $\xi_n$ in \eqref{eq:xi_decomp},
  \begin{align*}
  \mathbb{E}_{n-1}\xi_n\xi_n^\top &=S-\mathbb{E}_{n-1}\nabla f(x^*,\zeta_n) L(\Delta_{n-1},\zeta_{n})^\top -
  \mathbb{E}_{n-1} L(\Delta_{n-1},\zeta_n)\nabla f(x^*,\zeta_n)^\top \\
  &\quad+\mathbb{E}_{n-1} L(\Delta_{n-1},\zeta_n) L(\Delta_{n-1},\zeta_n)^\top .
  \end{align*}
  By Jensen's and Cauchy-Schwartz inequality,
  \begin{align*}
  & \; \opnorm{\mathbb{E}_{n-1}\nabla f(x^*,\zeta_n) L(\Delta_{n-1},\zeta_n)^\top } \\
  \leq & \; \mathbb{E}_{n-1} \opnorm{\nabla f(x^*,\zeta_n) L(\Delta_{n-1},\zeta_n)^\top } \\
  =    & \;   \mathbb{E}_{n-1} \|\nabla f(x^*,\zeta_n)\|_2  \|L(\Delta_{n-1},\zeta_n)^\top \|_2 \\
  \leq & \; [\mathbb{E}_{n-1}\Ltwo{\nabla f(x^*,\zeta_n)}^2 \mathbb{E}_{n-1}\Ltwo{L(\Delta_{n-1},\zeta_n)}^2]^{\frac{1}{2}} \\
  \leq & \; 2 \sqrt{\mathbb{E}\Ltwo{\nabla f(x^*,\zeta)}^2 \mathbb{E}H(\zeta)^2} \|\Delta_{n-1}\|_2
  \end{align*}
  A similar bound holds for the trace, since
  \begin{align*}
  & \; \trnorm{\mathbb{E}_{n-1}\nabla f(x^*,\zeta_n) L(\Delta_{n-1},\zeta_n)^\top } \\
  \leq & \; \mathbb{E}_{n-1} \trnorm{\nabla f(x^*,\zeta_n) L(\Delta_{n-1},\zeta_n)^\top } \\
  \leq   & \;   \mathbb{E}_{n-1} \|\nabla f(x^*,\zeta_n)\|_2  \|L(\Delta_{n-1},\zeta_n)^\top \|_2.
  \end{align*}
  For the second order term, note that
  \begin{align*}
  & \; \opnorm{\mathbb{E}_{n-1}L(\Delta_{n-1},\zeta_n) L(\Delta_{n-1},\zeta_n)^\top } \\
  \leq & \; \mathbb{E}_{n-1} \opnorm{L(\Delta_{n-1},\zeta_n) L(\Delta_{n-1},\zeta_n)^\top }
  =      \mathbb{E}_{n-1}  \|L(\Delta_{n-1},\zeta_n)\|^2_2 \\
  \leq & \; 4  \mathbb{E}H(\zeta)^2 \|\Delta_{n-1}\|_2^2.
  \end{align*}
  \begin{align*}
  & \; \trnorm{\mathbb{E}_{n-1}L(\Delta_{n-1},\zeta_n) L(\Delta_{n-1},\zeta_n)^\top } \\
  \leq & \; \mathbb{E}_{n-1} \trnorm{L(\Delta_{n-1},\zeta_n) L(\Delta_{n-1},\zeta_n)^\top }
  =      \mathbb{E}_{n-1}  \|L(\Delta_{n-1},\zeta_n)\|^2_2.
  \end{align*}
  Combining the above inequalities, we arrive at Assumption \ref{aspt:martingale} Claim 2.

  Using the decomposition $\xi_n=L(\Delta_{n-1},\zeta_n)-\nabla f(x^*,\zeta_n)$ and the fourth moment bounds on $L(\Delta_{n-1},\zeta_n)$ and $\nabla f(x^*,\zeta_n)$, applying Lemma \ref{lem:4th} we get Claim 3.
\end{proof}

\subsection{Proof of Lemma \ref{lem:Delta}}
\label{supp:Delta}

\begin{replemma}{lem:Delta}[Generalized version]
  Under Assumptions \ref{aspt:convexity} and \ref{aspt:martingale}, if the step size is chosen to be $\eta_n = \eta n^{-\alpha}$ with $\alpha \in (0,1)$, the iterates of error $\Delta_n=x_n-x^*$ satisfy the following.
  \begin{enumerate}[1)]
    \item There exist universal constants $C\asymp \dcon$ and $n_0$,
    such that for $n>m\geq n_0$, the following hold:
    \begin{align*}
    \mathbb{E}_m \|\Delta_n\|_2 & \leq \exp\left(-\tfrac{1}{4}\mu\sum_{i=m}^n \eta_i \right)\|\Delta_m\|_2+\sqrt{C}m^{-\alpha/2}, \\
    \mathbb{E}_m \|\Delta_n\|_2^2 & \leq \exp\left(-\tfrac{1}{2}\mu\sum_{i=m}^n \eta_i\right)\|\Delta_m\|_2^2+Cm^{-\alpha}, \\
    \mathbb{E}_m \|\Delta_n\|_2^4 & \leq \exp\left(-\tfrac{1}{2}\mu\sum_{i=m}^n \eta_i\right)\|\Delta_m\|_2^4+C^2m^{-2\alpha}.
    \end{align*}
    \item As a consequence,
    \begin{align*}
    \mathbb{E} \|\Delta_n\|_2 & \lesssim  n^{-\alpha/2}(\sqrt{C}+\|\Delta_{n_0}\|_2),\\
    \mathbb{E} \|\Delta_n\|_2^2 & \lesssim  n^{-\alpha}(C+\|\Delta_{n_0}\|_2^2), \\
    \mathbb{E} \|\Delta_n\|_2^4 & \lesssim n^{-2\alpha}(C^2+\|\Delta_{n_0}\|_2^4).
    \end{align*}
\end{enumerate}\end{replemma}

To prove Lemma \ref{lem:Delta}, we need a simple lemma.
\begin{lem}
  \label{lem:sequence}
  Let $z_n$ be a sequence in $\mathbb{R}^+$ that satisfies the recursion
  \[
  z_n\leq (1-\lambda\eta_n)z_{n-1}+D \eta_n^{2+\beta},
  \]
  where $D,\lambda,\beta>0$ are fixed constants, and the sequence $\eta_n$ is decreasing.  Then for any $m \leq n-1$,
  \[
  z_n\leq \exp(-\lambda(n-m)\eta_n)z_m+D\eta_m^{1+\beta} \lambda^{-1}.
  \]
\end{lem}
\begin{proof}
  Using the recursion, we can find that
  \[
  z_n\leq \prod_{k=m+1}^n(1-\lambda \eta_k)z_m+D\sum_{k=m+1}^n\prod_{j=k+1}^n(1-\lambda\eta_j) \eta_k^2.
  \]
  Note that
  \[
  \lambda\sum_{k=m+1}^n\prod_{j=k+1}^n(1-\lambda\eta_j) \eta_k=1-\prod_{j=m+1}^n(1-\lambda\eta_j)\leq 1.
  \]
  Since $\eta_k\leq \eta_m$ for all $k\geq m+1$,
  \[
  D\sum_{k=m+1}^n\prod_{j=k+1}^n(1-\lambda\eta_j) \eta_k^{2+ \beta}
  \leq D\lambda^{-1}\eta^{1+\beta}_m\lambda\sum_{k=m+1}^n\prod_{j=k+1}^n(1-\lambda\eta_j) \eta_k
  \leq D\lambda^{-1}\eta_m^{1+\beta}.
  \]
  On the other hand, since $1-x\leq \exp(-x)$ for $x\in [0,1]$, so
  \[
  \prod_{k=m+1}^n(1-\lambda \eta_k)\leq \exp\left(-\lambda\sum_{k=m+1}^n\eta_k\right)\leq \exp(-\lambda(n-m)\eta_n).
  \]
  Putting them back to the first inequality of this proof, we obtain our claim.
\end{proof}

With Lemma \ref{lem:sequence} in place, we are ready to prove Lemma \ref{lem:Delta}.

\begin{proof}[Proof of Lemma \ref{lem:Delta}]
  From the recursion equation \eqref{eq:sgd_general}, we find that
  \[
  \Delta_n=\Delta_{n-1}-\eta_n \nabla \widetilde{F}(\Delta_{n-1})+\eta_n \xi_n,
  \]
  where the shifted averaged loss function is $\widetilde{F}(\Delta)=F(\Delta+x^*)$. Therefore
  \begin{equation}
  \label{tmps:Delta21}
  \|\Delta_n\|_2^2=\|\Delta_{n-1}\|_2^2-2\eta_n\langle \nabla \widetilde{F}(\Delta_{n-1}), \Delta_{n-1} \rangle+2\eta_n\langle \xi_n , \Delta_{n-1} \rangle+\eta_n^2\|\xi_n-\nabla \widetilde{F}(\Delta_{n-1})\|_2^2.
  \end{equation}
  The square of \eqref{tmps:Delta21} is
  \begin{align*}
  &\|\Delta_n\|_2^4\\
  =&\|\Delta_{n-1}\|_2^4-4\eta_n\langle \nabla \widetilde{F}(\Delta_{n-1}), \Delta_{n-1} \rangle \|\Delta_{n-1}\|_2^2+4\eta_n\langle \xi_n , \Delta_{n-1} \rangle\|\Delta_{n-1}\|_2^2\\
  +&4\eta_n^2 \langle \nabla \widetilde{F}(\Delta_{n-1})-\xi_n, \Delta_{n-1} \rangle^2+2\eta^2_n\|\Delta_{n-1}\|^2\|\xi_n-\nabla \widetilde{F}(\Delta_{n-1})\|_2^2\\
  -&4\eta_n^3 \langle \nabla \widetilde{F}(\Delta_{n-1})-\xi_n, \Delta_{n-1}\rangle\|\xi_n-\nabla \widetilde{F}(\Delta_{n-1})\|^2+\eta_n^4\|\xi_n-\nabla \widetilde{F}(\Delta_{n-1})\|_2^4.
  \end{align*}
  By the property of the inner product, the fourth term above is less than twice of the fifth, which can be bounded by the following using Young's inequality,
  \begin{align*}
  2\eta^2_n\|\Delta_{n-1}\|_2^2\|\xi_n-\nabla \widetilde{F}(\Delta_{n-1})\|_2^2 \leq \frac{1}{3}\mu\eta_n\|\Delta_{n-1}\|_2^4+3\mu^{-1}\eta_n^3\|\xi_n-\nabla \widetilde{F}(\Delta_{n-1})\|_2^4;
  \end{align*}
  Moreover, H\"{o}lder's inequality gives the following bound for the sixth term:
  \[
  -4\eta_n^3 \langle \nabla \widetilde{F}(\Delta_{n-1})-\xi_n, \Delta_{n-1}\rangle\|\xi_n-\nabla \widetilde{F}(\Delta_{n-1})\|_2^2\leq \eta^3_n\|\Delta_{n-1}\|_2^4+3 \eta^3_n \|\xi_n-\nabla \widetilde{F}(\Delta_{n-1})\|_2^4.
  \]
  Combining all the inequality above back to the expansion of $\Ltwo{\Delta_n}^4$, we obtain
  \begin{eqnarray}
  \|\Delta_n\|_2^4&\leq &\|\Delta_{n-1}\|_2^4-4\eta_n\langle \nabla \widetilde{F}(\Delta_{n-1}), \Delta_{n-1} \rangle \|\Delta_{n-1}\|_2^2+4\eta_n\langle \xi_n , \Delta_{n-1} \rangle\|\Delta_{n-1}\|_2^2\cr
  \label{tmp:Delta41}
  && +\mu \eta_n \|\Delta_{n-1}\|_2^4+\eta_n^3\|\Delta_{n-1}\|_2^4+\eta_n^3(9\mu^{-1}+\eta_n+3)\|\xi_n-\nabla \widetilde{F}(\Delta_{n-1})\|_2^4.
  \end{eqnarray}
  Then using strong convexity of $\widetilde{F}$,
  \[
  \langle \nabla \widetilde{F}(\Delta_{n-1}), \Delta_{n-1} \rangle\geq  \widetilde{F}(\Delta_{n-1})+\frac{\mu}{2} \|\Delta_{n-1}\|_2^2\geq \frac{\mu}{2} \|\Delta_{n-1}\|_2^2.
  \]
  By the fact that $\xi_n$ is a martingale difference,
  \[
  \mathbb{E}_{n-1}\langle \xi_n , \Delta_{n-1} \rangle\|\Delta_{n-1}\|_2^2=0.
  \]
  Applying Young's inequality to Assumption \ref{aspt:martingale} over $\xi_n$, using the Lipschitzness of $\Delta F$ in Assumption \ref{aspt:convexity}
  \begin{align*}
  \mathbb{E}_{n-1}\|\xi_n-\nabla \widetilde{F}(\Delta_{n-1})\|_2^2 & \leq 2\mathbb{E}_{n-1}\|\xi_n\|_2^2+2\|\nabla \widetilde{F}(\Delta_{n-1})\|_2^2\\ & \leq 2\tr(S)+2\Sigma_2\Ltwo{\Delta_{n-1}}^2+2\Sigma_1\Ltwo{\Delta_{n-1}}
  +2L_F^2\|\Delta_{n-1}\|_2^2\\
  & \leq 2\tr(S)+\Sigma_1^\frac23+(2\Sigma_2+\Sigma_1^\frac43)\Ltwo{\Delta_{n-1}}^2+2L_F^2\|\Delta_{n-1}\|_2^2,
  \end{align*}
  and likewise with Lemma \ref{lem:4th}
  \begin{align*}
  \mathbb{E}_{n-1}\|\xi_n-\nabla \widetilde{F}(\Delta_{n-1})\|_2^4 & \leq 4\mathbb{E}_{n-1}\|\xi_n\|_2^4+4\|\nabla \widetilde{F}(\Delta_{n-1})\|_2^4\\ & \leq 4\Sigma_3+4(\Sigma_4+L_F^4)\|\Delta_{n-1}\|_2^4.
  \end{align*}
  Replacing propers terms  in \eqref{tmps:Delta21} and \eqref{tmp:Delta41} with upper bounds above, we can find a constant $C_1\asymp\dcon$,
  \begin{equation}
  \label{tmp:Delta2}
  \mathbb{E}_{n-1}\|\Delta_n\|_2^2\leq (1-\mu\eta_n+C^2_1\eta_n^2) \|\Delta_{n-1}\|_2^2+C_1\eta_n^2.
  \end{equation}
  \begin{equation}
  \label{tmp:Delta4}
  \mathbb{E}_{n-1}\|\Delta_n\|_2^4\leq (1-\mu \eta_n+C^4_1(\eta_n^3+\eta_n^4))\|\Delta_{n-1}\|_2^4+C^2_1\eta_n^3(\eta_n+1).
  \end{equation}
  Let
  \[
  n_0=\min\left\{n: C^2_1\eta_n^2 \leq \frac{1}{2}\mu\eta_n, C^4_1\eta_n^3(\eta_n+1)\leq \frac{1}{2}\mu\eta_n, \mu\eta_n(1-\alpha)\geq 8\alpha\log n\right\},
  \]
  we have $n_0\lesssim (\dcon^2\eta)^\frac{1}{\alpha}\asymp 1$. Then for $n\geq n_0$, and a $C_2\asymp C_1$, the inequalities can be simplified by
  \begin{eqnarray*}
    \mathbb{E}_{n-1}\|\Delta_n\|_2^2 & \leq  & \left(1-\frac{1}{2}\mu \eta_n\right)\|\Delta_{n-1}\|_2^2+C_2\eta_n^2, \\
    \mathbb{E}_{n-1}\|\Delta_n\|_2^4 & \leq & \left(1-\frac{1}{2}\mu \eta_n\right)\|\Delta_{n-1}\|_2^4+C^2_2\eta_n^3.
  \end{eqnarray*}
  Then Lemma \ref{lem:sequence} produces claim 1) for the second and forth moment, while we use Cauchy-Schwartz inequality it gives the bound for the first moment, as
  \[
  \mathbb{E}_{m}\|\Delta_n\|_2\leq \sqrt{\mathbb{E}_{m} \|\Delta_n\|_2^2}.
  \]
  To see the claim 2), note that
  \[
  \max\{C^2_1\eta_i^2, C^4_1 (\eta_i^3+\eta_i^4)\}\leq C^4_1\eta_i^2(\eta_i+1)^2,
  \]
  therefore applying discrete Gronwall's inequality to \eqref{tmp:Delta2} and \eqref{tmp:Delta4}, we find that
  \[
  \mathbb{E} \|\Delta_{n_0}\|_2^2\leq C_3(1+\|\Delta_0\|_2^2),\quad\mathbb{E} \|\Delta_{n_0}\|_2^4\leq C_3(1+\|\Delta_0\|_2^4),
  \]
  where
  \[
  C_3=\bigg(\prod_{i=1}^{n_0}(1+C^4_1\eta_i^2(\eta_i+1)^2)\bigg)\bigg(C^4_1\sum_{i=1}^{n_0} \eta_i^2(\eta_i+1)^2\bigg)\asymp 1.
  \]
  Recall that we assume $C^2_d\eta_i\lesssim 1$, so $n_0\asymp 1$. Then for $n\geq 2n_0$, using the claims from 1),  there is a constant $C_4\asymp C_1$ such that
  \[
  \mathbb{E}_{n_0}\|\Delta_{\frac{n}{2}}\|_2^2\leq \|\Delta_{n_0}\|_2^2+C_4 \eta_{n_0},\quad
  \mathbb{E}_{n_0}\|\Delta_{\frac{n}{2}}\|_2^4\leq \|\Delta_{n_0}\|_2^4+C^2_4\eta_{n_0}^2,
  \]
  and
  \begin{align*}
  \mathbb{E}_{\frac{n}{2}}\|\Delta_n\|_2^2\leq  & \exp\left(-\tfrac{1}{4}n\mu \eta_n\right) \|\Delta_{\frac{n}{2}}\|_2^2+C_4(\tfrac{n}{2})^{-\alpha},  \\
  \mathbb{E}_{\frac{n}{2}}\|\Delta_n\|_2^4\leq  &  \exp\left(-\tfrac{1}{2}n\mu \eta_n\right) \|\Delta_{\frac{n}{2}}\|_2^4+C^2_4(\tfrac{n}{2})^{-2\alpha}.
  \end{align*}
  So using the law of iterated expectation, and that
  \[
  \exp(-\tfrac{1}{4}n\mu \eta_n)\leq \exp(-\tfrac{1}{4}\mu\eta n^{1-\alpha})\leq n^{-2\alpha}\leq n^{-\alpha}
  \]
  there is a constant $C_5\asymp1$ that
  \[
  \mathbb{E}\|\Delta_n\|_2^2\leq  n^{-\alpha} C_5\left(\|\Delta_0\|_2^2+C_4\right),\quad
  \mathbb{E}\|\Delta_n\|_2^4\leq  n^{-2\alpha} C_5\left(\|\Delta_0\|_2^4+C^2_4\right).
  \]
  Applying Cauchy-Schwartz inequality again yields the claim for the first moment.
\end{proof}

\subsection{Lemma \ref{lem:lindelta} and its proof}
Before we move on, let us apply the same arguments and prove a similar simpler result for the linear oracle sequence $U_n$.
\begin{lem}
  \label{lem:lindelta}
  For the linear oracle sequence $U_n$ defined by \eqref{eqn:linear}, it satisfies the following:
  \[
  \mathbb{E} \|U_n\|_2^2  \lesssim  n^{-\alpha}(\dcon+\|U_0\|^2_2+\|\Delta_0\|^2_2).
  \]
\end{lem}
\begin{proof}
  Clearly, we have the following recursion
  \[
  \|U_n\|_2^2=\|U_{n-1}\|_2^2-2\eta_n\langle A U_{n-1}, U_{n-1} \rangle+2\eta_n\langle \xi_n , U_{n-1} \rangle+\eta_n^2\|\xi_n-A U_{n-1}\|_2^2.
  \]
  Note that
  \[
  \langle A U_{n-1}, U_{n-1} \rangle\geq \frac{\mu}{2} \|U_{n-1}\|_2^2,
  \]
  and also that $\xi_n$ is a martingale difference, so
  \[
  \mathbb{E}_{n-1}\langle \xi_n , U_{n-1} \rangle=0.
  \]
  Moreover by Assumption \ref{aspt:martingale}, there is  a constant $C_1\asymp \dcon$ such that
  \begin{align*}
  \mathbb{E}_{n-1}\|\xi_n-A U_{n-1}\|_2^2&\leq 2 \mathbb{E}_{n-1}\Ltwo{\xi_n}^2+2\opnorm{A}^2\Ltwo{U_{n-1}}^2\\
  &\leq 2(\tr(S)+\Sigma_1\|\Delta_{n-1}\|_2+\Sigma_2 \|\Delta_{n-1}\|^2_2 )+2 L_F^2 \Ltwo{U_{n-1}}^2 \\
  &\leq C_1+C_1^2(\Ltwo{\Delta_{n-1}}^2+\Ltwo{U_{n-1}}^2).
  \end{align*}
  Here we have used that $\nabla F$ being Lipschitz implies $\|A\|=\|\nabla^2 F(x^*)\|\leq L_F$.
  
  We find
  \[
  \mathbb{E}_{n-1}\|U_n\|_2^2\leq (1-\mu\eta_n+C^2_1\eta_n^2) \|U_{n-1}\|_2^2+C_1(\eta_n^2+C_1\Ltwo{\Delta_{n-1}}^2).
  \]
  Summing this inequality with \eqref{tmp:Delta2} yields the following with $z_n:=\|U_n\|^2_2+\|\Delta_n\|^2_2$ and $C_2=2C_1$
  \[
  \mathbb{E}_{n-1}z_n\leq (1-\mu\eta_n+C_2^2\eta_n^2) z_{n-1}+C_2\eta_n^2.
  \]
  Then redoing the analysis in the proof of Lemma \ref{lem:Delta} after \eqref{tmp:Delta2} yields that
  \[
  \mathbb{E} \Ltwo{U_n}^2\leq \E z_n\lesssim n^{-\alpha}(C_2+\|\Delta_0\|^2_2+\|U_0\|^2_2).
  \]
\end{proof}

\section{Consistency Proofs of the Plug-in Estimator in Section \ref{sec:plugin}}
\label{sec:supp_plug}
\subsection{Proof of Lemma \ref{lem:A_con} for the consistency of $A_n$ and $S_n$}
\label{supp:A_con}
\begin{replemma}{lem:A_con}
  Under Assumptions \ref{aspt:convexity}, \ref{aspt:martingale} and \ref{aspt:third}, the followings hold
  \[
  \mathbb{E}\|A_n-A\|\lesssim \dcon^2 n^{-\frac\alpha2}, \quad \mathbb{E}\|S_n-S\|\lesssim \dcon^2 n^{-\frac\alpha2}+\dcon^3 n^{-\alpha},
  \]
  where $\alpha$ is given {in the step size sequence $\eta_i=\eta i^{-\alpha}$, $i=1,2,\dots, n$}.
  
\end{replemma}
\begin{proof}
  We decompose $A_n-A$ into the following terms:
  \begin{align}
  \notag
  A_n-A &= \frac1n \sum_{i=1}^n \nabla ^2 f(x_{i-1} , \zeta_i)  -A \\
  \label{tmp:AnA}
  &= \left(\frac1n \sum_{i=1}^n \nabla ^2 f(x^*, \zeta_i )-A\right) + \frac1n \sum_{i=1}^n \left(\nabla^2f(x_{i-1},\zeta_i)-\nabla^2 f(x^*,\zeta_i)\right).
  \end{align}
  The first part can be bounded using the second moment bound, since by independence of $\zeta_i$ and $\zeta_j$,
  \[
  \mathbb{E}(\nabla^2 f(x^*, \zeta_i)-A) (\nabla^2 f(x^*,\zeta_j)-A)^\top 
  =\unit_{i=j}[\mathbb{E}[\nabla^2 f(x^*,\zeta_i)][\nabla^2 f(x^*,\zeta_i)]^\top -AA^\top ].
  \]
  Therefore,
  \begin{align*}
  &\mathbb{E}\left(\frac1n \sum_{i=1}^n \nabla ^2 f(x^*, \zeta_i )-A \right)\left(\frac1n \sum_{i=1}^n \nabla ^2 f(x^*, \zeta_i )-A\right)^\top \\
  =&\frac{1}{n^2}\sum_{i=1}^n[\mathbb{E}[\nabla^2 f(x^*,\zeta_i)][\nabla^2 f(x^*,\zeta_i)]^\top -AA^\top ]
  \end{align*}
  which has its norm bounded by $\frac1n L_4$. Note that by Cauchy-Schwartz inequality, for any matrix $B$
  \[
  \mathbb{E} \opnorm{B B^\top }=\mathbb{E}\opnorm{B}^2\geq [\mathbb{E}\opnorm{B}]^2.
  \]
  Therefore we have
  \[
  \mathbb{E}\opnorm{\frac1n \sum_{i=1}^n \nabla ^2 f(x^*, \zeta_i )-A}
  \leq \frac{\sqrt{L_4}}{\sqrt{n}}\lesssim C_d n^{-\frac{1}{2}}.
  \]
  The second term in \eqref{tmp:AnA} can be bounded using Assumption \ref{aspt:third},
  \begin{align*}
  & \opnorm{\mathbb{E}\frac1n \sum_{i=1}^n (\nabla^2f(x_{i-1},\zeta_i)-\nabla^2 f(x^*,\zeta_i))} \\
  \leq & \frac1n\mathbb{E}\sum_{i=1}^n \opnorm{\nabla^2f(x_{i-1},\zeta_i)-\nabla^2 f(x^*,\zeta_i)}
  \leq  \frac{L_2}{n}\sum_{i=1}^n\mathbb{E}\Ltwo{x_i-x^*}.
  \end{align*}
  Using Lemma \ref{lem:Delta} 2), $\mathbb{E}\|x_i-x^*\|\lesssim  \sqrt{\dcon} i^{-\frac{\alpha}{2}}$. Note that by Lemma \ref{lem:approx} 5) in below
  \[
  \sum_{i=1}^n i^{-\frac\alpha2}\lesssim n^{1-\frac\alpha 2}.
  \]
  So the second term  in \eqref{tmp:AnA} is of order $\dcon^2 n^{-\frac{\alpha}{2}}$, which concludes the consistency proof for $A_n$.
  
  For $S_n$, we  decompose
  \[
  \nabla f(x_{i-1}, \zeta_i) = \nabla f(x^*, \zeta_i) + (\nabla f(x_{i-1}, \zeta_i)-\nabla f(x^*, \zeta_i)):=X_i+Y_i,
  \]
  then
  \[
  S_n-S=\left(\frac{1}{n}\sum_{i=1}^n X_iX_i^\top -S\right)+\frac1n\sum_{i=1}^n X_iY_i^\top +\frac1n \sum_{i=1}^nY_i X_i^\top +\frac1n \sum_{i=1}^nY_iY^\top _i.
  \]
  Note that
  \[
  X_i=\nabla f(x^*, \zeta_i) =\nabla f(x^*, \zeta_i)-\nabla F(x^*)
  \]
  which is $\xi_n$ when $\Delta_{n-1}=0$, so by part 3 of Assumption \ref{aspt:martingale}, we know the fourth moment of $\|X_i\|$ is bounded by $\Sigma_3$. Then by $\E X_iX_i^\top =S$ and the independence between the $X_i$'s,
  \[
  \E \opnorm{\frac{1}{n}\sum_{i=1}^n X_iX_i^\top -S}^2
  \leq \tr\left(\E\left(\frac{1}{n}\sum_{i=1}^n X_iX_i^\top -S\right)^2 \right)
  =\frac{1}{n}\E \tr((X_iX_i^\top )^2-S^2).
  \]
  Note that $\E\tr(X_iX_i^\top )^2=\E (\tr(X_iX_i^\top ))^2\leq \Sigma_3$. By Cauchy-Schwartz inequality,
  \[
  \mathbb{E}\opnorm{\frac{1}{n}\sum_{i=1}^n X_iX_i^\top -S }\lesssim  \sqrt{\Sigma_3}n^{-\frac12}.
  \]
  Further,  by Assumption \ref{aspt:third} and Lemma \ref{lem:Delta},
  \[
  \begin{gathered}
  \mathbb{E} \opnorm{Y_i Y_i^\top }  \leq L^2_F \mathbb{E}\opnorm{\Delta_{i-1}}^2\lesssim \dcon^3 i^{-\alpha} \\
  \mathbb{E}\opnorm{X_i X_i^\top }   \lesssim \dcon \\
  \mathbb{E}\opnorm{X_i Y_i^\top }  \leq [\mathbb{E}\opnorm{X_i}^2\mathbb{E}\opnorm{Y_i}^2]^{\frac12}\lesssim \dcon^2 i^{-\frac\alpha2}.
  \end{gathered}
  \]
  Summing up all the terms, we obtain $\mathbb{E}\opnorm{S_n-S}\lesssim \dcon^2 n^{-\frac\alpha2}+\dcon^3 n^{-\alpha}$.
\end{proof}

\subsection{Proof of Theorem \ref{thm:plugin} for consistency of the plug-in estimator}
\label{subsec:plugin}
Before we provide the proof the consistency result of plug-in estimator in Theorem \ref{thm:plugin}, we first provide the following lemma on matrix perturbation inequality for the inverse of a matrix.
\begin{lem}[Matrix perturbation inequality for the inverse]
  \label{lem:matrixinv}
  Let $B = A+E$, where $A$ and $B$ are assumed to be invertible, and $\opnorm{A^{-1}E}<\frac12 $.
  We have
  \begin{equation}\label{eq:matrix_inverse}
  \opnorm{ B^{-1} - A^{-1}}  \leq 2 \opnorm{E} \opnorm{A^{-1}}^2
  \end{equation}
\end{lem}
\begin{proof}
  Using \cite[Eq. 3]{chang2006inversion}, we have
  \[
  B^{-1} - A^{-1} = - A^{-1} E( I +A^{-1} E)^{-1} A^{-1},
  \]
  which implies that
  \begin{align*}
  \opnorm{B^{-1} - A^{-1}}  = & \opnorm{ A^{-1} E( I +A^{-1} E)^{-1} A^{-1}}\\
  \le & \opnorm{A^{-1}}^2 \opnorm{E} \opnorm{(I+A^{-1}E)^{-1}}\\
  \le  & \opnorm{A^{-1}}^2 \opnorm{E} \frac{1}{ \lambda_{\min} ( I+A^{-1} E)}\\
  \le & \opnorm{A^{-1}}^2 \opnorm{E}\frac{1} { 1 - \opnorm{A^{-1}E}}
  \end{align*}
  Here $\lambda_{\min}$ is the minimum eigenvalue, and the last inequality uses Weyl's inequality $\lambda_{\min} ( A+B) \ge \lambda_{\min} (A) -\opnorm{B}$.
  The final result in \eqref{eq:matrix_inverse} follows the assumption that $\opnorm{A^{-1} E} < \frac12 $ .
\end{proof}
\begin{reptheorem}{thm:plugin}
  Under Assumptions \ref{aspt:convexity}, \ref{aspt:martingale} and \ref{aspt:third}, the thresholded plug-in estimator initialized from any bounded $x_0$  converges to the asymptotic covariance matrix,
  \[\eqref{eq:plug_in_bound}:
  \mathbb{E}\opnorm{\widetilde{A}_n^{-1} S_n \widetilde{A}_n ^{-1} -A^{-1} S A^{-1}}\lesssim  \|S\|(\dcon^2 n^{-\frac{\alpha}{2}}+\dcon^3 n^{-\alpha}),
  \]
  where $\alpha \in (0,1)$ is given in the step size sequence $\eta_i=\eta i^{-\alpha}$, $i=1,2,\dots, n$. When $C_d$ is a constant, the right hand side of \eqref{eq:plug_in_bound} is dominated by $O(n^{-\frac{\alpha}{2}})$.
\end{reptheorem}
\begin{proof}
  Let us define
  \[
  E_A=\widetilde{A}_n-A,\quad F_A=\widetilde{A}_n^{-1}-A^{-1},\quad F_S=S_n-S,
  \]
  Note that by Lemma \ref{lem:S_con}, $\E\opnorm{F_S}\lesssim \dcon^2 n^{-\frac\alpha2}+\dcon^3 n^{-\alpha}$ while by Markov inequality,
  \begin{align*}
  \mathbb{E}\opnorm{E_A}&\leq \mathbb{E}\opnorm{A_n-A}+\opnorm{A}\mathbb{P}(A_n\prec \delta I)\\
  &\lesssim \mathbb{E}\opnorm{A_n-A}+\mathbb{P}(\opnorm{A_n-A}\geq \lambda_\text{min}(A)-\delta)\\
  & \leq \mathbb{E}\opnorm{A_n-A}+\frac{1}{\lambda_\text{min}(A)-\delta}\mathbb{E}\opnorm{A_n-A}\lesssim \dcon^2  n^{-\alpha/2}.
  \end{align*}
  Then Markov inequality yields further that
  \[
  \mathbb{P}(\opnorm{A^{-1} E_A}\geq 1/2)\leq 2(\mathbb{E}\opnorm{E_A})\lesssim \dcon^2  n^{-\alpha/2}.
  \]
  Note that by construction, the following hold almost surely
  \[
  \opnorm{\widetilde{A}_n^{-1}}\leq \delta^{-1},\quad \opnorm{A^{-1}}\leq \lambda^{-1}_{\text{min}}(A),\quad
  \opnorm{F_A}\leq  \delta^{-1}+\lambda^{-1}_{\text{min}}(A).
  \]
  By Lemma \ref{lem:matrixinv},
  \[
  \opnorm{F_A}\leq \unit_{\opnorm{A^{-1}E_A}\leq 1/2} 2\opnorm{E_A}\opnorm{A^{-1}}^2+\unit_{\opnorm{A^{-1} E_A}\geq 1/2} (\delta^{-1}+\lambda^{-1}_{\text{min}}(A)),
  \]
  therefore
  \begin{align*}
  \mathbb{E} \opnorm{F_A}&\leq 2\opnorm{A^{-1}}^2\mathbb{E}\opnorm{E_A}+(\delta^{-1}+\lambda^{-1}_{\text{min}}(A))\mathbb{P}(\opnorm{A^{-1}E_A}\geq 1/2)\lesssim \dcon^2 n^{-\frac \alpha 2}.
  \end{align*}
  Finally, we can decompose the error in the estimation into the following
  \begin{align*}
  \widetilde{A}_n^{-1} S_n \widetilde{A}_n ^{-1} -&A^{-1} S A^{-1}=
  (A^{-1}+F_A)(S+F_S)(A^{-1}+F_A)-A^{-1}SA^{-1},\\
  &=(A^{-1}+F_A)F_S (A^{-1}+F_A)+A^{-1}SF_A+F_A SA^{-1}+F_A S F_A.
  \end{align*}
  Note that $\|F_A\|\leq \delta^{-1}+\mu^{-1}\asymp 1$, so
  \begin{align*}
  \E \|(A^{-1}+F_A)F_S (A^{-1}+F_A)\|&\lesssim \E \|F_S\|\lesssim \dcon^2 n^{-\frac\alpha2}+\dcon^3 n^{-\alpha},
  \end{align*}
  and
  \begin{align*}
  \|A^{-1}SF_A\|+\|F_A SA^{-1}\|+\|F_A S F_A\|&\lesssim \|S\|(\dcon^2 n^{-\frac{\alpha}{2}}+\dcon^3 n^{-\alpha}).
  \end{align*}
  In their summation, we have our claim.
\end{proof}

\section{Consistency proof of the batch-means estimator}
\label{suppsec:bm}
\subsection{Technical Lemmas}
\label{suppsubsec:twoest}

As a preparation for the consistency result, we introduce the following technical lemmas. Some variants of them have appeared in \cite{PJ92} or \cite{BM11}. Since we need non-asymptotic bounds to investigate the relationship between the estimation error and $M,N$, we redo some key steps in \cite{PJ92, BM11} with simpler notations.

Let us start with some approximations of  various quantities related to the number of batches and batch sizes. These will simplify our discussion afterwards.

\begin{lem}
  \label{lem:approx}
  Under the formulation of \eqref{eq:ending}, the followings hold:
  \begin{enumerate}[1)]
    \item The starting and ending indices of each non starting $(k\neq 0)$ batches,  and their neighbors $s_{k}-1,s_k$ and $e_k, e_k+1$ are of the same order,
    \[
    s_k-1\asymp s_k\asymp e_k\asymp e_k+1\asymp (kN)^{\frac{1}{1-\alpha}}.
    \]
    \item The total number of iterations used in our estimator can be estimated by:
    \[
    S_M\asymp e_M\asymp (MN)^{\frac{1}{1-\alpha}}.
    \]
    \item The number of iterations in each batch can be estimated by:
    \[
    n_k=e_k-e_{k-1}\asymp k^{\frac{\alpha}{1-\alpha}} N^{\frac{1}{1-\alpha}}.
    \]
    \item The sum of stepsize in each batch, which represents the log of decorrelation across that batch, is of order
    \[
    \sum_{i=s_k}^{e_k} \eta_i\geq \eta_{e_k} n_k\gtrsim (kN)^{-\frac{\alpha}{1-\alpha}} k^{\frac{\alpha}{1-\alpha}}N^{\frac{1}{1-\alpha}}=N.
    \]
    \item For any fixed $\gamma\neq 1$, the following estimate holds for sum of geometric sequence
    \[
    M^{-1}\sum_{k=1}^Mk^{-\gamma}\lesssim  M^{-1}\int^{M+1}_1 x^{-\gamma}dx\lesssim M^{(1-\gamma)^+-1}.
    \]
  \end{enumerate}
\end{lem}
\begin{proof}
  The first two inequalities comes from direct verifications, since $\frac{k+1}{k}\leq 2$ for $k\geq 1$. The third one comes from the fact that
  \[
  e_k-e_{k-1}=[(k+1)^{\frac{1}{1-\alpha}}-k^{\frac{1}{1-\alpha}}]N^{\frac{1}{1-\alpha}}=\frac{N^{\frac{1}{1-\alpha}} }{1-\alpha}
  \int^{k+1}_k x^{\frac{\alpha}{1-\alpha}} dx,
  \]
  and then observe that $k^{\frac{\alpha}{1-\alpha}}\leq x^{\frac{\alpha}{1-\alpha}}\leq2^{\frac{\alpha}{1-\alpha}} k^{\frac{\alpha}{1-\alpha}}.$ The fourth one is a corollary of the first and third. The last inequality comes from that $k+1\leq 2k$ for $k\geq 1$ again.
\end{proof}

Recall the linear oracle sequence defined by \eqref{eqn:linear}: $U_n=\left(I-\eta_n A \right) \Delta_{n-1} + \eta_n \xi_n$. We now introduce Lemma \ref{lem:matrixestimate}.
\begin{lem}\label{lem:matrixestimate}
  For each  positive integer $j$, let $Y_j^k$ be a matrix sequence with $Y_j^j=I$ and for $k \geq j+1$
  \[
  Y^{k}_j= (I-\eta_k A) Y_j^{k-1}=\prod_{i=j+1}^k(I-\eta_i A),
  \]
  where $\eta_k=\eta k^{-\alpha}$ with $\alpha\in (1/2,1)$ and $A$ is a PSD matrix with $\lambda_A>0$ be any lower bound on the minimum eigenvalue of $A$. Then
  \begin{enumerate}[1)]
    \item For $i\in \{j,\ldots,k\}$,   the following holds
    \[
    \|Y^k_j\|\leq \exp \left(-\lambda_A \sum_{i=j+1}^k\eta_i \right)\leq \exp\left(-\lambda_A (k-j)\eta_k \right),
    \]
    \item Let $S_j^k=\sum_{i=j+1}^kY^i_j$ and $Z_j^k=\eta_j S_j^k-A^{-1}$, then
    \[
    \opnorm{S_j^k}\lesssim k^{\alpha},\quad \opnorm{Z_j^k}\lesssim  k^{\alpha} j^{-1} +\exp \left(-\lambda_A \sum_{i=j}^{k}\eta_i\right).
    \]
    \item When $ s_1\leq j < k\leq e_M$, for sufficiently large $N$, there exists  a constant $c$ such that
    \[
    \opnorm{S_j^k}\lesssim j^{\alpha},\quad \opnorm{Z_j^k}\lesssim  j^{\alpha-1}+\exp\left(-c\lambda_A (k-j)\eta_j \right),\quad \opnorm{\eta_j S_j^k}\lesssim 1,
    \]
    as long as $M\leq N^m$ for certain positive integer $m$.
  \end{enumerate}
\end{lem}

\begin{proof}
  Assume $Q\Lambda Q^\top $ is the eigenvalue decomposition of $A$, then
  \[
  \prod^k_{i=j+1} (I-\eta_i A)=Q\prod^k_{i=j+1} (I-\eta_i \Lambda) Q^\top 
  \]
  which immediately gives us the  claim 1) since $1-x\leq \exp(-x)$ for $x\in [0,1]$. The first part of 2) comes as an immediate consequence since:
  \[
  \|S_j^k\|\leq \sum_{i=j+1}^k \|Y_j^i\|\leq
  \sum_{i=j+1}^k \exp(-\lambda_A (i-j)\eta_k)\leq \lambda_A^{-1}\eta_k^{-1},
  \]
  while $\eta_k\asymp k^{-\alpha}$. To show the second part of 2), we sum the identity $Y^{i}_j=Y^{i-1}_j-\eta_k A Y^{i-1}_j$ from $i=j+1$ to $i=k+1$, which gives us:
  \[
  Y^{k+1}_{j}=I-\sum_{i=j}^{k}\eta_{k}AY^{i}_j.
  \]
  This can be rewritten as
  \[
  \sum_{i=j}^{k}\eta_{i} Y^{i}_j=A^{-1}-A^{-1}Y_j^{k+1}.
  \]
  Therefore,
  \begin{align*}
  \opnorm{Z^{k}_j}
  &\leq \opnorm{\eta_iS^{n}_i-\sum_{i=j}^{k}\eta_{i} Y^{i}_j}+\opnorm{\sum_{i=j}^{k}\eta_{i} Y^{i}_j-A^{-1}}\\
  &\leq \opnorm{\eta_iS^{n}_i-\sum_{i=j}^{k}\eta_{i} Y^{i}_j}+\opnorm{A^{-1}}\exp\left(-\lambda_A\sum_{m=j+1}^{k+1}\eta_m\right).
  \end{align*}
  In order for us to have claim 2), it suffices to bound the first part. Notice that for $j\leq m\leq i$,
  \begin{align*}
  \eta_{m-1}-\eta_{m}& =\eta ((m-1)^{-\alpha}-m^{-\alpha})=\frac{\eta}{\alpha} \int_{m-1}^{m}x^{-\alpha-1} dx
  \\ & \leq \frac{2^{\alpha+1}\eta}{\alpha}  m^{-\alpha-1}
  \leq 4(\eta\alpha)^{-1}\eta_{m}\eta_{i+1} (i+1)^\alpha j^{-1},
  \end{align*}
  therefore
  \begin{align*}
  \eta_j-\eta_i=\sum_{m=j+1}^{i} (\eta_{m-1}-\eta_m)
  & \leq  \sum_{m=j+1}^i4(\eta\alpha)^{-1}\eta_{m}\eta_{i} (i+1)^\alpha j^{-1} \\
  & \leq 4(\eta\alpha)^{-1}(i+1)^{\alpha}j^{-1}\eta_{i+1}\sum_{m=j+1}^i\eta_m.
  \end{align*}
  Using $S^{k}_j=\sum_{i=j+1}^{k}Y^{i}_j$,
  \begin{align*}
  & \opnorm{\eta_j S^{k}_j-\sum_{i=j}^{k}\eta_{i} Y^{i}_j} \\
  \leq \;&   \sum_{i=j}^{k}(\eta_j-\eta_i)\opnorm{Y_j^i}+\eta_j \\
  \leq \;&  4(\eta\alpha)^{-1}k^\alpha j^{-1}\sum_{i=j}^{k}\eta_{i+1}\bigg(\sum_{m=j+1}^i\eta_m\bigg)\exp\bigg(-\lambda_A\sum_{m=j+1}^i \eta_m \bigg)+\eta j^{-\alpha}
  \end{align*}
  Note that $xe^{-\lambda_A x}\leq  e^{\lambda_ A}ye^{-\lambda_A y}$ when
  $x\leq y\leq x+1$, so
  \begin{align*}
  &\eta_{i+1}\bigg(\sum_{m=j+1}^i\eta_m\bigg)\exp\Bigl(-\lambda_A\sum_{m=j+1}^i \eta_m \Bigr)\\
  = &\;\int^{\sum_{m=j+1}^{i+1}\eta_m}_{\sum_{m=j+1}^{i}\eta_m}\bigg(\sum_{m=j+1}^i\eta_m\bigg)\exp\bigg(-\lambda_A\sum_{m=j+1}^i \eta_m\bigg)\\
  \leq \; & e^{\lambda_A}\int^{\sum_{m=j+1}^{i+1}\eta_m}_{\sum_{m=j+1}^{i}\eta_m}y e^{-\lambda_ Ay}dx.
  \end{align*}
  So
  \[
  \sum_{i=j}^{k}\eta_{i+1}\bigg(\sum_{m=j+1}^i\eta_m\bigg)\exp\Biggl(-\lambda_A\sum_{m=j+1}^i \eta_m\Biggr)
  \leq e^{\lambda_A}\int^\infty_0 x e^{-\lambda_ Ax}dx,
  \]
  which is a constant. Note that $j^{-\alpha}\leq j^{\alpha-1}\leq k^\alpha j^{-1}$, therefore
  \[
  \opnorm{Z^{k}_j}
  \leq \opnorm{\eta_iS^{n}_i-\sum_{i=j}^{k}\eta_{i} Y^{i}_j}+\opnorm{\sum_{i=j}^{k}\eta_{i} Y^{i}_j-A^{-1}}
  \lesssim   k^\alpha j^{-1}+\exp\bigg(-\lambda_A\sum_{i=j+1}^k\eta_i\bigg).
  \]
  
  For claim 3), when $j$ and $k$ are in neighboring batches, the second claim already produces the results, since the $k/j$ is bounded by $3^{\frac{1}{1-\alpha}}$. Otherwise there is a batch between $j,k$, that is for  an $m$ the following holds,
  \[
  j\leq e_{m-1}\leq e_m\leq k.
  \]
  Note that
  \[
  S_j^k=S_j^{e_m}+Y_{j}^{e_m} S_{e_m}^k
  \]
  so
  \[
  Z^k_j=\eta_j S_j^k-A^{-1}=Z_j^{e_m}+Y_{e_m}^k S_{e_m}^k
  \]
  Then because $e_m$ and $j$ are in neighboring batches, by Lemma \ref{lem:approx} claim 3, there is a $c$ such that
  \[
  \opnorm{Z_j^{e_m}}\lesssim K_1j^{\alpha-1}+\exp\left(-\lambda_A \sum_{i=j+1}^{e_m}\eta_i\right)\leq
  j^{\alpha-1}+\exp(-c\lambda_A N).
  \] Because $e_{m}$ is at least one batch away from $j$, so by the first claim of this lemma and Lemma \ref{lem:approx} 4)
  \[
  \opnorm{Y_j^{e_{m}}}\leq \exp\left(-\lambda_A \sum_{i=s_{m}}^{e_m}\eta_i\right)\leq \exp(-\lambda_A c N),
  \]
  while by the second claim,
  \[
  \opnorm{S^{k}_{e_{m}}}\lesssim e_M^\alpha\asymp (NM)^{\frac{\alpha}{1-\alpha}}.
  \]
  Then because $M$ is of polynomial order of $N$, so for $N$ large enough,
  \[
  \|Y_{e_m}^k S_{e_m}^k\|\lesssim  \eta^\alpha (NM)^{\frac{\alpha}{1-\alpha}}\exp(-\lambda_A c N)\leq e_M^{\alpha-1}\leq j^{\alpha-1}
  \]
  Therefore,
  \[
  \opnorm{Z^k_j}\lesssim j^{\alpha-1}+\exp(-c\lambda_A N)\asymp j^{\alpha-1}.
  \]
\end{proof}

\subsection{A linear framework}
This section is devoted to the linear framework illustrated by Lemma \ref{thm:linear}. As we will see, the following two quantities
\[
\widehat{S}:= M^{-1}\sum_{k=1}^Mn_k^{-1}\left(\sum_{i=s_k}^{e_k}\xi_i\right)\left(\sum_{i=s_k}^{e_k}\xi_i\right)^\top 
\]
and $M^{-1}\sum_{k=1}^M n_k\overline{U}_{n_k}\overline{U}_{n_k}^\top $ converge to the asymptotic covariance $A^{-1}S A^{-1}$. The strategy of the proof below is first showing $\widehat{S}$ converges to the right covariance, then $M^{-1}\sum_{k=1}^M n_k\overline{U}_{n_k}\overline{U}_{n_k}^\top $ is not far from the $\widehat{S}$, and eventually reaches Lemma \ref{thm:linear}.

\begin{lem}
  \label{lem:trace}
  Given two symmetric $d\times d$ matrices $S$ and $\Sigma$, suppose $S$ is also positive semidefinite, then
  \[
  \tr(S\Sigma)\leq \tr(S)\opnorm{\Sigma}.
  \]
\end{lem}
\begin{proof}
  Suppose the eigenvalue decomposition of $\Sigma$ is $\Phi D\Phi^\top $. Suppose also that $Q^\top  \Lambda Q$ is the eigenvalue decomposition of $\Phi^\top  S\Phi$. Then
  \[
  \tr(S\Sigma)=\tr(\Phi^\top S\Phi D)=\tr(Q^\top \Lambda Q D)=\tr(\Lambda^{1/2} Q D Q^\top  \Lambda^{1/2}).
  \]
  Let $e_i$ be the i-th standard euclidean basis vector, then
  \[
  e_i^\top \Lambda^{1/2} Q D Q^\top  \Lambda^{1/2}e_i=\lambda_i e_i^\top QDQ^\top e_i\leq \lambda_i \|D\| \|Q^\top e_i\|^2=\lambda_i \|\Sigma\|.
  \]
  Summing over all $i$ yields our claim.

\end{proof}

\begin{lem}
  \label{lem:martingale}
  Under Assumptions \ref{aspt:convexity} and \ref{aspt:martingale}, the following
  \begin{equation}\label{eq:hat_S}
  \widehat{S}:= M^{-1}\sum_{k=1}^Mn_k^{-1}\left(\sum_{i=s_k}^{e_k}\xi_i\right)\left(\sum_{i=s_k}^{e_k}\xi_i\right)^\top 
  \end{equation}
  is a consistent estimator of $S$, namely
  \[
  \mathbb{E}\|\widehat{S}-S\|^2\lesssim  \left(\dcon^3 N^{-\frac{\alpha}{2(1-\alpha)}}M^{-\beta_\alpha}+\dcon^2 M^{-1} \right)
  \]
  with $\beta_\alpha:=1-\left( 1-\frac{ \alpha}{2-2\alpha}\right)^+$.
\end{lem}
\begin{proof}
  Denote by $\widetilde\xi_n=-\nabla f(x^*,\zeta_n)$, and $\widehat\xi_n=\xi_n-\widetilde\xi_n$. Therefore the sequence $\{\xi_n\}$ is decomposed into an \emph{i.i.d.} zero-mean sequence $\{\widetilde\xi_n\}$, and a martingale difference sequence $\{\widehat\xi_n\}$. Further denote by
  \[
  \widetilde{S}:= M^{-1}\sum_{k=1}^Mn_k^{-1}\left(\sum_{i=s_k}^{e_k}\widetilde\xi_i\right)\left(\sum_{i=s_k}^{e_k}\widetilde\xi_i\right)^\top .
    \]
  We have
  \begin{equation}\label{eq:lemd4dec}
  \E\opnorm{\widehat{S}-S}^2\lesssim \E\opnorm{\widetilde{S}-S}^2+\E\opnorm{\widehat{S}-\widetilde{S}}^2.
  \end{equation}

  For the first term on the right hand side of \eqref{eq:lemd4dec}, since
  \[
  \mathbb{E}\opnorm{\widetilde{S}-S}^2 \leq \tr\; \E (\widetilde{S}-S)^2,
  \]
  it suffices to provide a bound of $\tr \;\E (\widetilde{S}-S)^2 $. Notice that
  \[
  \mathbb{E}\widetilde{S}
  = M^{-1}\sum_{k=1}^Mn_k^{-1}\sum_{i=s_k}^{e_k}\mathbb{E}\widetilde\xi_i\widetilde\xi_i^\top =M^{-1}\sum_{k=1}^Mn_k^{-1} \sum_{i=s_k}^{e_k}S=S.
  \]
  Thererfore $\mathbb{E}(\widetilde{S}-S)^2=\mathbb{E} \widetilde{S}^2-S^2.$ For the computation of $\mathbb{E}\widetilde{S}^2$, which is
  \[
  \mathbb{E}\widetilde{S}^2=M^{-2}\sum_{j=1}^M\sum_{k=1}^M n_j^{-1}n_k^{-1}\sum_{i_1,i_2=s_j}^{e_j}\sum_{i_3,i_4=s_k}^{e_k}
  \mathbb{E}\widetilde\xi_{i_1}\widetilde\xi_{i_2}^\top  \widetilde\xi_{i_3}\widetilde\xi_{i_4}^\top.
  \]
  Since $\widetilde\xi_i$ are \emph{i.i.d.}, one can see that the terms are 0 unless indices come in pairs, i.e.  $i_1=i_2, i_3=i_4$ or $i_1=i_3, i_2=i_4$ or   $i_1=i_4, i_2=i_3$. Also notice that the latter two only happens when all indices are in the same batch. So
  \small
  \begin{align*}
  \mathbb{E}\widetilde{S}^2=  \frac{1}{M^2}\Biggl[ & \sum_{j=1}^M\sum_{k=1}^M n_j^{-1}n_k^{-1}\sum_{i_1=s_j}^{e_j}\sum_{i_3=s_k}^{e_k}
  \mathbb{E}\widetilde\xi_{i_1}\widetilde\xi_{i_1}^\top  \widetilde\xi_{i_3}\widetilde\xi_{i_3}^\top 
  \\
  & + \sum_{j=1}^Mn_j^{-2}\sum_{i_1,i_3=s_j,i_1\neq i_3}^{e_j}\left(
  \mathbb{E}\widetilde\xi_{i_1}\widetilde\xi_{i_3}^\top  \widetilde\xi_{i_1}\widetilde\xi_{i_3}^\top   +\mathbb{E}\widetilde\xi_{i_1}\widetilde\xi_{i_3}^\top  \widetilde\xi_{i_3}\widetilde\xi_{i_1}^\top \right)\Biggr],
  \end{align*}
  \normalsize
  We note that $
  \mathbb{E}\widetilde\xi_{i_1}\widetilde\xi_{i_1}^\top  \widetilde\xi_{i_3}\widetilde\xi_{i_3}^\top
  =S^2$ when $i_1\neq i_3$ and $\tr(\mathbb{E}\widetilde\xi_{i_1}\widetilde\xi_{i_1}^\top  \widetilde\xi_{i_3}\widetilde\xi_{i_3}^\top)\lesssim C_d^2+C_d^6 i_1^{-2\alpha}$ when $i_1=i_3$. Therefore
    \small
  \begin{align}
  \nonumber \tr(\mathbb{E}\widetilde{S}^2-S^2)= &M^{-2}\sum_{j=1}^M\sum_{k=1}^M n_j^{-1}n_k^{-1}\sum_{i_1=s_j}^{e_j}\sum_{i_3=s_k}^{e_k}\tr(\mathbb{E}\widetilde\xi_{i_1}\widetilde\xi_{i_1}^\top  \widetilde\xi_{i_3}\widetilde\xi_{i_3}^\top -S^2)\\
  \nonumber&+ M^{-2}\sum_{i_1,i_2=s_j,i_1\neq i_2}^{e_j}\tr(\mathbb{E}\widetilde\xi_{i_1}\widetilde\xi_{i_2}^\top \widetilde\xi_{i_1}\widetilde\xi_{i_2}^\top )+\tr(\mathbb{E}\widetilde\xi_{i_1}\widetilde\xi_{i_2}^\top \widetilde\xi_{i_2}\widetilde\xi_{i_1}^\top )\\
  \nonumber= &M^{-2}\sum_{j=1}^M n_j^{-2}\left(\sum_{i_1=s_j}^{e_j}\mathbb{E}\big\|\widetilde\xi_{i_1}\big\|_2^4+\sum_{i_1,i_2=s_j,i_1\neq i_2}^{e_j}\mathbb{E}\big(\widetilde\xi_{i_2}^\top \widetilde\xi_{i_1}\big)^2+\E\big\|\widetilde\xi_{i_2}\big\|_2^2\big\|\widetilde\xi_{i_1}\big\|_2^2 \right)\\
  \nonumber\lesssim &M^{-2}\sum_{j=1}^M \left(n_j^{-1}(\dcon^2+\dcon^6 s_j^{-2\alpha})+\left(\tr(S)\right)^2 \right).
  \end{align}
        \normalsize
  Applying Lemma \ref{lem:approx} 1) and 5), we have
  \begin{equation}\label{eq:part1}
  \mathbb{E}\opnorm{\widetilde{S}-S}^2\lesssim \tr(\mathbb{E}\widetilde{S}^2-S^2)\lesssim\dcon^6 M^{-2}N^{\frac{-1-2\alpha}{1-\alpha}}+\dcon^2M^{-1}.
  \end{equation}

  For the second term on the right hand side of \eqref{eq:lemd4dec}, we first denote by
  \[
  \widecheck{S}=M^{-1}\sum_{k=1}^Mn_k^{-1}\left(\sum_{i=s_k}^{e_k}\widehat\xi_i\right)\left(\sum_{i=s_k}^{e_k}\widehat\xi_i\right)^\top.
  \]
  We have \small
  \begin{align}
  \nonumber  \E\opnorm{\widehat{S}-\widetilde{S}}^2= &\E \opnorm{M^{-1}\sum_{k=1}^Mn_k^{-1}\left(\left(\sum_{i=s_k}^{e_k}\widehat\xi_i\right)\left(\sum_{i=s_k}^{e_k}\widetilde\xi_i\right)^\top+\left(\sum_{i=s_k}^{e_k}\widetilde\xi_i\right)\left(\sum_{i=s_k}^{e_k}\widehat\xi_i\right)^\top\right)+2\widecheck{S}}^2\\
  \nonumber\lesssim & \E\opnorm{M^{-1}\sum_{k=1}^Mn_k^{-1}\left(\sum_{i=s_k}^{e_k}\widehat\xi_i\right)\left(\sum_{i=s_k}^{e_k}\widetilde\xi_i\right)^\top}^2+\E\opnorm{\widecheck{S}}^2\\
  \label{eq:new1}\lesssim &\sqrt{\E\;\tr(\widecheck{S}^2)}\sqrt{\E\;\tr(\widetilde{S}^2)}+\E\;\tr\big(\widecheck{S}^2\big).
  \end{align}  \normalsize
  It suffices to provide a bound for $\E \tr\big(\widecheck{S}^2\big)$. 
 Let $\overline{\xi}_k=n_k^{-1}\sum_{i=s_k}^{e_k} \widehat\xi_i$. We have 
  \begin{align*}
  \E\;\tr\big(\widecheck{S}^2\big) \; &=M^{-2}\sum_{k=1}^M\sum_{l=1}^M n_kn_l\E \tr(\overline{\xi}_k\overline{\xi}_k^\top \overline{\xi}_l \overline{\xi}_l^\top)\\
  \text{(Cauchy)}&\lesssim M^{-2}\sum_{k=1}^M\sum_{l=1}^M n_kn_l\sqrt{ \E\|\overline{\xi}_k\|_2^4}\sqrt{\E\|\overline{\xi}_l\|_2^4}\\
  \text{\cite[Theorem 2.1]{rio2009moment}}&\lesssim M^{-2}\sum_{k=1}^M\sum_{l=1}^Mn_k^{-1}n_l^{-1}\sum_{i=s_k}^{e_k}\sqrt{\E\big\|\widehat\xi_i\big\|_2^4}\sum_{j=s_l}^{e_l}\sqrt{\E\big\|\widehat\xi_j\big\|_2^4}.
  \end{align*}
   By Lemma \ref{lem:Delta}, we have
  \begin{align*}
  \E\big\|\widehat\xi_i\big\|_2^4
  \lesssim&\E\big\|\nabla F(x_{i-1})-\nabla F(x^*)\big\|_2^4
  +  \E\big\|\nabla f(x_{i-1},\zeta_i)-\nabla f(x^*,\zeta_i)\big\|_2^4\\
  \lesssim & \dcon^4\;\E\|x_{i-1}-x^*\|_2^4\lesssim \dcon^6 (i-1)^{-2\alpha}.
  \end{align*}
  Therefore, using estimations from Lemma \ref{lem:approx} 1) and 5), we have
  \begin{align*}
  \E\;\tr\big(\widecheck{S}^2\big)&\lesssim \dcon^{6} M^{-2}\sum_{k=1}^M\sum_{l=1}^M(s_k-1)^{-\alpha}(s_l-1)^{-\alpha}\\
  &\lesssim \dcon^{6} M^{-2}\sum_{k=1}^M\sum_{l=1}^M(kl)^{-\frac{\alpha}{1-\alpha}}N^{-\frac{2\alpha}{1-\alpha}}\\
  &=  \dcon^{6} M^{-2}\left(\sum_{k=1}^M k^{-\frac{\alpha}{1-\alpha}}\right)^2 N^{-\frac{2\alpha}{1-\alpha}}\lesssim  \dcon^{6} \left(M^{-1+(1-\frac{\alpha}{1-\alpha})^+}N^{-\frac{\alpha}{1-\alpha}}\right)^2.
  \end{align*}
  Plugged into \eqref{eq:new1}, we have
  \[
  \E\opnorm{\widehat{S}-\widetilde{S}}^2\lesssim \dcon^{4} M^{-1+(1-\frac{\alpha}{1-\alpha})^+}N^{-\frac{\alpha}{1-\alpha}}
  +\dcon^2 M^{-1}.
  \]
  Combining with \eqref{eq:part1} we complete the proof.
\end{proof}

\begin{lem}
  \label{thm:linearDelta}
  Under Assumptions \ref{aspt:convexity} and \ref{aspt:martingale}, the following holds for the linear oracle sequence
  \[
  \mathbb{E}\; \Bigl\|M^{-1}\sum_{k=1}^M n_k\overline{U}_{n_k}\overline{U}_{n_k}^\top -A^{-1}SA^{-1} \Bigr\|\lesssim \dcon M^{-\frac{1}{2}}+\dcon N^{-\frac{1}{2}}+\dcon^{\frac32} (MN)^{-\frac{\alpha}{4-4\alpha}}.
  \]
\end{lem}
\begin{proof}
  We define the following matrices sequences through recursion as in Lemma \ref{lem:matrixestimate},
  \[
  Y^j_i=(I-\eta_jA)Y^{j-1}_i,\quad Y^i_i=I_d.
  \]
  Then the recursion of  $U_n$ can be rewritten as
  \begin{align*}
  U_n&=Y^n_{n-1}U_{n-1}-Y^n_n\eta_n\xi_n\\
  &=Y^n_{n-2}U_{n-2}-Y^n_{n-1}\eta_{n-1}\xi_{n-1}-Y^n_n\eta_n\xi_n\\
  &=Y^n_{s-1}U_{s-1}-\sum_{i=s}^n Y^n_{i}\eta_i\xi_i.
  \end{align*}
  Therefore $\overline{U}_{n_k}$ can be written as
  \begin{align*}
  \overline{U}_{n_k}&=n_k^{-1}\left[\sum_{l=s_k}^{e_k}Y^l_{s_k-1}U_{s_k-1}-\sum_{l=s_k}^{e_k}\sum_{i=s_k}^l Y^l_{i}\eta_i\xi_i\right]\\
  &=n_k^{-1}\left[\sum_{l=s_k}^{e_k}Y^l_{s_k-1}U_{s_k-1}-\sum_{i=s_k}^{e_k}\bigg(\sum_{l=i}^{e_k} Y^l_{i}\bigg)\eta_i\xi_i\right]\\
  &=n_k^{-1}\left[\sum_{l=s_k}^{e_k}Y^l_{s_k-1}U_{s_k-1}-\sum_{i=s_k}^{e_k}\bigg(S_i^{e_k}+I\bigg)\eta_i\xi_i\right]\\
  &=n_k^{-1}S_{s_k-1}^{e_k} U_{s_k-1}-n_k^{-1}A^{-1}\sum_{i=s_k}^{e_k}\xi_i-n_k^{-1}\sum_{i=s_k}^{e_k} Z_i^{e_k}\xi_i-n_k^{-1}\sum_{i=s_k}^{e_k}\eta_i\xi_i
  \end{align*}
  Here we denoted
  \[
  S_{m}^{e_k}=\sum_{l=m+1}^{e_k} Y_{m}^l=\sum_{l=m}^{e_k} Y_{m}^l-I,\quad Z_m^{e_k}=\eta_m S_{m}^{e_k}-A^{-1},
  \]
  as in Lemma \ref{lem:matrixestimate}. Denote
  \[
  A_k:=-A^{-1}\sum_{i=s_k}^{e_k}\xi_i,\quad B_k:=S_{s_k-1}^{e_k}U_{s_k-1}-\sum_{i=s_k}^{e_k} Z_i^{e_k}\xi_i-\sum_{i=s_k}^{e_k}\eta_i\xi_i,
  \]
  then
  \begin{align}
  \label{tmp:decomp}
  M^{-1}\sum_{k=1}^M n_k \overline{U}_{n_k}\overline{U}_{n_k}^\top =&M^{-1}\sum_{k=1}^Mn_k^{-1}A_k A_k^\top \nonumber\\
  &+M^{-1}\sum_{k=1}^M n_k^{-1}\left[A_kB_k^\top +B_k A_k^\top + B_kB_k^\top \right].
  \end{align}
  Denote
  \[
  \widehat{S}:= M^{-1}\sum_{k=1}^Mn_k^{-1}\left(\sum_{i=s_k}^{e_k}\xi_i\right)\left(\sum_{i=s_k}^{e_k}\xi_i\right)^\top .
  \]
  By Lemma \ref{lem:martingale}, we have for sufficiently large $N$
  \begin{align*}
  \mathbb{E}\opnorm{M^{-1}\sum_{k=1}^Mn_k^{-1}A_k A_k^\top -A^{-1}SA^{-1}}^2
  = & \mathbb{E}\opnorm{A^{-1}(\widehat{S}-S)A^{-1}}^2 \\ \lesssim &  \left[\dcon^3 N^{-\frac{\alpha}{2(1-\alpha)}}M^{-\beta_\alpha}+\dcon^2 M^{-1} \right].
  \end{align*}
  By Cauchy-Schwartz inequlaity,
  \[
  \mathbb{E}\opnorm{M^{-1}\sum_{k=1}^Mn_k^{-1}A_k A_k^\top -A^{-1}SA^{-1}}
  \lesssim \dcon^{3/2}N^{-\frac{\alpha}{4-4\alpha}}M^{-\frac{\beta_\alpha}{2}}+\dcon M^{-\frac{1}{2}}
  \]
  Next, note that
  \[
  \mathbb{E}\|B_k\|_2^2=\mathbb{E}\opnorm{B_k B_k^\top }\leq \mathbb{E}\tr[B_k B_k^\top  ].
  \]
  Using that $\xi_n$ are martingale differences
  \begin{align*}
  & \mathbb{E}B_k B_k^\top \\
  =\; &\mathbb{E}\left[S_{s_k-1}^{e_k}U_{s_k-1}-\sum_{i=s_k}^{e_k} Z_i^{e_k}\xi_i-\sum_{i=s_k}^{e_k}\eta_i\xi_i\right]\left[S_{s_k-1}^{e_k}U_{s_k-1}-\sum_{i=s_k}^{e_k} Z_i^{e_k}\xi_i-\sum_{i=s_k}^{e_k}\eta_i\xi_i\right]^\top ,\\
  =\; & S_{s_k-1}^{e_k}\mathbb{E}\left[U_{s_k-1}U_{s_k-1}^\top \right](S_{s_k-1}^{e_k})^\top +\sum_{i=s_k}^{e_k} (Z_i^{e_k}+\eta_iI)\mathbb{E}(\xi_i\xi_i^\top ) (Z_i^{e_k}+\eta_iI)^\top ,
  \end{align*}
  From \ref{lem:matrixestimate} claim 3,
  \begin{align*}
  \tr(\mathbb{E}B_kB_k^\top )
  \leq \opnorm{S^{e_k}_{s_k-1}}^2\mathbb{E}\|U_{s_k-1}\|_2^2+\sum_{i=s_k}^{e_k}\opnorm{Z_i^{e_k}+\eta_iI}^2\mathbb{E}\|\xi_i\|_2^2,
  \end{align*}
  Note that from Lemma  \ref{lem:matrixestimate} claim 3
  \[
  \opnorm{Z_i^{e_k}+\eta_iI}\leq \opnorm{Z_i^{e_k}}+\eta i^{-\alpha}\lesssim i^{\alpha-1}+\exp\left(-c\lambda_A (e^k-i)\eta_i \right)+i^{-\alpha}.
  \]
  Using Young's inequality and the estimate from Lemma \ref{lem:approx}, and $i^{-\alpha}\leq i^{\alpha-1}$ with $\alpha\in[1/2,1)$,
  \[
  \opnorm{Z_i^{e_k}+\eta_iI}^2\lesssim {s_k}^{2\alpha-2}+\exp(-2c\lambda_A(e_k-i)\eta_{e_k}).
  \]
  Thus from the bounds of $\|U_n\|_2^2$ in Lemma \ref{lem:lindelta}, and that $\mathbb{E}\|\xi_i\|_2^2\lesssim \dcon$,
  \[
  \frac1{\dcon}\tr(\mathbb{E}B_kB_k^\top ) \lesssim s_k^\alpha+\sum_{i=s_k}^{e_k}\opnorm{Z_i^{e_k}+\eta_iI}^2\lesssim s_k^\alpha+n_k{s_k}^{2\alpha-2}+\sum_{i=s_k}^{e_k}\exp(-2c\lambda_A(e_k-i)\eta_{e_k})
  \]
  
  Note that
  \[
  \sum_{i=s_k}^{e_k}\exp(-2c\lambda_A(e_k-i)\eta_{e_k})\leq\sum_{i=0}^\infty\exp(-2c\lambda_A i\eta_{e_k})=[2c\lambda_A\eta_{e_k}]^{-1}.
  \]
  Then based on Lemma \ref{lem:approx},
  \[
  n_ks_k^{2\alpha-2}+\eta_{e_k}^{-1}
  \lesssim    N^{-\frac{1-2\alpha}{1-\alpha}}k^{-\frac{2-\alpha}{1-\alpha}}+ N^{-\frac{\alpha}{1-\alpha}}k^{-\frac{\alpha}{1-\alpha}}\lesssim s_k^\alpha.
  \]
  Using the bound that $e_k^\alpha (s_k-1)^{-\alpha}\leq 4$,
  \[
  M^{-1}\sum_{k=1}^M n_k^{-1}\mathbb{E}\|B_k\|_2^2\lesssim \dcon M^{-1}\sum_{k=1}^M n_k^{-1} s_k^\alpha.
  \]
  By Lemma \ref{lem:approx},
  \[
  n_k^{-1}s_k^\alpha\lesssim  N^{-\frac{1}{1-\alpha}}k^{-\frac{\alpha}{1-\alpha}} N^{\frac{\alpha}{1-\alpha}}k^{\frac{\alpha}{1-\alpha}}=N^{-1}.
  \]
  So
  \[
  M^{-1}\sum_{k=1}^M n_k^{-1}\mathbb{E}\|B_k\|_2^2\lesssim \dcon N^{-1}.
  \]
  On the other hand, by Lemma \ref{lem:Delta},
  \begin{align*}
  n_k^{-1}\mathbb{E}\|A_k\|_2^2=& n_k^{-1}\tr\mathbb{E}A_k A_k^\top  \\
  \leq & n_k^{-1}\|A^{-1}\|^2 \left[n_k\tr(S)+2 \sum_{i=s_k}^{e_k}\mathbb{E}(\Sigma_1\|\Delta_i\|_2+\Sigma_2\|\Delta_i\|_2^2)\right]
  \lesssim \dcon.
  \end{align*}
  So
  \[
  n_k^{-1}\mathbb{E} \|A_k\|_2\|B_k\|_2\lesssim \sqrt{n_k^{-1}\mathbb{E}\|A_k\|^2_2}\sqrt{n_k^{-1}\mathbb{E}\|B_k\|_2^2}\lesssim \dcon N^{-\frac{1}{2}}.
  \]
  Therefore, in conclusion,
  \begin{align*}
  &\mathbb{E} \opnorm{M^{-1}\sum_{k=1}^M \overline{U}_{n_k}\overline{U}_{n_k}^\top -A^{-1}SA^{-1}}\\
  \leq \; & \mathbb{E} \opnorm{M^{-1}\sum_{k=1}^M n_k^{-1}A_k A_k^\top -A^{-1}SA^{-1}} \\
  & +2M^{-1}\sum_{k=1}^M n_k^{-1}(\mathbb{E} \|A_k\|_2\|B_k\|_2+\mathbb{E}\|B_k\|_2^2)\\
  \lesssim\;&\dcon^{3/2}N^{-\frac{\alpha}{4-4\alpha}}M^{-\frac{\beta_\alpha}{2}}+\dcon M^{-\frac{1}{2}}+\dcon N^{-\frac{1}{2}}.
  \end{align*}
  
  Finally, if $\frac12\leq \frac{ \alpha}{4-4\alpha}$, then
  \[
  \dcon^{3/2}N^{-\frac{\alpha}{4-4\alpha}}M^{-\frac{\beta_\alpha}{2}}\lesssim  \dcon^{3/2}N^{-\frac{\alpha}{4-4\alpha}} M^{-\frac12}\lesssim \dcon M^{-\frac12},
  \]
  so it is an insignificant term. When $\frac12\geq \frac{ \alpha}{4-4\alpha}$, then the first term is of order $(MN)^{- \frac{ \alpha}{4-4\alpha}}$.

\end{proof}

\subsection{Proof of Lemma \ref{thm:linear}}
\label{suppsec:thmlinear}
\begin{replemma}{thm:linear}
  Under Assumptions \ref{aspt:convexity} and \ref{aspt:martingale}, when $d$ is fixed and {the step size is chosen to be $\eta_i = \eta i^{-\alpha}$ with $\alpha\in (\frac12,1)$,} the batch-means estimator based on the sequence $U_n$ with any bounded $U_0$ satisfies the following inequality for sufficiently large $N$ and $M,$
  \begin{align*}
  & \mathbb{E}\;   \opnorm{ M^{-1}\sum_{k=1}^M n_k(\overline{U}_{n_k}-\overline{U}_M)(\overline{U}_{n_k}-\overline{U}_M)^\top -A^{-1}SA^{-1}}\\
  &\lesssim  \dcon M^{-\frac{1}{2}}+\dcon N^{-\frac{1}{2}}+\dcon^{\frac32} (NM)^{-\frac{\alpha}{4-4\alpha}}.
  \end{align*}
\end{replemma}
\begin{proof}[Proof of Lemma \ref{thm:linear}]
  Using the identity $\sum_{k=1}^Mn_k\overline{U}_M=\sum_{k=1}^Mn_k\overline{U}_{n_k}$,
  we can easily find
  \begin{align*}
  &M^{-1}\sum_{k=1}^M n_k(\overline{U}_{n_k}-\overline{U}_M)(\overline{U}_{n_k}-\overline{U}_M)^\top \\
  =&M^{-1}\sum_{k=1}^M n_k\overline{U}_{n_k}\overline{U}^\top _{n_k}-M^{-1}\sum_{k=1}^M n_k\overline{U}_M\overline{U}_M^\top .
  \end{align*}
  Lemma \ref{thm:linearDelta} already shows that
  \[
  \mathbb{E} \left\|M^{-1}\sum_{k=1}^M n_k\overline{U}_{n_k}\overline{U}_{n_k}^\top -A^{-1}SA^{-1}\right\|\lesssim \dcon M^{-\frac{1}{2}}+\dcon N^{-\frac{1}{2}}+\dcon^{\frac32} (NM)^{-\frac{\alpha}{4-4\alpha}}
  \]
  holds, so in order to complete our proof, it suffices to show
  $\mathbb{E} \|M^{-1}S_M \overline{U}_M\overline{U}_M^\top \|$ is of the  order $\dcon M^{-1}$. Note that
  \[
  \overline{U}_{M}
  =S_M^{-1}S_{s_1-1}^{e_M} U_{s_1-1}-S_M^{-1}\sum_{i=s_1}^{e_M}\Big(S_{i}^{e_M}+I\Big)\eta_i\xi_i
  \]
  so using the fact that $\xi_i$ are martingale differences
  \begin{align*}
  & \mathbb{E}\|\overline{U}_{M}\|_2^2 \\
  \leq \; &
  S_M^{-2}\|S_{s_1-1}^{e_M}\|^2\mathbb{E} \|U_{s_1-1}\|_2^2
  +S_M^{-2}\sum_{i=s_1}^{e_M}\eta_i^2\|S_{i}^{e_M}+I\|^2
  \left[\tr(S)+\mathbb{E}(\Sigma_1\|\Delta_{i}\|_2+\Sigma_2 \|\Delta_{i}\|_2^2)\right].
  \end{align*}
  Lemma \ref{lem:matrixestimate} claim 3  provides the following bound using the approximations in Lemma \ref{lem:approx}
  \[
  M^{-1}S_M^{-1}\|S_{s_1-1}^{e_M}\|^2\mathbb{E} \|U_{s_1-1}\|_2^2\lesssim
  M^{-1}S_M^{-1}s_1^\alpha\dcon
  \lesssim \dcon M^{-1}N^{-1}.
  \]
  On the other hand, Lemma \ref{lem:matrixestimate} 3) also shows that
  $\eta_i^2 \|S_{i}^{e_M}+I\|^2$ is uniformly bounded,
  and by Lemma \ref{lem:Delta},
  \[
  [\tr(S)+\mathbb{E}(\Sigma_1\|\Delta_{i}\|_2+\Sigma_2 \|\Delta_{i}\|_2^2)]\lesssim \dcon.
  \]
  Therefore
  \[
  M^{-1}S_M^{-1}\sum_{i=s_1}^{e_M}\eta_i^2\|S_{i}^{e_M}+I\|^2\left[\tr(S)+\mathbb{E}(\Sigma_1\|\Delta_{i}\|_2+\Sigma_2 \|\Delta_{i}\|_2^2)\right]\lesssim M^{-1}\dcon.
  \]
  Collecting the terms with lower order produces the claimed result.
\end{proof}

\subsection{Consistency proof of Batch-Means Estimator}
\label{sec:supp_batch}
\begin{reptheorem}{thm:general}
  Under Assumptions \ref{aspt:convexity} and \ref{aspt:martingale}, when $d$ is fixed {and the step size is chosen to be $\eta_i = \eta i^{-\alpha}$ with $\alpha\in (\frac12,1)$,} the batch-means estimator initialized by any bounded $x_0$ is a consistent estimator. In particular, for sufficiently large $N$ and $M$, we have
  \begin{align*}
  &\mathbb{E}  \opnorm{M^{-1}\sum_{k=1}^Mn_k(\overline{X}_{n_k}-\overline{X}_{M})(\overline{X}_{n_k}-\overline{X}_{M})^\top -A^{-1}SA^{-1}} \\
  &\lesssim \dcon M^{-\frac{1}{2}}+\dcon N^{-\frac{1}{2}}+\dcon^{\frac32} (MN)^{-\frac{\alpha}{4-4\alpha}}+\dcon^2M^{-1}+\dcon^3M^{-1}N^{\frac{1-2\alpha}{1-\alpha}}.
  \end{align*}
  
\end{reptheorem}
\begin{proof}[Proof of Theorem \ref{thm:general}]
  Recall the linear oracle sequence defined by \eqref{eqn:linear}: $U_n=\left(I-\eta_n A \right) U_{n-1} + \eta_n \xi_n$, $U_0=\Delta_0$.
  From Lemma \ref{thm:linear}
  \begin{align*}
  &\mathbb{E}\; \opnorm{M^{-1}\sum_{k=1}^Mn_k(\overline{U}_{n_k}-\overline{U}_{M})(\overline{U}_{n_k}-\overline{U}_{M})^\top -A^{-1}SA^{-1}} \\
  &\lesssim  \dcon M^{-\frac{1}{2}}+\dcon N^{-\frac{1}{2}}+\dcon^{\frac32} (MN)^{-\frac{\alpha}{4-4\alpha}}.
  \end{align*}
  Let $\delta_n:=\Delta_n-U_n$, which follows a recursion:
  {\begin{align*}
    \delta_n&=\Delta_{n-1}-U_{n-1}-\eta_n\nabla F(x_{n-1})-\eta_n AU_{n-1}\\
    &=(I-\eta_n A) \delta_{n-1}+\eta_n(A\Delta_{n-1}-\nabla F(x_{n-1})).
    \end{align*}}
  Define the  mean differences for batches and overall:
  {\[
    \overline{\delta}_{n_k}:=n_k^{-1}\sum_{i=s_k}^{e_k}(\Delta_i-U_i),\quad \overline{\delta}_M:=S_M^{-1}\sum_{i=s_1}^{e_M}(\Delta_i-U_i).
    \]}
  Using $\sum_{k=1}^Mn_k\overline{U}_M=\sum_{k=1}^Mn_k\overline{U}_{n_k}$ and $\sum_{k=1}^Mn_k\overline{X}_M=\sum_{k=1}^Mn_k\overline{X}_{n_k}$, we have the following decomposition:
  \allowdisplaybreaks
  { \begin{align*}
    &\frac{1}{M}\sum_{k=1}^Mn_k(\overline{X}_{n_k}-\overline{X}_{M})(\overline{X}_{n_k}-\overline{X}_{M})^\top =\frac{1}{M}\sum_{k=1}^M n_k\overline{X}_{n_k}\overline{X}^\top _{n_k}-\frac{1}{M}\sum_{k=1}^M n_k\overline{X}_M\overline{X}_M^\top \\
    =&\frac{1}{M}\sum_{k=1}^M n_k\overline{\Delta}_{n_k}\overline{\Delta}^\top _{n_k}
    -\frac{1}{M}\sum_{k=1}^M n_kx^*x^{*T}
    +\frac{1}{M}\sum_{k=1}^M n_k\overline{\Delta}_{M}x^{*T}\\
    &+\frac{1}{M}\sum_{k=1}^M n_kx^*\overline{\Delta}_{M}^\top -\frac{1}{M}\sum_{k=1}^M n_k(x^*+\overline{\Delta}_M)(x^*+\overline{\Delta}_M)^\top \\
    =&\frac{1}{M}\sum_{k=1}^M n_k \overline{\Delta}_{n_k}\overline{\Delta}_{n_k}^\top -\frac{1}{M}\sum_{k=1}^M n_k \overline{\Delta}_M\overline{\Delta}_M^\top \\
    =&\frac{1}{M}\sum_{k=1}^M n_k (\overline{U}_{n_k}+\overline{\delta}_{n_k})(\overline{U}_{n_k}+\overline{\delta}_{n_k})^\top -\frac{1}{M}\sum_{k=1}^M n_k (\overline{U}_M+\overline{\delta}_M)(\overline{U}_M+\overline{{\delta}}_M)^\top \\
    =&\frac{1}{M}\sum_{k=1}^M n_k (\overline{U}_{n_k}-\overline{U}_{M})(\overline{U}_{n_k}-\overline{U}_{M})^\top 
    +\frac{1}{M}\sum_{k=1}^M n_k (\overline{U}_{n_k}-\overline{U}_{M})(\overline{\delta}_{n_k}-\overline{\delta}_{M})^\top \\
    &+\frac{1}{M}\sum_{k=1}^M n_k (\overline{\delta}_{n_k}-\overline{\delta}_{M})(\overline{U}_{n_k}-\overline{U}_{M})^\top 
    +\frac{1}{M}\sum_{k=1}^M n_k (\overline{\delta}_{n_k}-\overline{\delta}_{M})(\overline{\delta}_{n_k}-\overline{\delta}_{M})^\top .
    \end{align*}
  }
  Then by Young's inequality and $\|C\|\leq \tr(C)$ for any PSD matrix $C$,
  {\begin{align*}
    &\mathbb{E}\left\|\frac{1}{M}\sum_{k=1}^Mn_k(\overline{X}_{n_k}-\overline{X}_{M})(\overline{X}_{n_k}-\overline{X}_{M})^\top -A^{-1}SA^{-1}\right\|\\
    \leq & \mathbb{E} \left\|\frac{1}{M}\sum_{k=1}^Mn_k(\overline{U}_{n_k}-\overline{U}_{M})(\overline{U}_{n_k}-\overline{U}_{M})^\top -A^{-1}SA^{-1}\right\| \\
    &+ \frac{1}{M}\sum_{k=1}^M n_k  \mathbb{E}\tr\left[(\overline{\delta}_{n_k}-\overline{\delta}_{M})(\overline{\delta}_{n_k}-\overline{\delta}_{M})^\top \right] \\
    &
    +\frac{2}{M}\sqrt{\sum_{k=1}^Mn_k\mathbb{E}\tr[(\overline{U}_{n_k}-\overline{U}_{M})(\overline{U}_{n_k}-\overline{U}_{M})^\top ]\sum_{k=1}^M n_k \mathbb{E}\tr\left[(\overline{\delta}_{n_k}-\overline{\delta}_{M})(\overline{\delta}_{n_k}-\overline{\delta}_{M})^\top \right]}
    \end{align*}
  }
  Note that
  \begin{align*}
  &\sum_{k=1}^Mn_k\tr(\overline{\delta}_{n_k}-\overline{\delta}_{M})(\overline{\delta}_{n_k}-\overline{\delta}_{M})^\top \\
  =&\sum_{k=1}^Mn_k\left[\tr(\overline{\delta}_{n_k}\overline{\delta}^\top _{n_k})
  -\tr(\overline{\delta}_{M}\overline{\delta}^\top _{M})\right]\leq \sum_{k=1}^Mn_k\tr(\overline{\delta}_{n_k}\overline{\delta}^\top _{n_k}).
  \end{align*}
  Therefore, a proper bound  for $\frac{1}{M}\sum_{k=1}^M n_k\mathbb{E}\|\overline{\delta}_{n_k}\|_2^2$ will conclude our proof.
  
  Using the definition of $Y_j^k$ and $S_j^k$ in Lemma \ref{lem:matrixestimate}, we find that
  \begin{align*}
  \delta_n&=(I-\eta_n A) \delta_{n-1}+\eta_n(A\Delta_{n-1}-\nabla F(x_{n-1}))\\
  &=\left[\prod\limits_{i=s_k}^{n}(I-\eta_iA)\right]\delta_{s_k-1}+\sum_{i=s_k}^{n}\left[\prod_{j=i+1}^n(I-\eta_jA)\right]\eta_i(A\Delta_{i-1}-\nabla F(x_{i-1}))\\
  &=Y_{s_k-1}^n\delta_{s_k-1}+\sum_{i=s_k}^{n}Y_i^n\eta_i(A\Delta_{i-1}-\nabla F(x_{i-1}))\\
  \overline{\delta}_{n_k}&=\frac{1}{n_k}\sum_{j=s_k}^{e_k}\delta_j=\frac{1}{n_k}\sum_{j=s_k}^{e_k}Y_{s_k-1}^j\delta_{s_k-1}+\frac{1}{n_k}\sum_{j=s_k}^{e_k}\sum_{i=s_k}^{j}Y_i^j\eta_i(A\Delta_{i-1}-\nabla F(x_{i-1}))\\
  &=\frac{1}{n_k}S^{e_k}_{s_k-1}\delta_{s_k-1}+\frac{1}{n_k}\sum_{i=s_k}^{e_k}\sum_{j=i}^{e_k}Y_i^{e_k}\eta_i(A\Delta_{i-1}-\nabla F(x_{i-1}))\\
  &=\frac{1}{n_k} S^{e_k}_{s_k-1}\delta_{s_k-1}+\frac{1}{n_k}\sum_{n=s_k}^{e_k} (I+S_{n}^{e_k}) \eta_n (A\Delta_{n-1}-\nabla F(x_{n-1})).
  \end{align*}
  So by Young's inequality and Cauchy-Schwartz inequality
  {\begin{align*}
    \mathbb{E}\|\overline{\delta}_{n_k}\|_2^2\leq&  2n^{-2}_k  \|S^{e_k}_{s_k-1}\|^2\mathbb{E} \|\delta_{s_k-1}\|_2^2\\
    &+
    2n_k^{-2}\mathbb{E}\left(\sum_{n=s_k}^{e_k} \|\eta_n (I+S_{n}^{e_k})\| \|(A\Delta_{n-1}-\nabla F(x_{n-1})\|\right)^2\\
    \leq&  2n^{-2}_k  \|S^{e_k}_{s_k-1}\|^2\mathbb{E} \|\delta_{s_k-1}\|_2^2\\
    &+
    2n_k^{-2}\sum_{n=s_k}^{e_k} \|\eta_n (S_{n}^{e_k}+I)\|^2\mathbb{E} \sum_{n=s_k}^{e_k} \|A\Delta_{n-1}-\nabla F(x_{n-1})\|_2^2
    \end{align*}}
  Apply Lemma \ref{lem:Delta} and Lemma \ref{lem:matrixestimate},
  \[
  \|S^{e_k}_{s_k-1}\|^2\lesssim s_k^{2\alpha},\quad
  \sum_{n=s_k}^{e_k} \|\eta_n S_{n}^{e_k}\|^2\lesssim n_k.
  \]
  \[
  \mathbb{E} \sum_{n=s_k}^{e_k} \|A\Delta_{n-1}-\nabla F(x_{n-1})\|_2^2
  \lesssim \sum_{n=s_k}^{e_k}\mathbb{E}L_F^2\|\Delta_{n-1}\|_2^2\lesssim  n_k s_k^{-\alpha} \dcon^3.
  \]
  Notice that Lemma \ref{lem:Delta} can be applied to  both $U_n$ and $\Delta_n$, so
  \[
  \mathbb{E} \|\delta_{s_k-1}\|_2^2\leq 2\mathbb{E}\|\Delta_{s_k-1}\|_2^2+2\mathbb{E}\|U_{s_k-1}\|_2^2\lesssim s_k^{-\alpha}\dcon.
  \]
  Furthermore, recall the batch size $e_k=((k+1)N)^{\frac{1}{1-\alpha}}$,
  \[
  n_k=\left[(k+1)^{\frac{1}{1-\alpha}}-k^{\frac{1}{1-\alpha}}\right]N^{\frac{1}{1-\alpha}}=\frac{N^{\frac{1}{1-\alpha}} }{1-\alpha}
  \int^{k+1}_k x^{\frac{\alpha}{1-\alpha}} dx\asymp k^{\frac{\alpha}{1-\alpha}}N^{\frac{1}{1-\alpha}} ,
  \]
  \[
  n_k\mathbb{E}\|\overline{\delta}_{n_k}\|_2^2\lesssim \dcon n_k^{-1}s_k^\alpha+n_k s_k^{-2\alpha}\dcon^3\lesssim \dcon N^{-1}+\dcon^3k^{-\frac{\alpha}{1-\alpha}}N^{\frac{1-2\alpha}{1-\alpha}}.
  \]
  For the sum of geometric sequence in the last term, the following holds
  \[
  \frac{1}{M}\sum_{k=1}^Mk^{-\frac{\alpha}{1-\alpha}}\lesssim  \frac{1}{M}\int^{M+1}_1 x^{-\frac{\alpha}{1-\alpha}}dx\lesssim M^{(1-\frac{\alpha}{1-\alpha})^+-1},
  \]
  which implies
  \[
  \frac{1}{M}\sum_{k=1}^Mn_k\mathbb{E}\|\overline{\delta}_{n_k}\|_2^2\lesssim \dcon N^{-1}+\dcon^3M^{(\frac{1-2\alpha}{1-\alpha})^+-1}N^{\frac{1-2\alpha}{1-\alpha}}.
  \]
  But since $\alpha> 1/2$, so
  \[
  M^{(\frac{1-2\alpha}{1-\alpha})^+-1}N^{\frac{1-2\alpha}{1-\alpha}}=M^{-1}N^{\frac{1-2\alpha}{1-\alpha}}\leq M^{-1}.
  \]
  In conclusion,
  \begin{align*}
  &\mathbb{E}\left\|M^{-1}\sum_{k=1}^Mn_k(\overline{X}_{n_k}-\overline{X}_{M})(\overline{X}_{n_k}-\overline{X}_{M})^\top -A^{-1}SA^{-1}\right\|\\
  &\quad \lesssim \dcon M^{-\frac{1}{2}}+\dcon N^{-\frac{1}{2}}+\dcon^{\frac32} (MN)^{-\frac{\alpha}{4-4\alpha}}+\dcon^2M^{-1}+\dcon^3M^{-1}N^{\frac{1-2\alpha}{1-\alpha}}.
  \end{align*}
\end{proof}

\subsection{Proof of Corollaries \ref{cor:directplugin} and  \ref{cor:pluginterval} }

\begin{repcorollary}{cor:directplugin}
  Under Assumptions \ref{aspt:convexity}, \ref{aspt:martingale} and \ref{aspt:third}, as $n\rightarrow \infty$,
  \[
  \opnorm{A_n^{-1} S_n A_n ^{-1} -A^{-1} S A^{-1}}=O_p\big(\|S\|(\dcon^2 n^{-\frac{\alpha}{2}}+\dcon^3 n^{-\alpha})\big).
  \]
\end{repcorollary}
\begin{proof}
  Let
  $
  \mathcal{Q}:=\{\lambda_{\min}(A_n)>\delta\}=\{A_n=\widetilde{A}_n\},
  $ by Markov inequality and Lemma \ref{lem:A_con},
  \[
  \mathbb{P}(\lambda_{\min}(A_n)\leq \tfrac\mu2)\leq \mathbb{P}(\|A_n-A\|{\geq} \tfrac\mu2)\leq \frac{2\E \|A_n-A\| }{\mu}\lesssim \dcon^2 n^{-\alpha/2}.
  \]
  then $\mathbb{P}(\mathcal{Q}{^c})\leq c_1\dcon^2 n^{-\frac\alpha2}$ for some $c_1$. Then by Theorem \ref{thm:plugin}, for a constant $c_2$,
  \[
  \mathbb{E}\unit_{\mathcal{Q}}\opnorm{A_n^{-1} S_n A_n ^{-1} -A^{-1} S A^{-1}}\leq c_2 \|S\|(\dcon^2 n^{-\frac{\alpha}{2}}+\dcon^3 n^{-\alpha}).
  \]
  By Markov inequality again, for any $c_3>0$,
  \[
  \mathbb{P}\big(\opnorm{A_n^{-1} S_n A_n ^{-1} -A^{-1} S A^{-1}}>c_3c_2\|S\|(\dcon^2 n^{-\frac{\alpha}{2}}+\dcon^3 n^{-\alpha}))\big)\leq c_1C_d^2n^{-\frac{\alpha}{2}}+c_3^{-1}.
  \]
  Therefore we can choose a large $c_3$, so when $n\rightarrow \infty$,
  \[
  \opnorm{A_n^{-1} S_n A_n ^{-1} -A^{-1} S A^{-1}}=O_p(\|S\|(\dcon^2 n^{-\frac{\alpha}{2}}+\dcon^3 n^{-\alpha})).
  \]
\end{proof}

\label{sec:supp_pluginterval}
\begin{repcorollary}{cor:pluginterval}
  Under the assumptions of Theorem \ref{thm:plugin}, {if the step size is chosen to be $\eta_i = \eta i^{-\alpha}$ with $\alpha\in (\frac12,1)$,} when $d$ is fixed and $n\to \infty$,
  \[
  \Pr\left(\bar{x}_{n,j}-z_{q/2}\hat{\sigma}^P_{n,j}/\sqrt{n}\leq x^*_j\leq \bar{x}_{n,j}+z_{q/2}\hat{\sigma}^P_{n,j}/\sqrt{n}\right)\to 1-q.
  \]
\end{repcorollary}
\begin{proof}[Proof of Corollary \ref{cor:pluginterval}]
  Denote $\sigma_j=\left(A^{-1} S A^{-1}\right)^{1/2}_{jj}$.  Fix any $\epsilon>0$, note that
  \begin{align*}
  &\Pr\left(\bar{x}_{n,j}-z_{q/2}\hat{\sigma}^P_{n,j}/\sqrt{n}\leq x^*_j\leq \bar{x}_{n,j}+z_{q/2}\hat{\sigma}^P_{n,j}/\sqrt{n}\right)\\
  &\leq \Pr \left(\bar{x}_{n,j}-z_{q/2-\epsilon}\sigma_{j}/\sqrt{n} \leq x^*_j\leq \bar{x}_{n,j}+z_{q/2-\epsilon}\sigma_{j}/\sqrt{n}\right)+\Pr(z_{q/2-\epsilon}\sigma_{j}\leq z_{q/2}\hat{\sigma}^P_{n,j} ).
  \end{align*}
  The first term by theorem 2 of \cite{PJ92} converges to $1-q+2\epsilon$ when $n\to \infty$, the second term by Theorem \ref{thm:plugin} and Markov inequality converges to $0$ when $n\to \infty$. Therefore we have
  \[
  \Pr\left(\bar{x}_{n,j}-z_{q/2}\hat{\sigma}^P_{n,j}/\sqrt{n}\leq x^*_j\leq \bar{x}_{n,j}+z_{q/2}\hat{\sigma}^P_{n,j}/\sqrt{n}\right)\leq 1-q+2\epsilon,\quad n\to \infty.
  \]
  Likewise we can have the opposite direction. Since this holds for any $\epsilon>0$, we have our result.
\end{proof}

\section{Supplement to the High Dimensional Section (Section \ref{sec:high_dim})}
\label{sec:supp_highdim}
\subsection{Proof of Proposition \ref{prop:L1_conv}}
\label{subsec:L1_conv}
\begin{repproposition}{prop:L1_conv}
  Under Assumption \ref{aspt:cov} and using the same algorithm parameters as Proposition 1 in \cite{agarwal2012stochastic}, there exists a constant $c_0$, such that $\widehat{x}_n$ in Algorithm \ref{algo:high_CI} satisfies
  \begin{eqnarray}\label{eq:x_L1_conv}
  \|\widehat{x}_n-x^*\|_1 \leq c_0 s  \sqrt{\frac{\log d}{n}}
  \end{eqnarray}
  {uniformly in $x^*\in\mathcal{B}(s)$}  with high probability. Further, for each $j= \{ 1,\ldots, d\}$,
  \begin{eqnarray}\label{eq:gamma_L1_conv}
  \|\widehat{\gamma}^j + \Omega_{j,j}^{-1} (\Omega_{j,-j})^\top  \|_1 {\leq c_0} s_j \sqrt{\frac{\log d}{n}}
  \end{eqnarray}
  holds with high probability.
\end{repproposition}
\begin{proof}
  {
    where $R_i = \frac{R_1}{2^{i/2}}$, $T_i \geq c_1 \frac{s^2}{R_i^2} \log d$ for a universal constant $c_1$, $A_\psi\leq e\log d$, $\omega_i^2=\omega^2+24 \log i$,  and $K^2$ is the variance proxy of $a$ in Assumption \ref{cor:uniform}. Therefore, let $c_0=(c_1\log d)^{-\frac14}(2\mu)^{\frac12}\left[e(\log d) (L_F^2+K^2)+\omega_i^2K^2\right]^{\frac14}$. We have
    \[
    \lambda_{K_n}\leq \frac{c_0R_{K_n}}{\sqrt{2}s}\leq c_0\frac{R_1}{2^{K_n/2}s}.
    \]
    By Eq. (32) in \cite{agarwal2012stochastic}, the total number of iterations $ n\leq s^2 (\log d)  2^{K_n}$.
    Therefore, $\|\widehat{x}_n-x^*\|_1\leq \frac{9c_0R_1}{\mu} \cdot s\sqrt{\frac{\log n}{d}}$. The constant $c_0$ is uniform for $x^*\in\mathcal{B}(s)$.
    Consider the set of parameters $\mathcal{B}(s)=\{x\in\R^d;\;|\{j:x_j\neq 0\}|\leq s\}$.
  }
  The convergence of $\|\widehat{x}_n - x^*\|_1$ in \eqref{eq:x_L1_conv} is a direct consequence of Proposition 1 in \cite{agarwal2012stochastic}. To see this, from Lemma 1 and Proposition 1 of \cite{agarwal2012stochastic},
  By Lemma 1 (Eq. (39)) of \cite{agarwal2012stochastic}, we have
  \begin{eqnarray}\label{eq:x_L1_conv_lambda}
  \|\widehat{x}_n-x^*\|_1 \leq  \frac{9}{\mu}s\lambda_{K_n},
  \end{eqnarray}
  where $K_n$ is the number of epochs and $\lambda_{K_n}$ is  the regularization parameters used in the $K_n$-th epoch. By Assumption \ref{aspt:cov}, there exists positive constants $\mu, L_F>0$ such that $\mu < \lambda_{\min} (A) <  \lambda_{\max} (A) < L_F$, and $\E(b-a^\top  x)^2$ is Lipschitz continuous with constant $L_F$. By the setting of $\lambda_i$ for the $i$-th epoch in Eq. (34) in \cite{agarwal2012stochastic}, we have
  \begin{eqnarray}\label{eq:lambda_square}
  \lambda_i^2  =\frac{R_i\mu}{ s \sqrt{T_i}} \sqrt{A_{\psi}(L_F^2+K^2)+\omega_i^2K^2},
  \end{eqnarray} where $R_i = \frac{R_1}{2^{i/2}}$, $T_i \geq c_1 \frac{s^2}{R_i^2} \log d$ for a universal constant $c_1$, $A_\psi\leq e\log d$, $\omega_i^2=\omega^2+24 \log i$,  and $K^2$ is the variance proxy of $a$ in Assumption \ref{cor:uniform}. Therefore, let $c_0=(c_1\log d)^{-\frac14}(2\mu)^{\frac12}\left[e(\log d) (L_F^2+K^2)+\omega_i^2K^2\right]^{\frac14}$. We have
  \begin{equation}\label{eq:lambda_K_n}
  \lambda_{K_n}\leq \frac{c_0R_{K_n}}{\sqrt{2}s}\leq c_0\frac{R_1}{2^{K_n/2}s}.
  \end{equation}
  
  To further bound $\lambda_{K_n}$ in \eqref{eq:lambda_K_n}, we only need a lower bound on the number of epochs $K_n$. By the setting of $T_i$ in Eq. (32) in \cite{agarwal2012stochastic}, we can easily derive that,
  \begin{eqnarray*}
    n= \sum_{i=1}^{K_n} T_i  \geq  s^2 (\log d)  2^{K_n},
  \end{eqnarray*}
  which implies that
  \[
  \lambda_{K_n} \leq \frac{c_0R_1}{2^{K_n/2} s } \leq c_0R_1 \sqrt{\frac{\log d}{n}},
  \]
  which combining with \eqref{eq:x_L1_conv_lambda} completes the proof of \eqref{eq:x_L1_conv} with the notice that $R_1$ is bounded uniformly by Assumption \ref{aspt:cov}.
  
  For  \eqref{eq:gamma_L1_conv}, recall that $\widehat{\gamma}^j$ is the solution of RADAR for solving the problem in \eqref{eq:L1_nodewise}. We can apply the Proposition 1 from \cite{agarwal2012stochastic} to provide the bound on $\|\widehat{\gamma}^j- (\gamma^j)^*\|_1$. Here $(\gamma^j)^*$ is the optimum for the un-regularized population loss, i.e.,
  \begin{align*}
  (\gamma^j)^*&=\argmin \E\left( ( a_j - a_{-j} \gamma^j)^2 \right)\\
  &=  -\Omega_{j,j}^{-1} \left(\Omega_{j,-j} \right)^\top 
  \end{align*}
  
  We further note that by the row sparsity assumption of $\Omega$ in Assumption \ref{aspt:cov}, $(\gamma^j)^*$ is an $s_j$-sparse vector. By exactly the same argument as for proving \eqref{eq:x_L1_conv_lambda}, we obtain the result in \eqref{eq:gamma_L1_conv}.
\end{proof}

\subsection{Proof of Theorem \ref{thm:high_inf}}
\label{supp:high_inf}
We first prove two lemmas about estimating $\Omega_{jj}$ and $\Omega$,

\begin{lemma}
  Let $\hat \tau_j \equiv \frac1n \norm{D_{\cdot,j}- D_{\cdot,-j}  \hat \gamma_j}^2_2 $ and $\gamma_j:=\arg\min_{\gamma\in\mathbb{R}^{p-1}}\E[\|a_j-a_{-j}\gamma\|_2^2]$. Assuming $\norm{\hat \gamma_j - \gamma_j}_2 \le C \sqrt{\frac{s_j \log d}{n}}$, $\norm{\hat \gamma_j - \gamma_j}_1 \le C s_j \sqrt{\frac{\log d }{n}}$, and $n> s_j \log d$, we have
  \begin{equation}
  \Bigg| \frac{\hat \tau_j}{\Omega_{jj}^{-1}} -1
  \Bigg| \le C \sqrt{\frac{\log d }{n}} + 2 \lambda_{\max} (A) C\frac{s_j \log d}{n}
  \end{equation}
  \label{lem:tau-estimate}
\end{lemma}
\begin{proof}
  Fix an index $j$. Let $z=\sqrt{\Omega_{j,j}}(D_{\cdot,j} - D_{\cdot,-j} \gamma_j) $. The entries of $z$ are $\emph{i.i.d.}$ realizations of $\sqrt{\Omega_{j,j}}(a_j-a_{-j}\gamma_j )= \sqrt{\Omega_{j,j}}(a_j - \Omega_{jj} ^{-1} \Omega_{j,-j} a_{-j})$. Since $a$ is sub-Gaussian, $z$ is also sub-Gaussian and the entries of $z$ are \emph{i.i.d.} with mean $0$ and variance $1$.
  \begin{align*}
  \sqrt{\hat \tau_j}= \frac1{\sqrt{n}}\norm{D_{\cdot,j} - D_{\cdot,-j} \hat \gamma_j}_2& = \frac1{\sqrt{n}}\norm{D_{\cdot,j} - D_{\cdot,-j}  \gamma_j  + D_{\cdot,-j} (\hat \gamma_j - \gamma_j)}_2 \nonumber \\
  &= \frac1{\sqrt{n}}\norm {\sqrt{\Omega_{jj}^{-1}} z+ D_{\cdot,-j} ( \hat \gamma_j - \gamma_j)}_2
  \end{align*}
  Therefore
  \begin{equation}
  \label{eq:var-est}
  \left| \sqrt{\hat \tau_j}- \frac1{\sqrt{n}}\norm {\sqrt{\Omega_{jj}^{-1}} z}_2\right|\le   \frac1{\sqrt{n}} \norm{ D_{\cdot,-j} ( \hat \gamma_j - \gamma_j)}_2
  \end{equation}
  By \cite{rudelson2012reconstruction}, $\frac1n D_{\cdot,-j}^\top D_{\cdot,-j}$ satisfies the Upper-RE condition in \cite[Equation 2.13]{loh2012high}, we have that
  \begin{align*}
  \frac1n \norm{D_{\cdot,-j} ( \hat \gamma_j - \gamma_j)}^2_2 \le& 2 \lambda_{\max}(A) \norm{ \hat \gamma_j - \gamma_j }_2 ^2 + \tau(n,d) \norm{\hat \gamma_j - \gamma_j}_1 ^2,\\
  \frac1{\sqrt{n}}   \norm{D_{\cdot,-j} ( \hat \gamma_j - \gamma_j)}_2 \le& \sqrt{2 \lambda_{\max}(A) } \norm{ \hat \gamma_j - \gamma_j }_2 + \sqrt{\tau(n,d)}\norm{\hat \gamma_j - \gamma_j}_1
  \end{align*}
  for all $j$, where $\gamma(n,d)\asymp\sqrt{(\log d)/n}$.
  For a sub-Gaussian random variable $X$ with parameter $K$, it's easy to see $X^2$ is sub-exponential with parameter $K^2$ \cite[Lemma 5.14]{vershynin2010introduction}. By the concentration inequality for $K^2$ sub-exponential random variables \cite[Proposition 5.16]{vershynin2010introduction},
  \begin{align*}
  \Pr\left( \left|  \frac1n \norm{z}_2 ^2 -1 \right | > t\right) \le 2 \exp\left(-\min\Big(\frac{nt^2}{2K^4},\frac{nt}{2K^2}\Big)\right).
  \end{align*}
  
  In below, let $C_K$ be a dimension independent constant that may change its value through derivation. By setting $t\asymp \sqrt{\frac{ \log d}{n}}$ with probability at least $ 1- \frac{1}{d^C} $,
  \begin{align*}
  \big|\frac1n \norm{z}^2 -1\big|&\le C_{K} \sqrt{\frac{\log d}{n}}.
  \end{align*}
  Using $\sqrt{a+b}< \sqrt{a}+\sqrt{b} $ and re-arranging,
  \begin{align*}
  \sqrt{1-C_K \sqrt{\frac{ \log d}{n} }}&\le \frac{\norm{z}_2}{\sqrt{n}} \le \sqrt{1+C_K \sqrt{\frac{ \log d}{n} }},\\
  1-C_K \left(\frac{ \log d}{n} \right)^{.25}&\le \frac{\norm{z}_2}{\sqrt{n}} \le 1+C_K \left(\sqrt{\frac{ \log d}{n} }\right)^{.25}.
  \end{align*}
  
  Using these two estimates with \eqref{eq:var-est} ,
  \begin{align*}
  \sqrt{\hat \tau_j}&\le  \sqrt{\Omega_{jj}^{-1}}( 1+ C_K \left(\frac{\log d}{n}\right)^{.25} ) + \sqrt{2 \lambda_{\max}(A) } \sqrt{\frac{s_j \log d }{n}} + Cs_j \frac{ \log d}{n}, \\
  \sqrt{\hat \tau_j}&\ge \sqrt{\Omega_{jj}^{-1}}( 1- C_K \left(\frac{\log d}{n}\right)^{.25} ) - \sqrt{2 \lambda_{\max}(A) } \sqrt{\frac{s_j \log d }{n}} - Cs_j \frac{ \log d}{n}.
  \end{align*}
  Dividing by $\sqrt{\Omega_{jj} ^{-1}}$, squaring, and using $\frac{s_j \log d}{n}<1$ gives
  \begin{align*}
  | \frac{\hat \tau_j}{\Omega_{jj}^{-1}} -1 | &\lesssim  \sqrt{\frac{\log d }{n}} + 2 \lambda_{\max}(A)  \frac{s_j \log d}{n}.
  \end{align*}
\end{proof}
\begin{lemma}\label{lem:M_L2}
  Let $\Omega_j$ and $\widehat{\Omega}_j$ be the $j$-th row of $\Omega$ and $\widehat{\Omega}$. Under the row sparsity assumption of $\Omega$ (see Assumption \ref{aspt:cov}), we have with high probability
  \begin{eqnarray*}
    \max_{j\in \{1,\ldots,d\}} \| \widehat{\Omega}_j - \Omega_j\|_2 \le C   \sqrt{\frac{ \max_j s_j\log d}{n}}\\
    \max_{j\in \{1,\ldots,d\}} \| \widehat{\Omega}_j - \Omega_j\|_1 \le C   \max_j s_j \sqrt{\frac{\log d}{n}}\\
  \end{eqnarray*}
\end{lemma}

\begin{proof}[Proof of Lemma \ref{lem:M_L2}]
  By the definition of $\widehat{\Omega}$ in \eqref{eq:M}, we have the $j$-th row of $\widehat{\Omega}$ takes the following form,
  \begin{eqnarray*}
    \widehat{\Omega}_j= \widehat{C}_j / \widehat{\tau}_j,  \qquad \widehat{C}_j = (-\widehat{\gamma}^j_1, \ldots, -\widehat{\gamma}^j_{j-1}, 1, \ldots, -\widehat{\gamma}^j_d).
  \end{eqnarray*}
  By Proposition  \ref{prop:L1_conv}, we have for any $j$,
  \begin{eqnarray}\label{eq:C_j_L1}
  \| \widehat{C}_j- \Omega_{j,j}^{-1} \Omega_{j}\|_1 \lesssim  s_j \sqrt{\frac{\log d}{n}},
  \end{eqnarray}
  holding with high probability.
  
  Using Lemma \ref{lem:tau-estimate}, $\hat \tau_j =O(1)$, and noting that $\frac{s_j \log d}{n} < \sqrt{\frac{\log d}{n}}$, we see that
  \begin{equation}
  \left| \frac{\Omega_{jj}^{-1}}{\hat \tau_j} -1 \right| \le C \sqrt{\frac{\log d}{n}}.
  \label{eq:inverse_tau_bound}
  \end{equation}

  Combining  \eqref{eq:inverse_tau_bound}  with \eqref{eq:C_j_L1} and noticing that
  \begin{eqnarray*}
    \|\Omega_j\|_1 \leq \sqrt{s_j} \|\Omega_j\|_2 \leq \frac{\sqrt{s_j}}{\lambda_{\min}(A)},
  \end{eqnarray*}
  we have with probability at least $1-1/d^2$,
  \begin{align*}
  \|\widehat{\Omega}_j-\Omega_j\|_1 & = \left\|\left(\widehat{C}_j / \widehat{\tau}_j - \Omega_{j,j}^{-1} \Omega_{j}/ \widehat{\tau}_j  \right) +\left( \Omega_{j,j}^{-1} \Omega_{j}/ \widehat{\tau}_j -\Omega_{j}\right)   \right\|_1 \\
  & \leq \frac{1}{\widehat{\tau}_j} \|\widehat{C}_j- \Omega_{j,j}^{-1} \Omega_{j}\|_1+ \left|\frac{\Omega_{j,j}^{-1}}{\widehat{\tau}_j} -1 \right| \|\Omega_j\|_1 \\
  & \lesssim s_j \sqrt{\frac{\log d}{n}},
  \end{align*}
  which completes the proof by taking a union bound over $j$. The proof for the $l_2$ bound is analogous.
\end{proof}

With Lemma \ref{lem:M_L2} in place, we are ready to prove Theorem \ref{thm:high_inf}.

\begin{reptheorem}{thm:high_inf}
  Under Assumption \ref{aspt:cov}, for suitable choices of $\lambda \asymp \sqrt{\log d/n}$ and $\lambda_j \asymp \sqrt{\log d/n}$, we have for all $j \in \{1,\ldots, d\}$ and all $z \in \mathbb{R}$,
  \label{cor:uniform}
  \begin{eqnarray}\label{eq:normal_conv}
  \sup\limits_{x^*\in\mathcal{B}(s)}\left|\mathbb{P}_{x^*}\Bigg(\frac{\sqrt{n}(\widehat{x}^d_j- x_j^*)}{\sigma\sqrt{(\widehat{\Omega} \widehat{A} \widehat{\Omega}^\top )_{jj}}}\leq z\Bigg)-\Phi(z)\right|=o_{p}(1).
  \end{eqnarray}
  where $\widehat{x}^d$ is the debiased estimator defined in \eqref{eq:debiasedLasso}, $\widehat{\Omega}$ is defined in \eqref{eq:M} and the sample covariance matrix $\widehat{A}= \frac{1}{n} D^\top  D$.
\end{reptheorem}
\begin{proof}
  We first plug the linear model $b= D x^* + \epsilon$ into the definition of the  desparsified estimator in \eqref{eq:debiasedLasso} and obtain,
  \begin{align}\label{eq:74}
  \widehat{x}^d & = \widehat{x}_n + \frac1n \widehat{\Omega}\widehat{A}\left(x^*-  \widehat{x}_n  \right) + \frac{1}{n} \widehat{\Omega} D^\top  \epsilon \\
  & =x^* + \left( \widehat{\Omega}\widehat{A}- I \right) \left(x^*-  \widehat{x}_n  \right) + \frac{1}{n} \widehat{\Omega} D^\top  \epsilon. \nonumber
  \end{align}
  It is easy to see that $\frac{1}{\sqrt{n}} \widehat{\Omega} D^\top  \epsilon \sim N(0, \sigma^2 \widehat{\Omega} \widehat{A} \widehat{\Omega}^\top )$. Therefore, to prove  Theorem \ref{thm:high_inf}, we only need to show the term
  \begin{equation}\label{eq:Delta_op}
  \Delta :=  \left( \widehat{\Omega}\widehat{A}- I \right) \left(x^*-  \widehat{x}_n  \right) = o_{\mathbb{P}_{x^*}}(\frac{1}{\sqrt{n}}),
  \end{equation}
  uniformly in $x^*\in\mathcal{B}(s)$.
  
  By \cite[Lemma 23]{javanmard2014confidence}, we have
  \[
  \mathrm{Pr}\left(\|\Omega\widehat{A}-I\|_\infty\geq a\sqrt{\frac{\log d}{n}}\right)\leq 2d^{-c_2}
  \]
  with $c_2=(a^2\lambda_{\min})/(24e^2\kappa^4\lambda_{\max})-2$, where $\kappa$ is the sub-gaussian norm for $A^{-1}a$.
  By adding and subtracting $ \Omega \widehat{A}  (\hat x_n- x^*)$,
  \begin{align*}
  \norm{\Delta}_\infty &\le \norm{ ( \widehat{\Omega} \widehat{A} -\Omega \widehat{A} )(\hat x_n- x^*  )}_\infty + \norm{ (\Omega \widehat{A} - I) (\hat x_n -x ^*)}_\infty\\
  &\le\norm{ ( \widehat{\Omega} \widehat{A} -\Omega \widehat{A} )(\hat x_n- x^*  )}_\infty  + \norm{\Omega \widehat{A} - I}_\infty \norm{\hat x_n -x ^*}_1\\
  &\le \norm{ ( \widehat{\Omega} \widehat{A} -\Omega \widehat{A} )(\hat x_n- x^*  )}_\infty  + C \frac{s\log d}{n}
  \end{align*}
  
  The first term can be bounded
  \begin{align*}
  & \norm{ ( \widehat{\Omega} \widehat{A} -\Omega \widehat{A} )(\hat x_n- x^*  )}_\infty \\
  = & \norm{ (\widehat{\Omega}-\Omega) \frac1n D^\top  D (\hat x_n -x^*) } _\infty\\
  \le & \left[ \max_{j} \frac{1}{\sqrt{n}} \norm{D(\widehat{\Omega}^\top  -\Omega) e_j  }_2  \right] \frac{1}{\sqrt{n}}\norm{ D( \hat x_n -x^*)}_2\\
  \le &\left[ \max_j \sqrt{C_1 \lambda_{\max}(A) \norm{(\widehat{\Omega}^\top -\Omega) e_j}_2^2 +C_2 \frac{\log d }{n} \norm{(\widehat{\Omega}^\top  -\Omega)e_j}_1 ^2 } \right] \\
  & \sqrt{2\lambda_{\max}(A) \norm{\hat x_n -x^*}_2^2 + \tau(n,d) \norm{\hat x_n -x^*}_1 ^2 } \\
  \leq  & C_3\max_j \sqrt{\frac{s_j \log d }{n} + \frac{s_j^2 \log^2 d}{n^2}} \sqrt{\frac{s \log d }{n} + \frac{s^2 \log^2 d}{n^2}} \\
  \leq  & C_4\frac{s \log d}{n}
  \end{align*}
  where in the second inequality we used \cite[Equation 2.13]{loh2012high} and \cite{rudelson2012reconstruction}. In the third inequality we used
  \[
  \norm{(\widehat{\Omega}^\top  -\Omega)e_j}_2 \lesssim \sqrt { \frac{s_j \log d }{n}}, \quad \norm{(\widehat{\Omega}^\top  -\Omega)e_j}_1  \lesssim s_j\sqrt { \frac{ \log d }{n}}
  \]
  from Lemma \ref{lem:M_L2},   $\norm{\hat x_n -x ^*}_2 \leq c_0' \sqrt { \frac{s\log d }{n}}$, and $\norm{\hat x_n -x ^*}_1 \leq c_0 s \sqrt { \frac{ \log d }{n}}$ uniformly for all $x^*$ from \cite{agarwal2012stochastic} and Proposition \ref{prop:L1_conv}. In the last inequality, we use $s=o(\sqrt{n}/\log d)$ and $s_j \le C s$ by Assumption \ref{aspt:cov}.
  
  By the assumption that $s =o(\sqrt{n}/\log d)$, we have that $\norm{\Delta}_\infty \le  o_{\mathbb{P}_{x^*}}(\frac{1}{\sqrt{n}})$ uniformly in $x^*\in\mathcal{B}(s)$, which completes the proof.

\end{proof}

\section{More Simulations}
\label{sec:sim_supp}
In this section, we report more simulation experiments by considering different design matrices $\Sigma$ (recall that each covariate $a_n \sim N(0, \Sigma)$).

In Table \ref{table:linear_CI_toep}, we report the performance of the linear regression case with Toeplitz covariance matrix of $\Sigma_{i,j}=r^{|i-j|}$ for different correlation parameter $r$'s. As the correlation parameter $r$ varies from $0.4$ to $0.6$, the coverage rates remains similar when $d$ is small. On the other hand, when $d$ is large, the coverage rates slightly decrease as $r$ increases.  At the meantime, the average lengths increase, which is consistent with the trend of the oracle lengths (see the last column).  In Table \ref{table:linear_CI_equi}, we report the simulations with Equi Corr covariance matrix $\Sigma_{i,j}=r\cdot 1_{\{i\neq j\}}+ 1_{\{i= j\}}$ when $r$ varies from $0.1$ to $0.3$. The observation is similar to the case of Toeplitz covariance matrix. It is worth noting that for each of the cases, the minimum eigenvalue of the covariance matrix $\Sigma$ decreases when $r$ increases. Therefore Tables \ref{table:linear_CI_toep} and \ref{table:linear_CI_equi} show that our methods are robust with the minimum eigenvalue of covariance matrix $\Sigma$. Similar comparisons for logistic regression cases are provided in Tables \ref{table:log_CI_toep} and \ref{table:log_CI_equi}.

\setcounter{table}{4}
\begin{table}[!t]
  \caption{Linear Regression: the coverage rate, average length of confidence intervals and their standard errors, for the nominal coverage
    probability  $95\%$. The covariance matrix is in the Toeplitz setting, where $\Sigma_{i,j}=r^{|i-j|}$ for $r=0.4,0.5,0.6$, $i,j=1,2,\dots, d$. Standard errors are reported in the brackets.}
  \begin{tabular}{llrrrrr}
    \hline
    & $d$ & Plug-in & \multicolumn{3}{c}{BM}& Oracle \\
    &&&$M=n^{0.2}$ &  $M=n^{0.25}$ & $M=n^{0.3}$ &\\
    \hline
    $r=0.4$\\
    Cov Rate (\%)&5&95.16(0.69)&90.88(1.78)&92.56(0.93)&93.40(1.12)&87.44\\
    Avg Len ($\times 10^{-2}$)&&1.70(0.07)&1.60(0.06)&1.65(0.07)&1.69(0.06)&1.42 \\
    Cov Rate (\%)&20&94.91(0.92)&91.39(1.04)&93.07(1.13)&94.26(0.96)&89.12\\
    Avg Len ($\times 10^{-2}$)&&1.66(0.04)&1.56(0.04)&1.61(0.04)&1.63(0.04)&1.45 \\
    Cov Rate (\%)&100&94.79(1.00)&90.55(1.28)&92.00(1.22)&92.47(1.08)&88.26\\
    Avg Len ($\times 10^{-2}$)&&1.64(0.02)&1.53(0.02)&1.55(0.02)&1.54(0.02)&1.45 \\
    Cov Rate (\%)&200&94.41(1.02)&90.22(1.23)&91.46(1.18)&91.83(1.19)&87.49\\
    Avg Len ($\times 10^{-2}$)&&1.61(0.01)&1.50(0.02)&1.50(0.01)&1.49(0.01)&1.46 \\
    \hline
    $r=0.5$\\
    Cov Rate (\%)&5&95.24(0.92)&91.16(0.50)&94.28(0.86)&93.04(0.90)&88.31\\
    Avg Len ($\times 10^{-2}$)&&1.83(0.10)&1.74(0.10)&1.82(0.11)&1.78(0.12)&1.53 \\
    Cov Rate (\%)&20&94.84(0.97)&90.97(1.08)&93.75(0.93)&92.77(0.81)&87.26\\
    Avg Len ($\times 10^{-2}$)&&1.81(0.05)&1.71(0.06)&1.78(0.06)&1.76(0.06)&1.58 \\
    Cov Rate (\%)&100&95.01(1.12)&90.36(1.33)&91.83(1.09)&91.52(1.17)&89.11\\
    Avg Len ($\times 10^{-2}$)&&1.77(0.02)&1.67(0.03)&1.67(0.03)&1.69(0.02)&1.60 \\
    Cov Rate (\%)&200&94.69(1.33)&90.01(1.41)&91.65(1.36)&91.24(1.41)&89.43\\
    Avg Len ($\times 10^{-2}$)&&1.74(0.02)&1.62(0.02)&1.62(0.02)&1.62(0.02)&1.60 \\
    \hline
    $r=0.6$\\
    Cov Rate (\%)&5&94.92(1.21)&90.72(0.58)&93.20(0.87)&94.28(0.73)&87.34\\
    Avg Len ($\times 10^{-2}$)&&0.00(0.14)&1.91(0.14)&1.97(0.15)&2.03(0.17)&1.70 \\
    Cov Rate (\%)&20&95.11(1.03)&91.15(1.25)&92.54(0.85)&93.46(0.86)&88.59\\
    Avg Len ($\times 10^{-2}$)&&0.00(0.07)&1.93(0.09)&1.97(0.09)&2.00(0.09)&1.78 \\
    Cov Rate (\%)&100&94.76(1.19)&90.18(1.15)&91.18(1.20)&91.31(1.21)&89.11\\
    Avg Len ($\times 10^{-2}$)&&0.00(0.03)&1.87(0.04)&1.88(0.03)&1.86(0.03)&1.80 \\
    Cov Rate (\%)&200&94.59(1.31)&90.02(1.29)&90.87(1.38)&91.09(1.46)&87.46\\
    Avg Len ($\times 10^{-2}$)&&0.00(0.02)&1.79(0.03)&1.78(0.02)&1.85(0.03)&1.80 \\
    \hline
  \end{tabular}
  \label{table:linear_CI_toep}
\end{table}

\begin{table}[!t]
  \caption{Linear Regression: the coverage rate, average length of confidence intervals and their standard errors, for the nominal coverage
    probability  $95\%$. The covariance matrix is in the Equi Corr setting, where $\Sigma_{i,j}=r\cdot 1_{\{i\neq j\}}+ 1_{\{i= j\}}$ for $r=0.1,0.2,0.3$, $i,j=1,2,\dots, d$. Standard errors are reported in the brackets.}
  \begin{tabular}{llrrrrr}
    \hline
    & $d$ & Plug-in & \multicolumn{3}{c}{BM}& Oracle \\
    &&&$M=n^{0.2}$ &  $M=n^{0.25}$ & $M=n^{0.3}$ &\\
    \hline
    $r=0.1$\\
    Cov Rate (\%)&5&94.88(1.07)&90.84(1.62)&92.48(1.45)&94.40(1.26)&88.19\\
    Avg Len ($\times 10^{-2}$)&&1.52(0.01)&1.40(0.01)&1.45(0.01)&1.49(0.01)&1.26 \\
    Cov Rate (\%)&20&94.98(0.94)&90.65(1.00)&92.49(1.26)&93.46(0.93)&87.41\\
    Avg Len ($\times 10^{-2}$)&&1.50(0.01)&1.40(0.01)&1.43(0.01)&1.45(0.01)&1.28 \\
    Cov Rate (\%)&100&95.62(1.11)&90.67(1.33)&92.14(1.24)&92.92(1.15)&89.11\\
    Avg Len ($\times 10^{-2}$)&&1.47(0.01)&1.38(0.01)&1.41(0.01)&1.41(0.01)&1.30 \\
    Cov Rate (\%)&200&94.61(1.10)&90.42(1.35)&92.05(1.12)&92.49(1.18)&88.40\\
    Avg Len ($\times 10^{-2}$)&&1.41(0.01)&1.37(0.01)&1.38(0.01)&1.38(0.01)&1.30 \\
    \hline
    $r=0.2$\\
    Cov Rate (\%)&5&94.80(0.88)&90.92(1.09)&93.60(0.92)&92.32(0.68)&86.79\\
    Avg Len ($\times 10^{-2}$)&&1.60(0.01)&1.46(0.01)&1.55(0.01)&1.52(0.01)&1.31 \\
    Cov Rate (\%)&20&95.10(0.99)&91.15(1.14)&93.66(0.99)&92.78(0.92)&88.04\\
    Avg Len ($\times 10^{-2}$)&&1.59(0.01)&1.47(0.01)&1.54(0.01)&1.51(0.01)&1.36 \\
    Cov Rate (\%)&100&94.93(1.06)&90.86(1.26)&93.19(1.15)&92.29(1.10)&87.15\\
    Avg Len ($\times 10^{-2}$)&&1.56(0.01)&1.47(0.01)&1.52(0.01)&1.50(0.01)&1.38 \\
    Cov Rate (\%)&200&94.49(1.09)&90.57(1.45)&92.45(1.27)&91.91(1.13)&87.22\\
    Avg Len ($\times 10^{-2}$)&&1.51(0.01)&1.45(0.01)&1.49(0.01)&1.49(0.01)&1.38 \\
    \hline
    $r=0.3$\\
    Cov Rate (\%)&5&95.00(0.96)&91.80(1.34)&92.76(0.75)&94.16(1.21)&87.12\\
    Avg Len ($\times 10^{-2}$)&&1.74(0.01)&1.57(0.02)&1.60(0.02)&1.65(0.01)&1.38 \\
    Cov Rate (\%)&20&95.27(1.00)&91.08(1.32)&92.60(1.28)&93.52(1.16)&86.91\\
    Avg Len ($\times 10^{-2}$)&&1.72(0.01)&1.58(0.01)&1.62(0.02)&1.63(0.01)&1.45 \\
    Cov Rate (\%)&100&94.58(1.14)&90.68(1.25)&92.23(1.14)&93.35(1.11)&87.05\\
    Avg Len ($\times 10^{-2}$)&&1.71(0.01)&1.57(0.01)&1.61(0.02)&1.63(0.01)&1.47 \\
    Cov Rate (\%)&200&94.69(1.25)&90.31(1.46)&91.73(1.27)&93.21(1.11)&88.32\\
    Avg Len ($\times 10^{-2}$)&&1.64(0.01)&1.55(0.01)&1.60(0.01)&1.61(0.01)&1.48 \\
    \hline
  \end{tabular}
  \label{table:linear_CI_equi}
\end{table}

\setcounter{table}{6}
\begin{table}[!t]
  \caption{Logistic Regression: the coverage rate, average length of confidence intervals and their standard errors, for the nominal coverage
    probability  $95\%$. The covariance matrix is in the Toeplitz setting, where $\Sigma_{i,j}=r^{|i-j|}$ for $r=0.4,0.5,0.6$, $i,j=1,2,\dots, d$. Standard errors are reported in the brackets.}
  \begin{tabular}{llrrrrr}
    \hline
    & $d$ & Plug-in & \multicolumn{3}{c}{BM}& Oracle \\
    &&&$M=n^{0.2}$ &  $M=n^{0.25}$ & $M=n^{0.3}$ &\\
    \hline
    $r=0.4$\\
    Cov Rate (\%)&5&95.08(1.41)&89.04(2.19)&90.92(2.01)&90.28(1.97)&94.71\\
    Avg Len ($\times 10^{-2}$)&&3.74(0.28)&3.44(0.25)&3.36(0.24)&3.26(0.21)&3.70\\
    Cov Rate (\%)&20&94.69(1.33)&89.64(1.99)&90.49(1.87)&90.41(1.47)&93.89\\
    Avg Len ($\times 10^{-2}$)&&4.99(0.29)&4.65(0.24)&4.49(0.22)&4.41(0.22)&4.90\\
    Cov Rate (\%)&100&94.81(1.21)&89.38(1.84)&90.31(1.77)&90.26(1.92)&94.20\\
    Avg Len ($\times 10^{-2}$)&&7.21(0.34)&6.51(0.29)&6.33(0.26)&6.16(0.24)&7.15\\
    Cov Rate (\%)&200&94.54(1.04)&89.27(1.78)&90.22(1.69)&90.31(1.88)&93.96\\
    Avg Len ($\times 10^{-2}$)&&8.68(0.37)&7.88(0.31)&7.82(0.30)&7.71(0.28)&8.47\\
    \hline
    $r=0.5$\\
    Cov Rate (\%) & 5 & 94.96(1.58) & 88.96(2.32) & 90.56(2.06) & 90.12(2.04) &92.41\\
    Avg Len ($\times 10^{-2}$) &  &4.06(0.34) & 3.75(0.28) & 3.73(0.27) & 3.61(0.25) &4.04\\
    Cov Rate (\%) & 20 & 95.17(1.23) & 89.01(1.93) & 90.39(1.88) & 89.79(1.81)&91.07\\
    Avg Len ($\times 10^{-2}$) &  & 5.74(0.29) & 5.57(0.25)  &5.22(0.23)  & 4.95(0.22)&5.59\\
    Cov Rate (\%) & 100 &94.91(0.89) & 89.91(1.74)& 90.83(1.81) & 90.54(1.97)&91.47\\
    Avg Len ($\times 10^{-2}$) &  &8.47(0.37) & 8.01(0.28)  & 7.71(0.26)  & 7.37(0.25)&8.28\\
    Cov Rate (\%) & 200 &94.59(1.04) & 89.72(1.81)& 90.74(1.93) & 90.32(2.02)&92.29\\
    Avg Len ($\times 10^{-2}$) &  &9.81(0.41) & 9.24(0.34) & 8.95(0.31)  & 8.78(0.29)&9.84  \\
    \hline
    $r=0.6$\\
    Cov Rate (\%)&5&94.72(1.88)&89.32(2.38)&90.16(2.14)&89.96(2.16)&93.77\\
    Avg Len ($\times 10^{-2}$)&&4.69(0.47)&4.31(0.41)&4.22(0.39)&4.13(0.36)&4.54\\
    Cov Rate (\%)&20&95.11(1.49)&89.07(1.91)&90.44(1.97)&89.91(1.94)&94.62\\
    Avg Len ($\times 10^{-2}$)&&6.59(0.41)&6.38(0.36)&6.26(0.33)&6.11(0.31)&6.58\\
    Cov Rate (\%)&100&94.29(1.11)&90.14(1.79)&90.31(1.89)&90.42(1.82)&93.15\\
    Avg Len ($\times 10^{-2}$)&&10.01(0.39)&9.61(0.36)&9.11(0.34)&8.89(0.32)&9.93\\
    Cov Rate (\%)&200&94.47(1.28)&89.63(1.84)&90.17(1.98)&90.13(2.01)&93.27\\
    Avg Len ($\times 10^{-2}$)&&12.27(0.41)&11.39(0.37)&10.97(0.34)&10.66(0.31)&11.83\\
    \hline
  \end{tabular}
  \label{table:log_CI_toep}
\end{table}

\begin{table}[!t]
  \caption{Logistic Regression: the coverage rate, average length of confidence intervals and their standard errors, for the nominal coverage
    probability  $95\%$. The covariance matrix is in the Equi Corr setting, where $\Sigma_{i,j}=r\cdot 1_{\{i\neq j\}}+ 1_{\{i= j\}}$ for $r=0.1,0.2,0.3$, $i,j=1,2,\dots, d$. Standard errors are reported in the brackets.}
  \begin{tabular}{llrrrrr}
    \hline
    & $d$ & Plug-in & \multicolumn{3}{c}{BM}& Oracle \\
    &&&$M=n^{0.2}$ &  $M=n^{0.25}$ & $M=n^{0.3}$ &\\
    \hline
    $r=0.1$\\
    Cov Rate(\%)&5&94.72(1.75)&88.80(1.36)&89.68(1.58)&89.56(1.45)&94.09\\
    Avg Len($\times 10^{-2}$)&&3.39(0.32)&3.04(0.26)&2.97(0.24)&2.90(0.23)&3.21 \\
    Cov Rate(\%)&20&94.85(1.82)&89.48(1.48)&90.89(1.48)&90.68(1.56)&94.12\\
    Avg Len($\times 10^{-2}$)&&4.94(0.32)&4.07(0.24)&4.00(0.23)&3.82(0.22)&4.48 \\
    Cov Rate(\%)&100&95.11(1.58)&89.41(1.52)&90.51(1.55)&90.15(1.70)&93.88\\
    Avg Len($\times 10^{-2}$)&&9.06(0.47)&8.45(0.42)&8.24(0.40)&8.01(0.37)&8.83 \\
    Cov Rate(\%)&200&94.55(1.46)&88.95(1.58)&89.87(1.30)&89.60(1.79)&92.46\\
    Avg Len($\times 10^{-2}$)&&12.79(0.61)&11.99(0.55)&11.42(0.52)&11.28(0.52)&12.33 \\
    \hline
    $r=0.2$\\
    Cov Rate (\%) & 5 &94.80(1.66) & 88.08(1.46) & 88.64(1.73) & 89.48(1.51)&93.79 \\
    Avg Len ($\times 10^{-2}$) &  & 3.43(0.35)&3.28(0.28) & 3.24(0.25)& 3.20(0.24) &3.38\\  
    Cov Rate (\%) & 20 & 94.54(1.73) & 89.27(1.33) & 90.64(1.60) & 90.31(2.10)&92.50\\
    Avg Len ($\times 10^{-2}$) &  &5.37(0.31) &  4.84(0.26)  & 4.77(0.24) & 4.51(0.21)&5.19 \\
    Cov Rate (\%) & 100 & 94.79(1.08) & 89.01(1.70) & 90.27(1.76) & 89.42(2.01)&94.92\\
    Avg Len ($\times 10^{-2}$) &  &10.24(0.51) & 10.17(0.47)  & 9.75(0.42)  & 9.24(0.40)&10.89 \\
    Cov Rate (\%) & 200 & 94.24(1.09) & 89.13(1.44) & 90.01(1.92) & 89.23(1.79)&92.40\\
    Avg Len ($\times 10^{-2}$) &  &15.70(0.62) & 14.82(0.57)  & 14.01(0.55)  &13.88(0.52) &15.31 \\
    \hline
    $r=0.3$\\
    Cov Rate(\%)&5&95.00(1.84)&89.12(1.54)&89.68(1.77)&89.28(1.64)&93.42\\
    Avg Len($\times 10^{-2}$)&&3.65(0.33)&3.44(0.26)&3.39(0.24)&3.31(0.23)&3.61 \\
    Cov Rate(\%)&20&94.86(1.79)&89.25(1.45)&90.42(1.59)&89.84(1.63)&93.69\\
    Avg Len($\times 10^{-2}$)&&5.98(0.33)&5.69(0.25)&5.54(0.24)&5.48(0.21)&5.91 \\
    Cov Rate(\%)&100&94.57(1.44)&89.04(1.71)&90.07(1.65)&89.22(1.68)&94.01\\
    Avg Len($\times 10^{-2}$)&&13.07(0.64)&12.15(0.48)&11.80(0.46)&11.64(0.41)&12.77 \\
    Cov Rate(\%)&200&94.21(1.55)&89.00(1.62)&90.11(1.74)&89.07(1.78)&92.19\\
    Avg Len($\times 10^{-2}$)&&19.41(0.69)&17.15(0.60)&16.70(0.59)&16.29(0.55)&18.01 \\
    \hline
  \end{tabular}
  \label{table:log_CI_equi}
\end{table}
\newpage

\bibliographystyle{plain}
\bibliography{ref}

\end{document}